\documentclass{article}

 \usepackage[final]{neurips_2020}

\title{Stochasticity of Deterministic Gradient Descent: \\ \hspace{-47pt} \qquad Large Learning Rate for Multiscale Objective Function}

\usepackage{times}
\usepackage{color}
\usepackage{float}
\usepackage{amsfonts}
\usepackage{comment}
\usepackage[T1]{fontenc}
\usepackage{wrapfig}
\usepackage{amsmath}
\usepackage{amssymb}
\usepackage{amsthm}
\usepackage{natbib}
\usepackage{graphicx}
\usepackage{color}
\usepackage{hyperref}
\usepackage{subfigure}     

\usepackage[utf8]{inputenc} 
\usepackage[T1]{fontenc}    
\usepackage{hyperref}       
\usepackage{url}            
\usepackage{booktabs}       
\usepackage{amsfonts}       
\usepackage{nicefrac}       
\usepackage{microtype}      

\graphicspath{{./fig/}}

\newcommand{\tao}[1]{[{\color{cyan} Tao: #1}]}
\newcommand{\kong}[1]{[{\color{red} Kong: #1}]}

\theoremstyle{definition}
\newtheorem{definition}{Definition}
\newtheorem{remark}{Remark}

\theoremstyle{plain}
\newtheorem{theorem}{Theorem}
\newtheorem{proposition}[theorem]{Proposition}
\newtheorem{corollary}[theorem]{Corollary}
\newtheorem{lemma}[theorem]{Lemma}
\newtheorem{example}{Example}
\newtheorem{condition}{Condition}

\author{%
  Lingkai Kong\\
  School of Mathematics \\
  University of Science and Technology of China \\
  and Georgia Institute of Technology \\
  \And
  Molei Tao \\
  School of Mathematics \\
  Georgia Institute of Technology \\
  \url{mtao@gatech.edu}
}

\begin{document}

\maketitle

\begin{abstract}%
\noindent This article suggests that deterministic Gradient Descent, which does not use any stochastic gradient approximation, can still exhibit stochastic behaviors. In particular, it shows that if the objective function exhibit multiscale behaviors, then in a large learning rate regime which only resolves the macroscopic but not the microscopic details of the objective, the deterministic GD dynamics can become chaotic and convergent not to a local minimizer but to a statistical distribution. In this sense, deterministic GD resembles stochastic GD even though no stochasticity is injected. A sufficient condition is also established for approximating this long-time statistical limit by a rescaled Gibbs distribution, which for example allows escapes from local minima to be quantified. Both theoretical and numerical demonstrations are provided, and the theoretical part relies on the construction of a stochastic map that uses bounded noise (as opposed to Gaussian noise).
\end{abstract}

\section{Introduction}


Among first-order optimization methods which are a central ingredient of machine learning, arguably the most used is gradient descent method (GD), or rather one of its variants, stochastic gradient descent method (SGD). Designed for objective functions that sum a large amount of terms, which for instance can originate from big data, SGD introduces a randomization mechanism of gradient subsampling to improve the scalability of GD (e.g., \cite{zhang2004solving, moulines2011non, roux2012stochastic}). Consequently, the iteration of SGD, unlike GD, is not deterministic even when it is started at a fixed initial condition. In fact, if one fixes the learning rate (LR) in SGD, the iteration does not converge to a local minimizer like in the case of GD; instead, it converges to a statistical distribution with variance controlled by the LR (e.g., \cite{borkar1999strong,mandt2017stochastic,li2017stochastic}). Diminishing LR was thus proposed to ensure that SGD remains as an optimization algorithm (e.g., \cite{robbins1951stochastic}). On the other hand, more recent perspectives include that the noise in SGD may actually facilitate escapes from bad local minima and improve generalization (see Sec.\ref{sec_relatedWork} and references therein). In addition, non-diminishing LRs often correspond to faster computations, and therefore are of practical relevance\footnote{Optimizing LR is an important subarea but out of our scope; see e.g., \cite{smith2017cyclical} and references therein.}. Meanwhile, GD does not need the LR to be small in order to reduce the stochasticity, although in practices the LR is often chosen small enough to fully resolve the landscape of the objective, corresponding to a stability upper bound of $1/L$ under the common $L$-smooth assumption of the objective function.

We consider deterministic GD\footnote{Despite of the importance of SGD, there are still contexts in which deterministic GD is worth studying; e.g., for training with scarce data, for low-rank approximation (e.g., \cite{tu2015low}) and robust PCA (e.g., \cite{yi2016fast}), and for theoretical understandings of large neural networks (e.g, \cite{du2018algorithmic,du2019gradient}).} with fixed large LR, based on the conventional belief that it optimizes more efficiently than small LR. The goal is to understand if large LR works, and if yes, in what sense. We will show that in a specific and yet not too restrictive setup, if LR becomes large enough (but not arbitrarily large), GD no longer converges to a local minimum but instead a statistical distribution. This behavior bears significant similarities to SGD, including (under reasonable assumptions):
\begin{itemize}
    \item
        starting with an arbitrary initial condition, the empirical distribution of GD iterates (collected along discrete time) converges to a specific statistical distribution, which is not Dirac but almost a rescaled Gibbs distribution, just like SGD;
    \item
        starting an ensemble of arbitrary initial conditions and evolving each one according to GD, the ensemble, collected at the same number of iterations, again converges to the same almost Gibbs distribution as the number of iteration increases, also like SGD.
\end{itemize}
Their difference, albeit obvious, should also be emphasized:
\begin{itemize}
    \item 
        GD is deterministic, and the same constant initial condition will always lead to the same iterates. No filtration is involved, and unlike SGD the iteration is not a stochastic process.
\end{itemize}
In this sense, GD with large LR works in a statistical sense. One can obtain stochasticity without any algorithmic randomization! Whether this has implications on generalization is beyond the scope of this article, but large LR does provide a mechanism for escapes from local minima. We'll see that microscopic local minima can always be escaped, and sometimes macroscopic local minima too.

\subsection{Main Results}
\label{sec_mainResults}
How is stochasticity generated out of determinism? Here it is due to chaotic dynamics. To further explain, consider an objective function $f:\mathbb{R}^d \rightarrow \mathbb{R}$ that admits a macro-micro decomposition 
\begin{equation}
    f(x):=f_0(x)+f_{1,\epsilon}(x)
    \label{eq_multiscaleLoss}
\end{equation}
where $0 < \epsilon \ll 1$, $f_0, f_{1,\epsilon} \in \mathcal{C}^2(\mathbb{R}^d)$, and the microscopic $f_{1,\epsilon}$ satisfies the following conditions.
\begin{condition}
    There exists a bounded nonconstant random variable (r.v.) $\zeta$, with range in $\mathbb{R}^d$ and $\mathbb{E}\zeta=0$, such that: $\forall \epsilon>0$ and $\forall x\in\mathbb{R}^d$, there exists a positive measured set $\Gamma_{x,\epsilon}\subset B(0,\delta(\epsilon))$ with $\lim_{\epsilon\downarrow 0}\delta(\epsilon)= 0$, such that the r.v. uniformly distributed on $\Gamma_{x,\epsilon}$, denoted by $Y_{x,\epsilon}$, satisfies $\nabla f_{1,\epsilon}(x+Y_{x,\epsilon})\stackrel{w}{\longrightarrow}-\zeta$ uniformly with respect to $x$ as $\epsilon\rightarrow0$. Assume without loss of generality that $\mathbb{E}\zeta=0$ (nonzero mean can be absorbed into $f_0$).
	\label{cond_1}
\end{condition}
\paragraph{Notation:} Throughout this paper `$w$' means weak convergence: a sequence of random variables $\{X_n\}_{n=1}^\infty$ has a random variable $X$ as its weak limit, if and only if for any compactly supported test function $g\in\mathcal{C}^\infty(\mathbb{R}^d)$, $\mathbb{E}g(X_n)-\mathbb{E}g(X)\rightarrow0$ as $n\rightarrow\infty$.
\begin{condition} 
	\label{cond_2}$\epsilon\nabla^2f_{1,\epsilon}$ is uniformly bounded as $\epsilon\to 0$, and $\exists m\in\mathbb{R}$, s.t. for any bounded rectangle $\Gamma\subset\mathbb{R}^d$ whose area $|\Gamma|>0$, $\mathbb{E}\left[ \ln\|\epsilon\nabla^2f_{1,\epsilon}(U_\Gamma)\|_2 \right] \rightarrow m$, where $U_\Gamma$ is a uniform r.v. on $\Gamma$.
\end{condition} 

\begin{wrapfigure}{R}{0.3\textwidth}
\centering
    \vspace{-2pt}
    \includegraphics[width=0.3\textwidth]{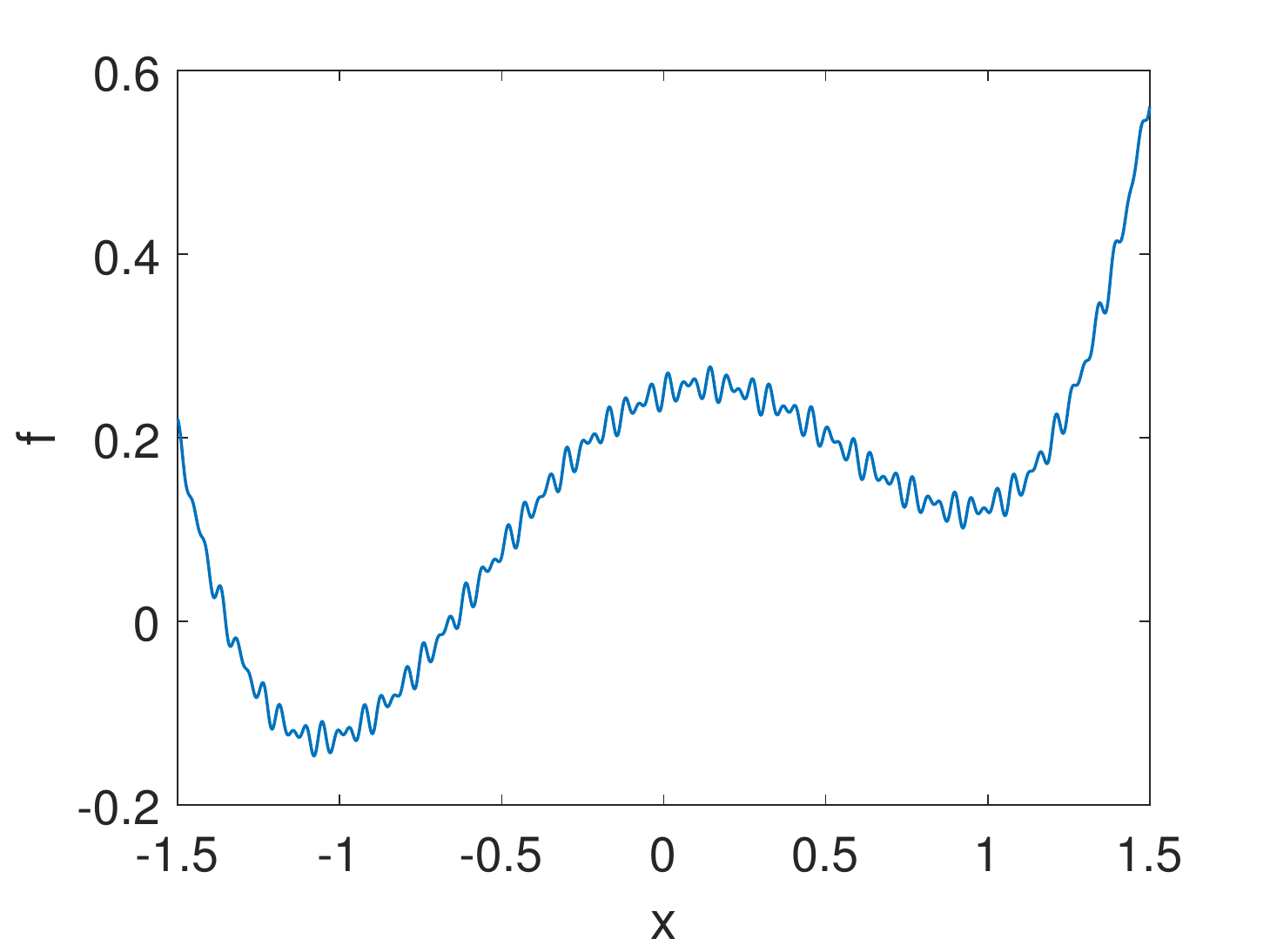}
    \vspace{-20pt}
    \caption{A multiscale function, $f(x)=(x^2-1)^2/4+x/8+\epsilon\left(\sin\left(x/\epsilon\right)+\sin\left(\sqrt{2}x/\epsilon\right)\right)$, $\epsilon=0.01$.}
    \label{fig_intro1}
    \vspace{-8pt}
\end{wrapfigure}

\begin{example}[periodic micro-scale]
\label{example_periodic}
    For intuition, consider a special case where  $f_{1,\epsilon}:=\epsilon f_1\left(\frac{x}{\epsilon}\right)$ for a periodic $f_1\in\mathcal{C}^2(\mathbb{R})$. It is easy to check that both conditions are satisfied.
\end{example}

\begin{example}[aperiodic micro-scale]
\label{example_aperiodic_new}
Given a $\mathcal{C}^2$ function $F(x_1,x_2,\cdots,x_N):\mathbb{R}^d \times \mathbb{R}^d \times \cdots \times \mathbb{R}^d \rightarrow \mathbb{R}$, that is periodic in each $x_i\in \mathbb{R}^d$, i.e., there exists constant vector $T\in\mathbb{R}^d$ such that $F(x_1,\cdots,x_i,\cdots,x_N)=F(x_1,\cdots,x_i+T,\cdots,x_N)$ for all $x_1,\cdots,x_N$ and $i=1,\cdots,N$. Then for any $\omega_1,\cdots,\omega_N \in \mathbb{R}$, $f_{1,\epsilon}(x):=\epsilon F(\frac{\omega_1 x}{\epsilon},\frac{\omega_2 x}{\epsilon},\cdots,\frac{\omega_N x}{\epsilon})$ satisfies Cond.\ref{cond_1} and \ref{cond_2}. If the $\omega$'s are nonresonant, meaning that the only solution to $z_1 \omega_1 + z_2 \omega_2 + \cdots + z_N \omega_N = 0$ for $z_i \in \mathbb{Z}$ is $z_1=z_2=\cdots=z_N=0$, then $f_{1,\epsilon}$ is not periodic. An example is $f_{1,\epsilon}=\epsilon(g_1(x/\epsilon)+g_2(\sqrt{2}x/\epsilon))$ for any 1-periodic $g_1$ and $g_2$.
\end{example}

\begin{remark}
Cond.\ref{cond_1} and \ref{cond_2} generalize and relax the periodic micro-scale requirement. Still required is, intuitively speaking, that every part of the small scale $f_{1,\epsilon}$ appears similar in a weak sense. In the special case of periodic micro-scale, it is easy to see $f_{1,\epsilon} = \mathcal{O}(\epsilon)$, $\nabla f_{1,\epsilon} = \mathcal{O}(1)$ and $\nabla^2 f_{1,\epsilon} = \mathcal{O}(\epsilon^{-1})$. However, after the relaxation of periodicity requirement, it may only be implied that $\nabla f_{1,\epsilon} = \mathcal{O}(1)$ (Cond.\ref{cond_1}) and $\nabla f_{1,\epsilon} = \mathcal{O}(\epsilon^{-1})$ (Cond.\ref{cond_2}).
Later on, Cond.\ref{cond_1} will help connect deterministic and stochastic maps, and Cond.\ref{cond_2} will help estimate the Lyapunov exponent so that the onset of chaos can be quantified.
\label{rmk_relaxedFromPeriodicity}
\end{remark}

Fig.\ref{fig_intro1} provides an example of $f$. This class of $f$ models objective landscapes that assume certain macroscopic shapes (described by $f_0$), but when zoomed-in exhibit additional small-in-$x \text{ and } f$ fluctuations (produced by $f_{1,\epsilon}$). Taking the loss function of a neural network as an example, our intuition is that if the training data is drawn from a distribution, the distribution itself produces the dominant macroscopic part of the landscape (i.e., $f_0$), and noises in the training data could lead to $f_{1,\epsilon}$ which corresponds to small and localized perturbations to the loss (see Appendix \ref{sec_MultiscaleLandscape} and also e.g., \cite{mei2018,jin2018local}).

Note although the length and height scales of $f_{1,\epsilon}$ can be both much smaller than those of $f_0$, $\nabla f_0$ and $\nabla f_{1,\epsilon}$ are nevertheless both $\mathcal{O}(1)$, creating nonconvexity and a large number of local minima even if $f_0$ is (strongly) convex.

What happens when gradient decent is applied to $f(x)$, following repeated applications of the map
\begin{equation*}
\varphi(x):=x-\eta\nabla f(x) = x-\eta \nabla f_0(x) - \eta \nabla f_{1,\epsilon}(x)?
\end{equation*}
($\eta$ will be called, interchangeably, learning rate (LR) or time step.)

When $\eta\ll\epsilon$, GD converges to a local minimum (or a saddle, or in general a stationary point where $\nabla f=0$). This is due to the well known convergence of GD when $\eta =o(1/L)$ for $L$-smooth $f$, and $L=\mathcal{O}(\epsilon^{-1})$ for our multiscale $f$'s (Rmk.\ref{rmk_relaxedFromPeriodicity}).

For  $\eta\gg1$, or more precisely when it exceeds $1/L_0$ for $L_0$-smooth $f_0$, the iteration generally blows up and does not converge. However, there is a regime in-between corresponding to $\epsilon \lesssim \eta\ll 1$, and this is what we call large LR, because here $\eta$ is too large to resolve the micro-scale (i.e., $f_{1,\epsilon}$, whose gradient has an $\mathcal{O}(\epsilon^{-1})$ Lipschitz constant).

\begin{figure}[h]
	\centering
	\includegraphics[width=0.8\linewidth]{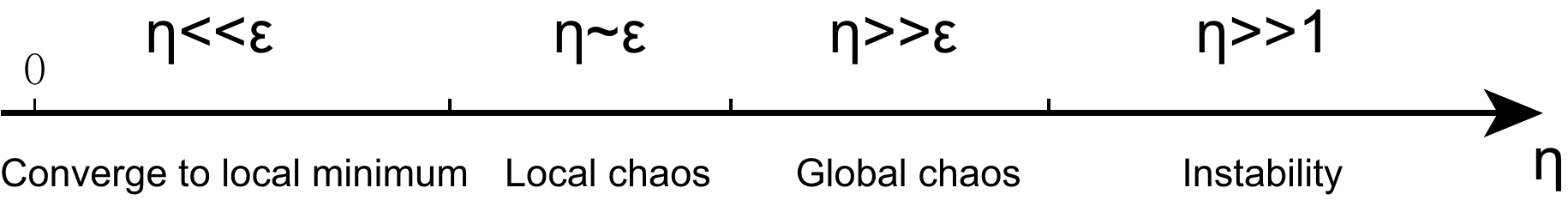}
	\caption{What happens as learning rate increases?}
	\label{etaChanges}
\end{figure}

Fig.\ref{etaChanges} previews what happens over the spectrum of $\eta$ values.
The difference between `local chaos' and `global chaos' will be detailed in Sec. \ref{sec_LYC} and \ref{sec_PD}.

In fact, for the multiscale function $f$, one may prefer to find a `macroscopic' local minimum created by $f_0$, instead of being trapped at one of the numerous local minima created by $f_{1,\epsilon}$, which could just be artifacts due to imperfection of training data. A small LR will not be able to do so, but we'll see below that large LR in some sense is better at this: it will lead GD to converge to a distribution peaked at $f_0$'s minimizer(s), despite that the iteration is based on the $\nabla f(x)=\nabla f_0(x)+\nabla f_{1,\epsilon}(x)$.

Our approach for demonstrating the `stochasticity' of $\varphi$ consists of three key ingredients: (i) construct another map $\hat{\varphi}$, which is a truly stochastic counterpart of $\varphi$, so that they share the same invariant distribution; (ii) find an approximation of the invariant distribution of $\hat{\varphi}$, namely rescaled Gibbs; (iii) establish conditions for $\varphi$ iterations to generate deterministic chaotic dynamics, which provides a route of convergence to a statistical distribution.

More specifically, we define the stochastic map $\hat\varphi$ as
\begin{equation*}
\hat\varphi: x \mapsto x-\eta\nabla f_0(x)+\eta\zeta,
\end{equation*}
where $\zeta$ is defined in Cond.\ref{cond_1}. Then we have (note many of these results persist in numerical experiments under relaxed conditions; see Sec.\ref{sec_numerics}).


\begin{theorem}[informal version of Thm.\ref{thm_sameLimitStats}]
	Fix $\eta$ and let $\epsilon\rightarrow0$. If $\varphi$ has a family of nondegenerate\footnote{By `nondegenerate', we require the distribution to be absolutely continuous w.r.t. Lebesgue measure. 
	Invariant distribution of $\varphi$ always exists; an example is a Dirac distribution concentrated at any stationary point of $f$. See Rmk.\ref{PurposeOfLip}.} invariant distributions for $\{\epsilon_i\}_{i=1}^\infty\rightarrow0$, which converges in the weak sense, then the weak limit is an invariant distribution of $\hat\varphi$.
    \label{thm_sameLimitStats_informal}
\end{theorem}

\begin{theorem}[informal version of Lem.\ref{GeoErgo}, Thm.\ref{thm_geomErgInW2} \& Thm.\ref{thm_rescaledGibbsIsApprx}] 
    Suppose $f_0\in \mathcal{C}^2$ is strongly convex and $L$-smooth, and $f_{1,\epsilon}\in\mathcal{C}^1$ satisfies condition \ref{cond_1}. Then for $\eta\leq C$ with some $C>0$ independent of $\epsilon$, $\hat\varphi$ has an unique invariant distribution, and its iteration converges exponentially fast to this distribution. Moreover, if the covariance matrix of $\zeta$ is isotropic, i.e., $\sigma^2 I_d$, then the rescaled Gibbs distribution $\frac{1}{Z}\exp\left(-\frac{2f_0(x)}{\eta\sigma^2}\right) dx$ is an $\mathcal{O}(\eta^2)$ approximation of it.
\end{theorem}

\begin{theorem}[informal version of Thm.\ref{thm_condition4LiYorke}]
    Suppose $f_0,f_{1,\epsilon} \in \mathcal{C}^1(\mathbb{R})$, $f_0$ is $L$-smooth, grows unboundedly at infinity, and $f_{1,\epsilon}$ satisfies Cond.\ref{cond_1}. If $f_0$ has a stationary point, then $\exists \eta_J>0$ such that for any fixed $0<\eta<\eta_J$, $\exists \epsilon_0>0$, s.t. when $\epsilon<\epsilon_0$, the $\varphi$ dynamics is chaotic.
\end{theorem}


In addition, we will show the onset of local chaos as $\eta$ increases is via the common route of period doubling  \citep{alligood1997chaos}. We will also establish and estimate the positive Lyapunov exponent of $\varphi$ in the large LR regime, which is strongly correlated with chaotic dynamics \citep{lyapunov1992general}.

The reason that we investigate chaos is the following. Although general theories are not unified yet, it is widely accepted that chaotic systems are often ergodic (on ergodic foliations), meaning the temporal average of an observable along any orbit (starting from the same foliation) converges, as the time horizon goes to infinity, to the spatial average of that observable over an invariant distribution (e.g., \cite{eckmann1985ergodic, young1998statistical, ott2002chaos}). Moreover, many chaotic systems are also mixing (see e.g., \cite{ott2002chaos}), which implies that if one starts with an ensemble of initial conditions and evolves each one of them by the deterministic map, then the whole ensemble converges to the (ergodic) invariant distribution.

Therefore, our last step in establishing stochasticity of GD is to show the deterministic $\varphi$ map becomes chaotic for large $\eta$. This way, in most situations it is also ergodic and the assumption of Theorem \ref{thm_sameLimitStats_informal} is satisfied, allowing us to demonstrate and quantify the stochastic behavior of deterministic GD. Note that we also know that if $f_0$ has multiple minima and associated potential wells, then GD can have stochastic behaviors with non-unique statistics (see Remark \ref{rmk_nonconvex}, \ref{rmk_nonconvex_chaos} and Section \ref{sec_NT_nonvonvex}). Therefore, mixing is not provable unless additional conditions are imposed, and this paper only presents numerical evidence (see section \ref{NT_Ergodicity} and \ref{sec_Matyas}). Meanwhile, note (i) since mixing implies ergodicity and Li-Yorke chaos \citep{akin2003li,iwanik1991independence}, our necessary conditions are also necessary for mixing, and (ii) proving mixing of deterministic dynamics is difficult, and only several examples have been well understood; see e.g., \cite{sinai1970dynamical, ornstein1973geodesic}.

\begin{remark}
    For these reasons, we clarify that the theory in this paper does not quantify the speed of convergence of deterministic GD ($\varphi$) to its long time statistical limit. It is only shown that the stochastic map $\hat{\varphi}$ converges to its statistical limit exponentially fast for strongly-convex $f_0$, and the deterministic map $\varphi$ shares the same statistical limit with $\hat{\varphi}$.
\end{remark}

\paragraph{Relevance to machine learning practices:} see Sec.\ref{sec_ExampleNeuralNetwork} (empirical) \& \ref{sec_MultiscaleLandscape} (theoretical) for examples.

\subsection{Related work}
\label{sec_relatedWork}
(S)GD is one of the most popular optimizing algorithms for deep learning, not only because of its practical performances, but also due to extensive and profound theoretical observations that it both optimizes well (e.g., \cite{lee2016gradient,jin2017escape,du2019gradient,du2019gradient2,allen2019convergence}) and generalizes well (e.g., \cite{neyshabur2015norm,bartlett2017spectrally,golowich2018size,dziugaite2017computing,arora2018stronger,li2018learning,li2018tighter,wei2019regularization,allen2019learning,neyshabur2019towards,cao2020generalization,ma2020comparative}).

However, to the best of our knowledge, there are not yet many results  that systematically study the effects of large learning rates from a general optimization perspective. \cite{jastrzkebski2017three} argue that large LR makes GD more likely to avoid sharp minima (we also note whether sharp minima correspond to worse generalization is questionable, e.g., \cite{dinh2017sharp}). Another result is \citep{li2019towards}, which suggests that large LR resists noises from data. In addition, \cite{smith2019super} associate large LR with faster training of neural networks. To relate to our work, note it can be argued from one of our results (namely the rescaled Gibbs statistical limit) that LR smooths out shallow and narrow local minima, which are likely created by noisy data. Therefore, it is consistent with \citep{li2019towards} and complementary to \citep{jastrzkebski2017three} and \citep{smith2019super}. At the same time, one of our contributions is the demonstration that this smoothing effect can be derandomized and completely achieved by deterministic GD. We also note a very interesting recent heuristic observation  \citep{lewkowycz2020large} consistent with our theory (see Fig.\ref{etaChanges}).

Another related result is \citep{draxler2018essentially}, which suggests that few substantial barriers appear in the loss landscape of neural networks, and this type of landscape fits our model, in which most potential wells are microscopic (i.e., shallow and narrow).

In addition, since we demonstrate stochasticity purely created by large LR, the technique of Polyak-Ruppert averaging \citep{polyak1992acceleration}
for reducing the variance and accelerating the convergence of SGD is expected to remain effective, even when no stochastic gradient or minibatch approximation is used. A systematic study of this possibility, however, is beyond the scope of this article. Also, our result is consistent with the classical decreasing LR treatment for SGD (e.g., \cite{robbins1951stochastic}) in two senses: (i) in the large LR regime, reducing LR yields smaller variance (eqn.\ref{eq_rescaledGibbs}); (ii) once the LR drops below the chaos threshold, GD simply converges to a local minimum (no more variance).

Regarding multiscale decomposition \eqref{eq_multiscaleLoss}, note many celebrated multiscale theories assume periodic small scale, (e.g., periodic homogenization \citep{pavliotis2008multiscale}), periodic averaging \citep{SaVeMu10}, and KAM theory \citep{moser1973stable}). We relaxed this requirement. Moreover, even when Conditions \ref{cond_1},\ref{cond_2} fail, our claimed result (stochasticity) persists as numerically observed  (see Sec.\ref{sec_nonperiodic}).

Another important class of relevant work is on continuum limits and modified equations, which Appendix \ref{sec_AppendixModifiedEq} will discuss in details.

\section{Theory}
\emph{Proofs and additional remarks are provided in Appendix \ref{sec_appendixProofs}.}

\subsection{Connecting the deterministic map and the stochastic map}
\label{sec_SB}
Here we will 
connect the stochastic map $\hat{\varphi}$ and the deterministic map $\varphi$. The intuition is that as $\epsilon\rightarrow0$ they share the same long-time behavior. In the following discussion, we fix the learning rate $\eta$, and in order to show the dependence of $\varphi$ on $\epsilon$, we write it as $\varphi_\epsilon$ explicitly in this section.
\begin{theorem}[convergence of the deterministic map to the stochastic map]
\label{thm_sameLimitStats}
Suppose $f_0$ is a $L$-smooth function and $f_{1,\epsilon}$ satisfies Cond.\ref{cond_1}. In order to show the dependence of $\varphi$ on $\epsilon$, $\varphi$ is written as $\varphi_\epsilon$ explicitly. Let $\hat\varphi(X):=X-\eta\nabla f_0(X)+\eta\zeta$ where $\zeta$ is the r.v. in Cond.\ref{cond_1}, i.i.d. if $\hat\varphi$ is iterated. 

	Assume there exist a set of random variables whose range is in $\mathbb{R}^d$, denoted by $\mathcal{F}$, and a subset $\mathcal{E}\subset \mathbb{R}$ with $0\in\bar{\mathcal{E}}\backslash \mathcal{E}$, satisfying:
	\begin{itemize}
		\item $\varphi_\epsilon$ is continuous in $\mathcal{F}$ in the weak sence $\forall\epsilon\in\mathcal{E}$. Namely, for any r.v. $X\in\mathcal{F}$ and for any sequence of r.v.'s $Y_n:\Omega\rightarrow\mathbb{R}^d$ satisfying $\|Y_n\|_\infty:=\sup_{\omega\in\Omega}\|Y_n(\omega)\|_2\rightarrow 0$, we have $\varphi_\epsilon(X+Y_n)\stackrel{w}{\longrightarrow}\varphi_\epsilon(X)$. (*)
	\end{itemize}

	Let $\{\epsilon_i\}_{i=1}^\infty\subset \mathcal{E}$ be a sequence with 0 limit and for each $i$, $X_{\epsilon_i}$ is a fixed point of $\varphi_{\epsilon_i}$. If $X_{\epsilon_i}\stackrel{w}{\longrightarrow}X$, then $X$ is a fixed point of $\hat{\varphi}$, i.e., $\hat{\varphi}(X)\overset{w}{=}X$.
\end{theorem}


\begin{remark}
\label{PurposeOfLip}
In this paper, invariant distributions that are absolutely continuous w.r.t. Lebesgue measure are called to be nondegenerate. Condition (*) implies nondegeneracy. 
We ruled out degenerate invariant distributions, which correspond to (convex combinations of) Dirac distributions at stationary points of $f$. In fact, if one starts GD with initial condition that is any stationary point of $f$, GD won't exhibit any true stochasticity no matter how large the LR is. We avoid considering such a degenerate limiting distribution by excluding them from our random variable space. 
\end{remark}

\begin{remark}
    If we further assume that all random variables in $\mathcal{F}$ have uniformly Lipschitz densities, the conclusion can be strengthened due to the sequential compactness of $\bar{\mathcal{F}}$: denote the set of fixed points of $\hat\varphi$ by $\hat{\mathcal{P}}\subset\bar{\mathcal{F}}$. Then the set of weak limit points of $\{X_{\epsilon_i}\}_{i=1}^\infty$, denoted by $\mathcal{P}\subset\bar{\mathcal{F}}$, is non-empty, and $\mathcal{P}\subset\hat{\mathcal{P}}$. 
\end{remark}

\subsection{The stochastic map: quantitative ergodicity}
\label{sec_rescaledGibbs}
This section will show that, when $f_0$ is strongly convex, the stochastic map $\hat\varphi$ induces a Markov process that is geometric ergodic, meaning it converges exponentially fast to a unique invariant distribution. We will also show that when $\zeta$ is isotropic, the invariant distribution can be approximated by a rescaled Gibbs distribution. As an additional remark, we also believe that rescaled Gibbs approximates the invariant distribution when $f_0$ is not strongly convex, even though no proof but only numerical evidence is provided (Sec.\ref{NT_Ergodicity}); however, geometric ergodicity can be lost.

\begin{lemma}[geometric ergodicity]
	\label{GeoErgo}
	Consider $\hat\varphi(x)=x-\eta\nabla f_0(x)+\eta\zeta$, where $\zeta$ is a bounded random variable in $\mathbb{R}^d$ with 0 mean, i.i.d. if $\hat\varphi$ is iterated. If $f_0$ is strongly convex and $L$-smooth, then there exists $\eta_0\in \mathbb{R}^+$, such that when $\eta<\eta_0$, the map $X\mapsto\hat\varphi(X)$ has a unique invariant distribution and the iteration $\hat\varphi^{(n)}(X)$ converges (as $n\rightarrow\infty$) to the invariant distribution in Prokhorov metric exponentially fast for any initial condition.
	\label{thm_geometricErgodicity}
\end{lemma}

\begin{proposition}[rescaled Gibbs nearly satisfies the invariance equation]
	Suppose $f_0\in \mathcal{C}^1(\mathbb{R}^d)$ is $L$-smooth. Consider $\hat\varphi$ defined in Lemma \ref{GeoErgo}. Suppose $\zeta$ is isotropic, i.e. with covariance matrix $\sigma^2 I_d$ for a scalar $\sigma$. Let $X_0$ be a random variable following rescaled Gibbs distribution
	\begin{equation}
	    X_0\sim\frac{1}{Z}\exp\left(-\frac{2f_0(x)}{\eta\sigma^2}\right)dx
	    \label{eq_rescaledGibbs}
	\end{equation}	Then for any $h\in \mathcal{C}^2$ with compact support, we have, for small enough $\eta$, that
	$$\mathbb{E}h(\hat{\varphi}(X_0))-\mathbb{E}h(X_0)=\mathcal{O}(\eta^3)$$
	\label{prop_rescaledGibbsError}
\end{proposition}

\begin{theorem}[rescaled Gibbs is an approximation of the invariant distribution]
	\label{thm_rescaledGibbsIsApprx}
	Assume $f_0 \in \mathcal{C}^2$ is strongly convex and L-smooth, and $\zeta$ is isotropic. Consider $\eta<\eta_0$ and denote by $\rho_\infty$ the density of the unique invariant distribution of $\hat{\varphi}$, whose existence and that of $\eta_0$ are given by Lemma \ref{GeoErgo}, then we have, in weak-* topology,
	\begin{equation}
	\rho_\infty = \tilde{\rho} + \mathcal{O}(\eta^2)
	\label{eq_rescaledGibbsIsApprx}
	\end{equation}
	where $\tilde{\rho}$ is rescaled Gibbs distribution with density $
	    \tilde{\rho}(x) = \frac{1}{Z}\exp\left(-\frac{2f_0(x)}{\eta\sigma^2}\right).
	$
\end{theorem}

\subsection{Deterministic map}
\label{sec_deterministic_map}
Since we want to link the invariant distributions of the deterministic map and the stochastic map, the existence of nondegenerate invariant distribution of the deterministic map (which is important, see Rmk.\ref{PurposeOfLip}) should be understood, as well as the convergence towards it. The last part of Sec.\ref{sec_mainResults} discussed that chaos can usually provide these properties, but it is not guaranteed, and mathematical tools are still lacking. Thus, in previous theorems, such existence was assumed instead of being proved. We first present two counter-examples to show that nondegenerate invariant distribution can actually be nonexistent. Details will be given in Thm. \ref{counter-example1} and \ref{counter-example2}. Both counter-examples are based on $f_{1,\epsilon}=\epsilon f_1(x/\epsilon)$ for some periodic $f_1$:
\begin{enumerate}
	\item In 1-dim, for any $f_1\in \mathcal{C}^2(\mathbb{R})$ and $\epsilon$, $\exists$ a convex $\mathcal{C}^2$ $f_0$ and an $\eta$ arbitrarily large, s.t. any orbit of $\varphi$ is bounded, but the invariant distribution has to be a fixed point (Thm.\ref{counter-example1})
	\item In 1-dim, for any $f_0\in \mathcal{C}^2(\mathbb{R})$ and $\eta$, $\exists$ a periodic $\mathcal{C}^2$ $f_1$ and an $\epsilon$ arbitrarily small, s.t. any orbit of $\varphi$ is bounded, but the invariant distribution has to be a fixed point (Thm.\ref{counter-example2})
\end{enumerate}
Then we show GD iteration is chaotic when LR is large enough (for nondegenerate $x_0$).

\subsubsection{Li-Yorke chaos}
\label{sec_LYC}
In this section, we fix $\eta$ in order to bound the small scale effect in simpler notations, and write the dependence of $\varphi$ on $\epsilon$ explicitly. The main message is $\varphi$ induces chaos in Li-Yorke sense. Note there are several definitions of chaos (e.g. \cite{block2006dynamics,devaney2018introduction,li1975period}, and  \cite{aulbach2001three} is a review of their relations). We quote Li-Yorke's celebrated theorem (\cite{li1975period}; see also \cite{sharkovskiui1995coexistence}) as Thm. \ref{LYThm} in appendix. Then we apply this tool to the GD map $\varphi$:

\begin{theorem}[sufficient condition for deterministic GD to be chaotic]
		Suppose $f_0,f_{1,\epsilon}\in \mathcal{C}^1(\mathbb{R})$, $f_{1,\epsilon}$ satisfies Cond.\ref{cond_1}, and $f_0$ is $L$-smooth, satisfying $f(x)\rightarrow+\infty$ when $|x|\rightarrow\infty$, $\lim_{x\rightarrow+\infty}f'(x)=+\infty$ and $\lim_{x\rightarrow-\infty}f'(x)=-\infty$. If 
		$\exists x$ s.t. $\nabla f_0(x)=0$, 
		then for any fixed $0<\eta<1/L$, $\exists \epsilon_0$, s.t. when $\epsilon<\epsilon_0$, $\varphi_\epsilon$ induces chaotic dynamics in Li-Yorke sense.
	\label{thm_condition4LiYorke}
\end{theorem}
\begin{remark}
    Here $\eta$ has an upper bound $\eta_J$, because when $\eta$ is too large, the iteration will be unstable and no interval $J$ closed under $\varphi_\epsilon$ exists (see Def. \ref{def_LY}). Rmk. \ref{rmk_nonconvex} gives an example on how $J$ depends on $\eta$.
\end{remark}
\begin{remark}
	Li-Yorke theory is restricted to 1D and Thm.\ref{thm_condition4LiYorke} cannot easily generalize to multi-dim. Lyapunov exponent in Sec.\ref{sec_LyapExp} however provides a hint and quantification for chaos in multi-dim.
\end{remark}
\begin{remark}
	The threshold $\epsilon_0$ may be dependent on the stationary point $x$, and thus $\epsilon_0$ obtained from an arbitrary $x$ may not be the largest threshold under which chaos onsets.
\end{remark}
\begin{remark}
    The threshold $\epsilon_0$ is only for local chaos to happen. In fact, as the proof will show, only very weak conditions are needed because here chaos onsets due to that GD evolving within a microscopic potential well is a unimodal map. See also Appendix.\ref{sec_PD}.
    
    However, as $\epsilon$ further decreases beyond the threshold, or equivalently as $\eta$ increases, global chaos onsets shortly after. The idea is, when there is only local chaos but not a global one, the empirical distribution of iterations concentrates at a local minimum inside a microscopic well, but its variance grows as $\eta$ increases. Shortly after, the distribution floods over the barriers of this microscopic well, and then local chaos transits into global chaos. Sec.\ref{sec_LyapExp} will allow us to see that both local and global chaos happen when $\eta\sim\epsilon$.

\end{remark}

\subsubsection{Lyapunov exponent}
\label{sec_LyapExp}
Lyapunov exponent characterizes how near-by trajectories deviate exponentially with the evolution time. A positive exponent shows sensitive dependence on initial condition, is often understood as a lack of predictability in the system (due to a standard argument that initial condition is never measured accurately), and is commonly associated with chaos. Strictly speaking it is only a necessary condition for chaos (see e.g., \cite{strogatz2018nonlinear} Chap 10.5), but it quantifies the strength of chaos.


Suppose $(x_0, x_1, ..., x_n, ...)$ is a trajectory of iterated map $\varphi$. Then the following measures the deviation of near-by orbits and thus defines the Lyapunov exponent:
\begin{equation}
    \lambda(x_0)=\lim_{n\rightarrow\infty}\frac{1}{n}\sum_{i=0}^{n-1}\ln||\nabla\varphi(x_i)||_2
    \label{LyapExpDef}
\end{equation}
This quantity is often independent of the initial condition (see e.g., \cite{oseledec1968multiplicative}), and we will see that this is true in numerical experiments with GD. 
We can quantitatively estimate $\lambda$: 
\begin{theorem}[approximate Lyapunov exponent of GD]
	\label{thm_LyapExp}
	Suppose $f_0$ and $f_1$ are both $\mathcal{C}^2$. Suppose the deterministic map is ergodic, and the small scaled effect $f_{1,\epsilon}$ satisfies Cond.\ref{cond_2}, then the Lyapunov exponent of the deterministic map starting from $x$, denoted by $\lambda(x)$, satisfies
	\[
	    \lim_{\eta \to 0} \lim_{\epsilon\to 0} \left( \lambda(x)-\ln\left(\frac{\eta}{\epsilon}\right) \right) =  m,
    \]
    where $m$ is the constant in Cond. \ref{cond_2}.
    
	In the special case when $f_1$ is periodic and $f_{1,\epsilon}(x)=\epsilon f_1(x/\epsilon)$, we have, in addition,
	$$\lambda(x)=m+\ln\left(\frac{\eta}{\epsilon}\right)+\mathcal{O}(\epsilon+\eta).$$
\end{theorem}

\begin{remark}
	A necessary condition for chaos is a positive Lyapunov exponent. From $\lambda(x)\approx m+\ln\left(\frac{\eta}{\epsilon}\right)$, we know the threshold for chaos satisfies $\eta>e^{-m}\epsilon$. This threshold 
	does not distinguish between local and global chaos, whose difference was hidden in the higher order term.
	\label{rmk_LyapunovExponent}
\end{remark}


\section{Numerical experiments}
\label{sec_numerics}

\emph{Additional results, such as verifications of statements about chaos (period doubling \& Lyapunov exponent estimation), nonconvex $f_0$, gradient descent with momentum, are in Appendix \ref{sec_moreNumerics}.}


\subsection{Stochasticity of deterministic GD: an example with periodic small scale}
\label{NT_Ergodicity}

Here we illustrate that GD dynamics is not only ergodic (on foliation) but also mixing, even when $f_0$ is not strongly convex but only convex (the strongly convex case was proved and will be illustrated in multi-dimension in Appendix \ref{sec_Matyas}). Recall ergodicity is the ability to follow an invariant distribution, and mixing ensures additional convergence to it. Fig.\ref{num_x4ensemble} shows that an arbitrary ensemble of initial conditions converges to approximately the rescaled Gibbs as the number of iteration increases. Fig.\ref{num_x4orbits} shows the empirical distribution of any orbit (i.e., $x_0,x_1,\cdots$ starting with an arbitrary $x_0$) also converges to the same limit. Fig.\ref{num_x4orbit1evolution} visualizes that any single orbit already appears `stochastic', even though the same initial condition would lead to exactly the same orbit.

\begin{figure}
	\centering
	\hfill
	\subfigure[\footnotesize{Evolution of an ensemble}\label{num_x4ensemble}]{\includegraphics[width=0.3\linewidth]{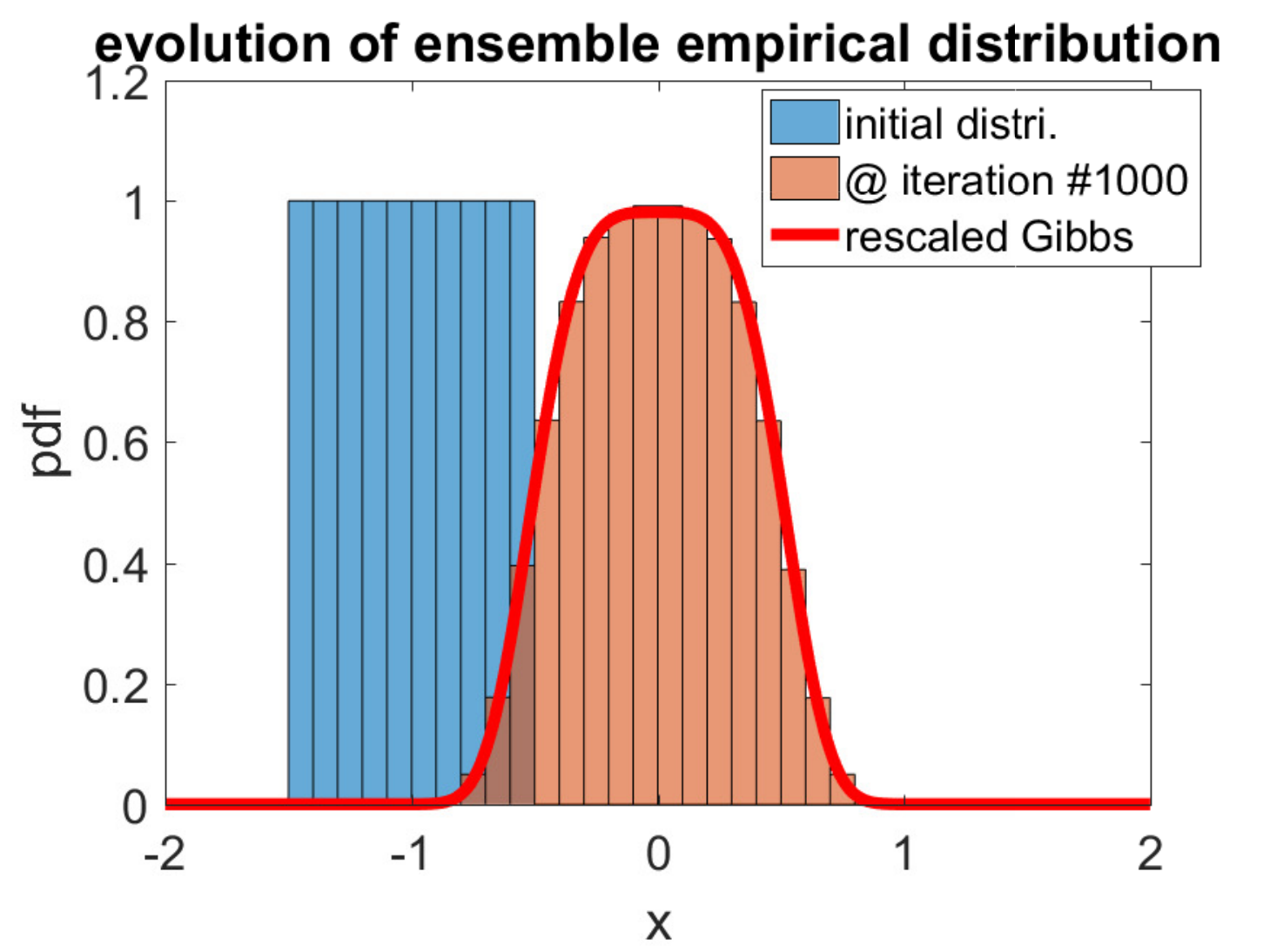}}
	\hfill
	\subfigure[\footnotesize{Empirical distrib. of 2 orbits}\label{num_x4orbits}]{\includegraphics[width=0.3\linewidth]{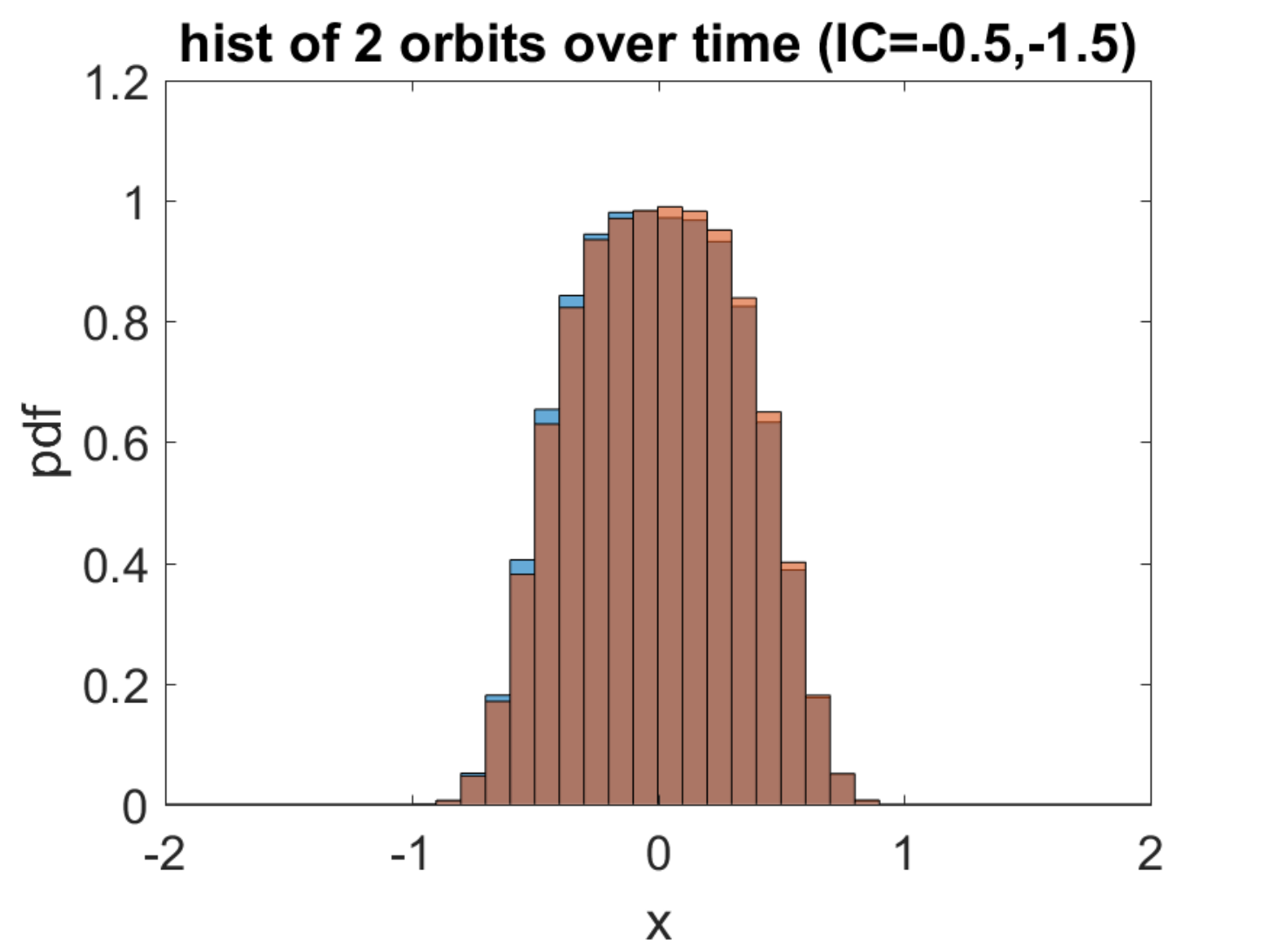}}
	\hfill
	\subfigure[\footnotesize{Stochasticity of an orbit}\label{num_x4orbit1evolution}]{\includegraphics[width=0.3\linewidth]{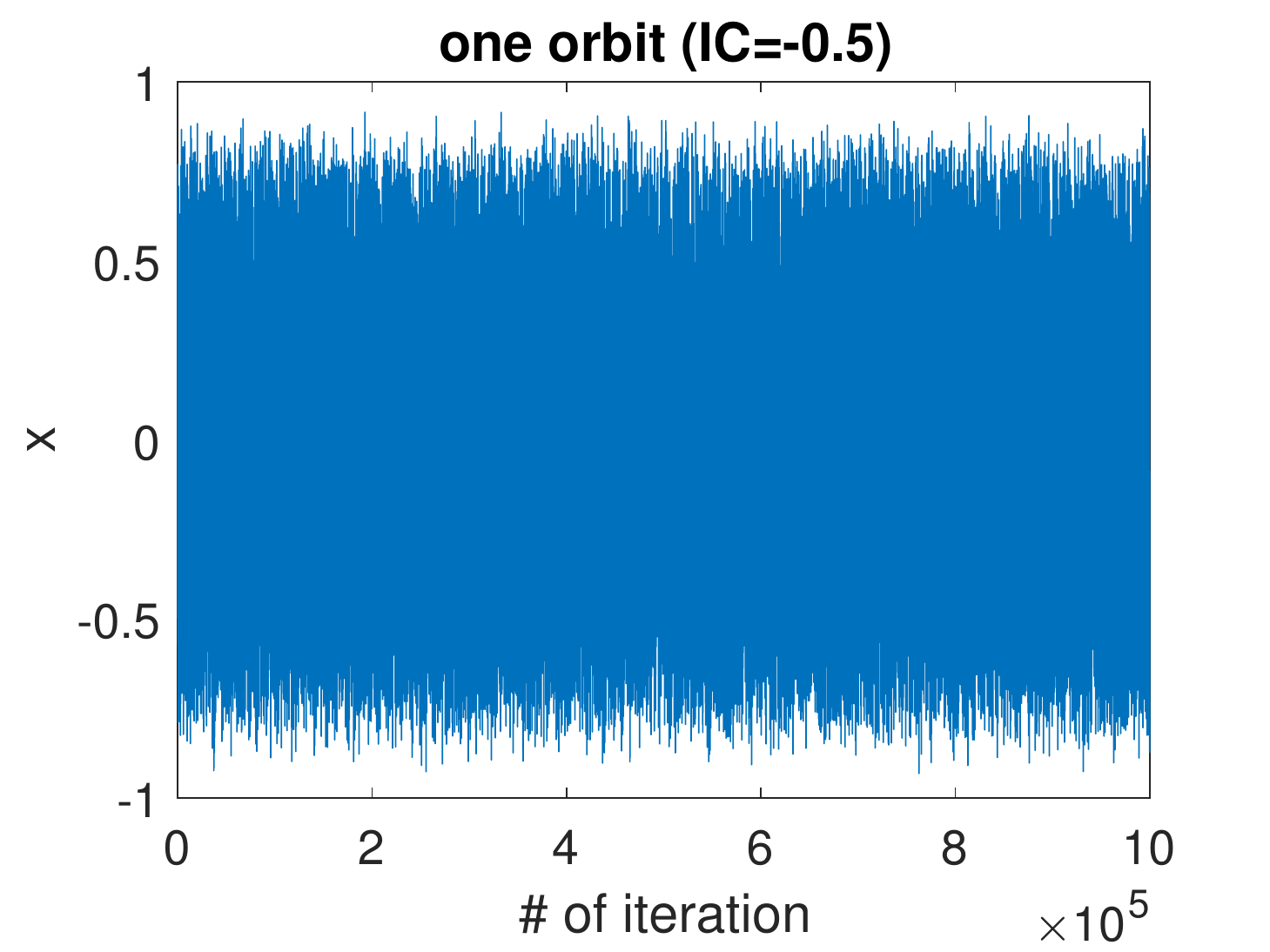}}
	\caption{Ergodicity and mixing of $\varphi$. $f_0=x^4/4$, $f_{1,\epsilon}(x)=\epsilon\sin(x/\epsilon)$ and $\eta=0.1$, $\epsilon=10^{-6}$.}
\end{figure}

\subsection{Stochasticity of deterministic GD: two examples with aperiodic small scales}
\label{sec_nonperiodic}

First consider an example whose small scale is not periodic, however satisfying Cond.\ref{cond_1} and \ref{cond_2}:
$f_0=x^4/4$, $f_{1,\epsilon}=\epsilon \sin(x/\epsilon)+\epsilon\sin(\sqrt{2}x/\epsilon)$. Fig. \ref{fig_quasiperiodic} shows that the system admits rescaled Gibbs as its invariant distribution (Thm. \ref{thm_rescaledGibbsIsApprx}) and is ergodic and mixing.
\begin{figure}
	\centering
	\hfill
	\subfigure[\footnotesize{Evolution of an ensemble}]{\includegraphics[width=0.3\linewidth]{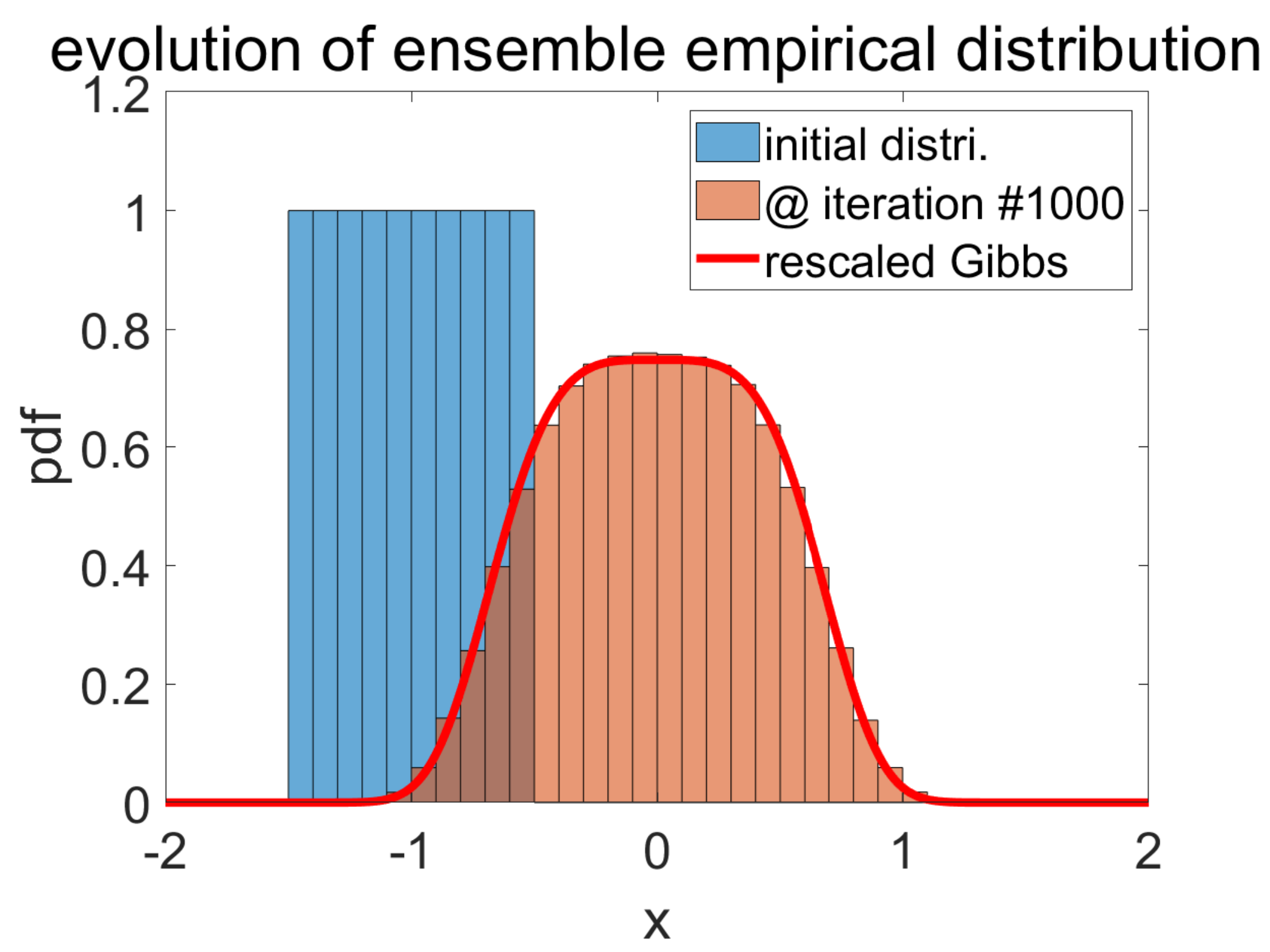}}
	\hfill
	\subfigure[\footnotesize{Empirical distrib. of an orbit}]{\includegraphics[width=0.3\linewidth]{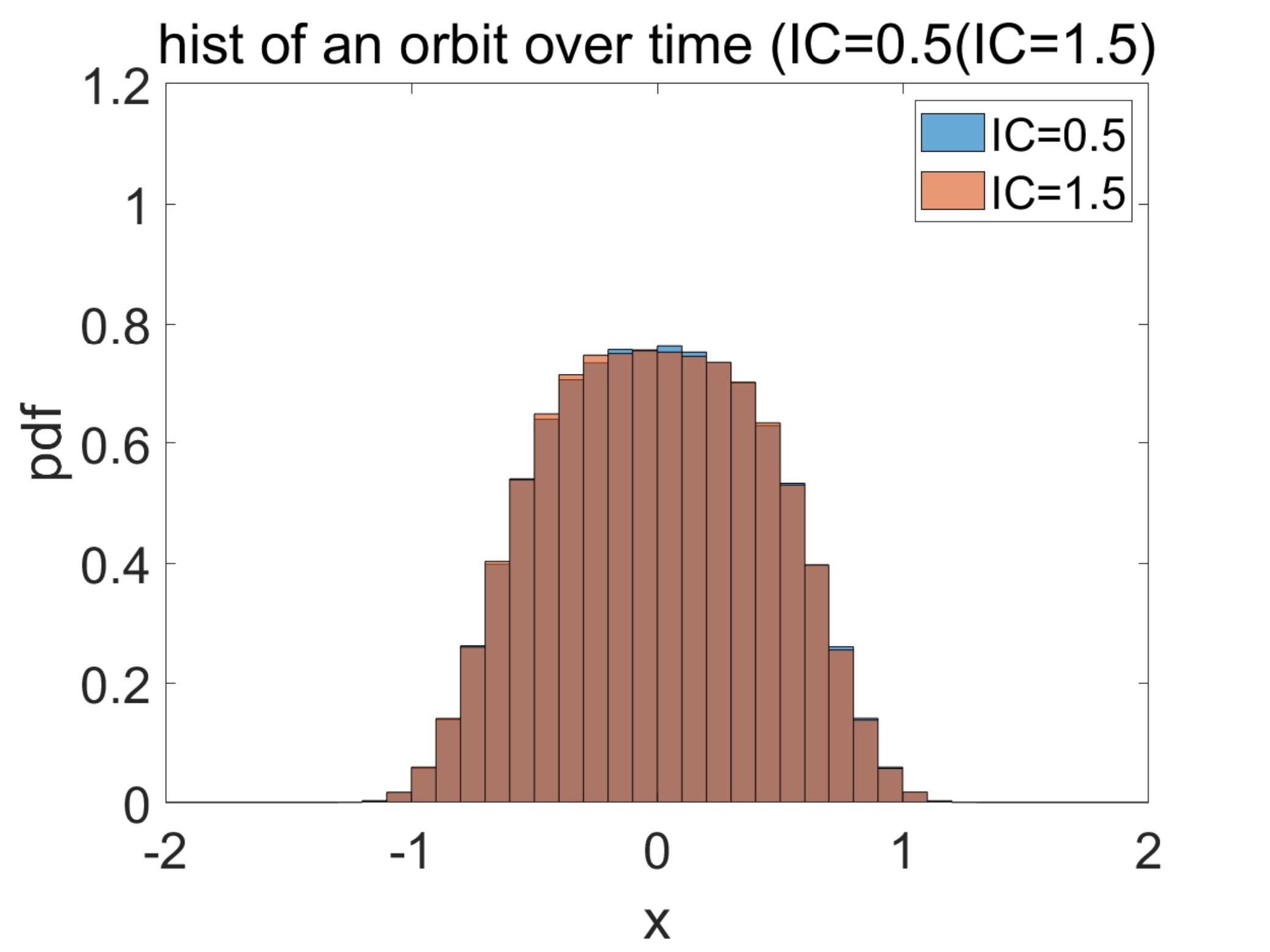}}
	\hfill
	\subfigure[\footnotesize{Iterations in an orbit}]{\includegraphics[width=0.3\linewidth]{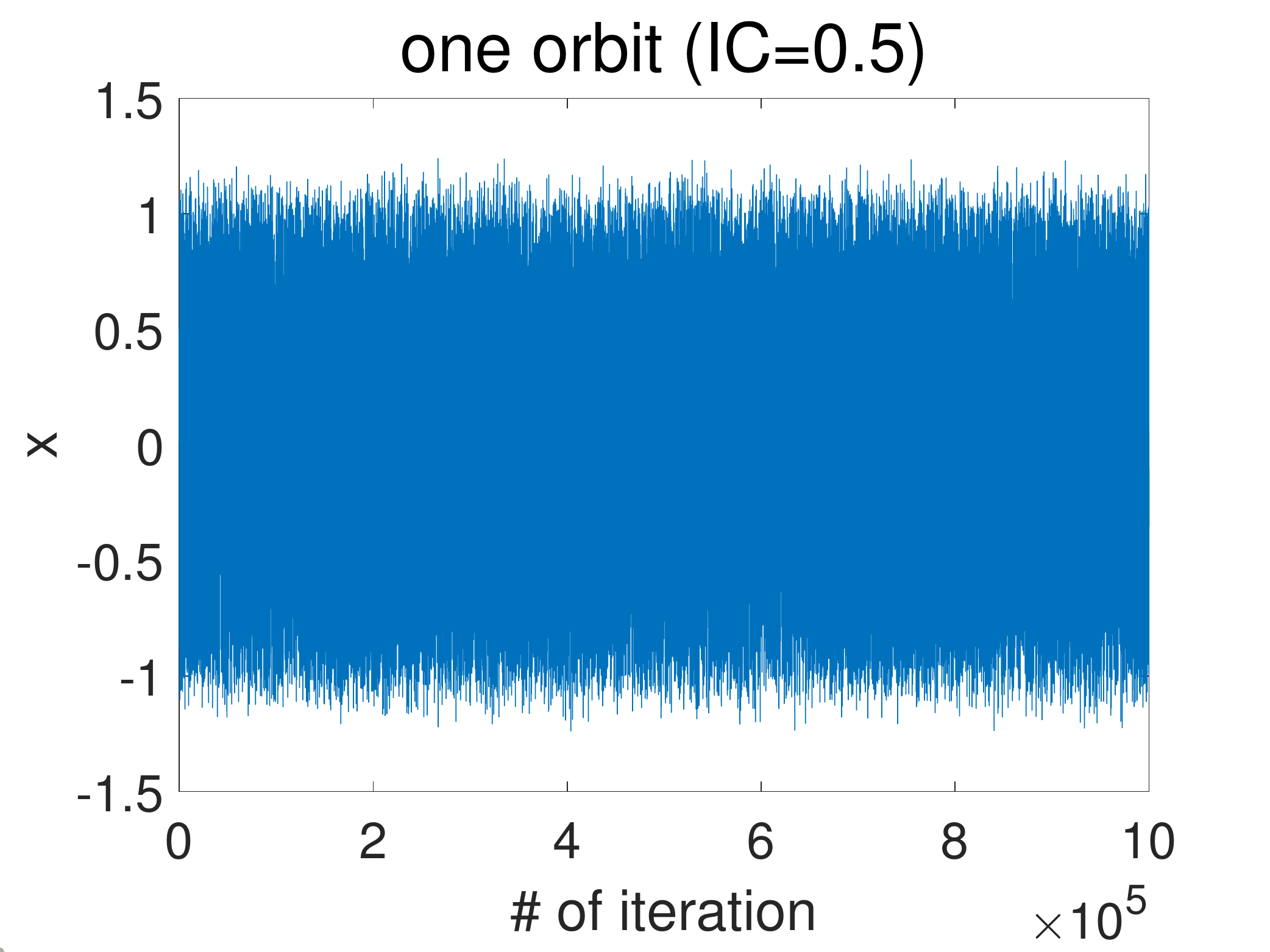}}
	\hfill
	\caption{Ergodicity and mixing of $\varphi$ for non-periodic $f_{1,\epsilon}$ given in Ex.\ref{example_aperiodic_new} with $\epsilon=10^{-6}$ and $\eta=0.1$.}
	\label{fig_quasiperiodic}
\end{figure}

Then we show, numerically, that stochastic behavior of large-LR-GD can persist even when Cond.\ref{cond_1} \& \ref{cond_2} fail. Here $f_0=x^2/2$ and $f_{1,\epsilon}(x)=\epsilon\cos(1+\cos(\frac{\sqrt{3}}{5}x)\frac{x}{\epsilon})$, the former the simplest, and the latter a made-up function that doesn't satisfy Cond.\ref{cond_1},\ref{cond_2} (due to that $\cos(\frac{\sqrt{3}}{5}x)/\epsilon$ can be 0). See Fig. \ref{fig_nonperiodic_generalization}. Note theoretically establishing local chaos (i.e., orbit filling a local potential well of $f_0+f_{1,\epsilon}$) is still possible, due to unimodal map's universality, e.g., \cite{strogatz2018nonlinear}; however, numerically observed is in fact global chaos, in which $f_1$ facilitates the exploration of the entire $f_0$ landscape.
\begin{figure}
	\centering
	\hfill
	\subfigure[\footnotesize{Evolution of an ensemble}]{\includegraphics[width=0.3\linewidth]{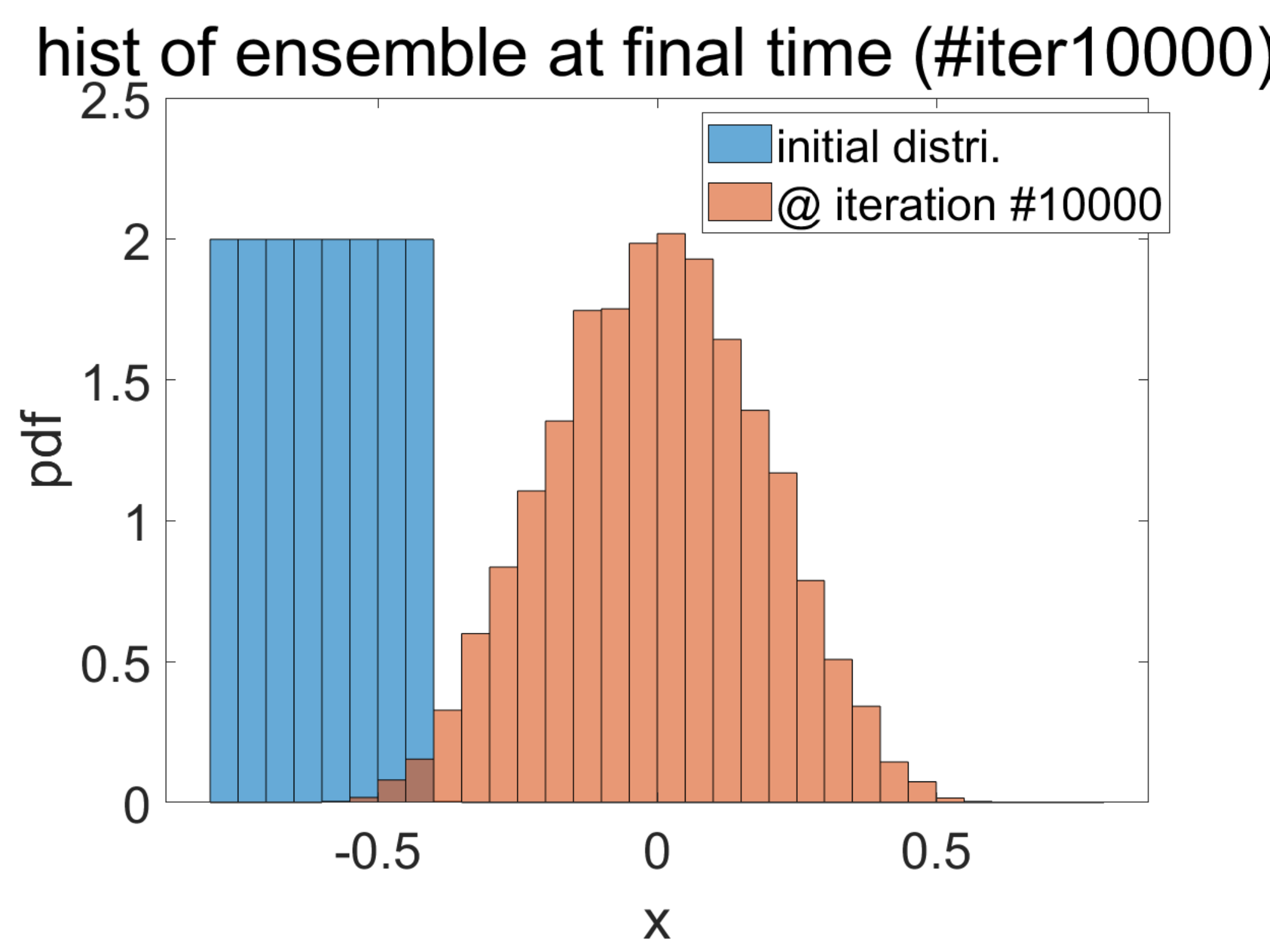}}
	\hfill
	\subfigure[\footnotesize{Empirical distrib. of an orbit}]{\includegraphics[width=0.3\linewidth]{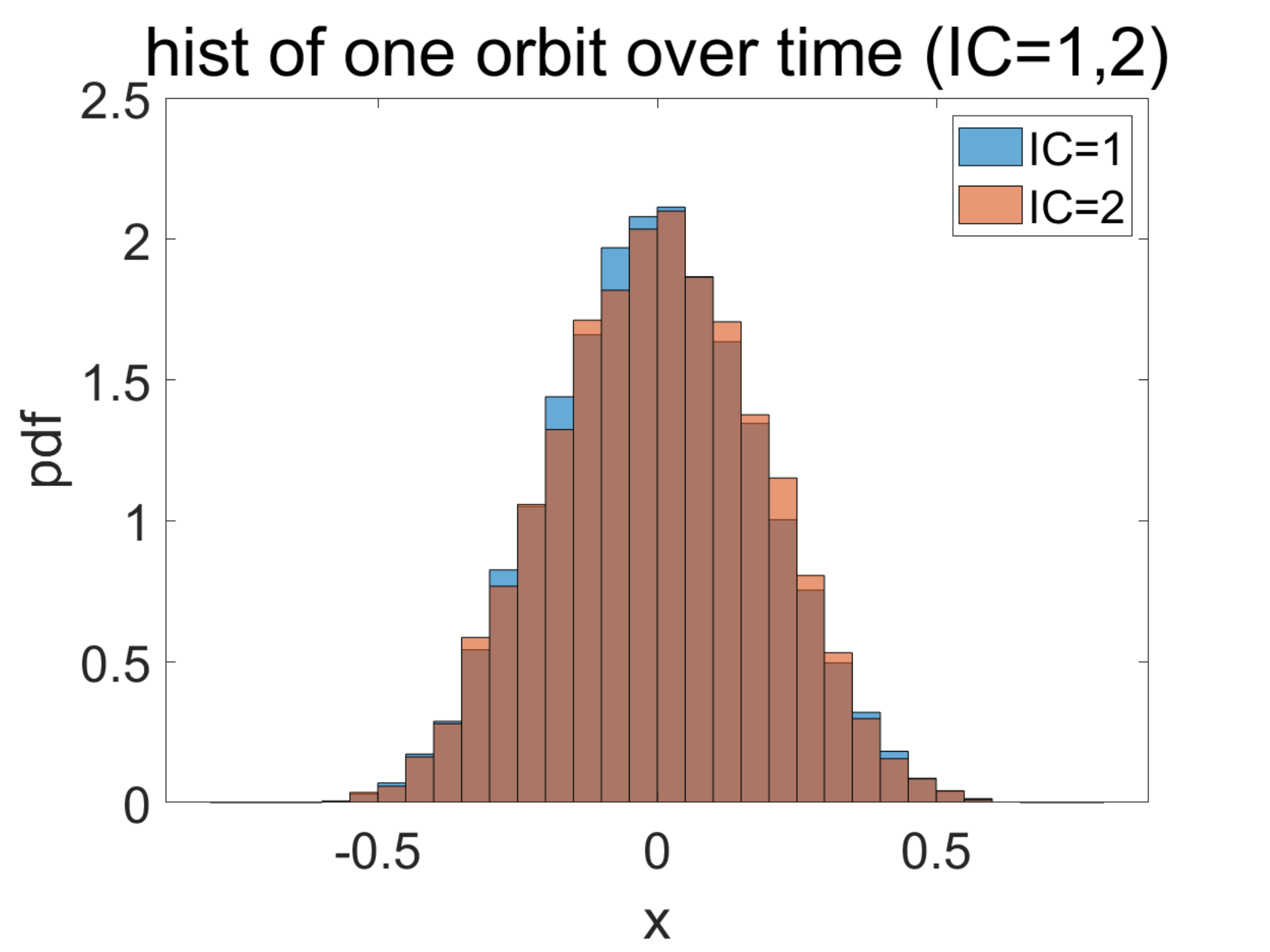}}
	\hfill
	\subfigure[\footnotesize{Stochasticity of an orbit}]{\includegraphics[width=0.3\linewidth]{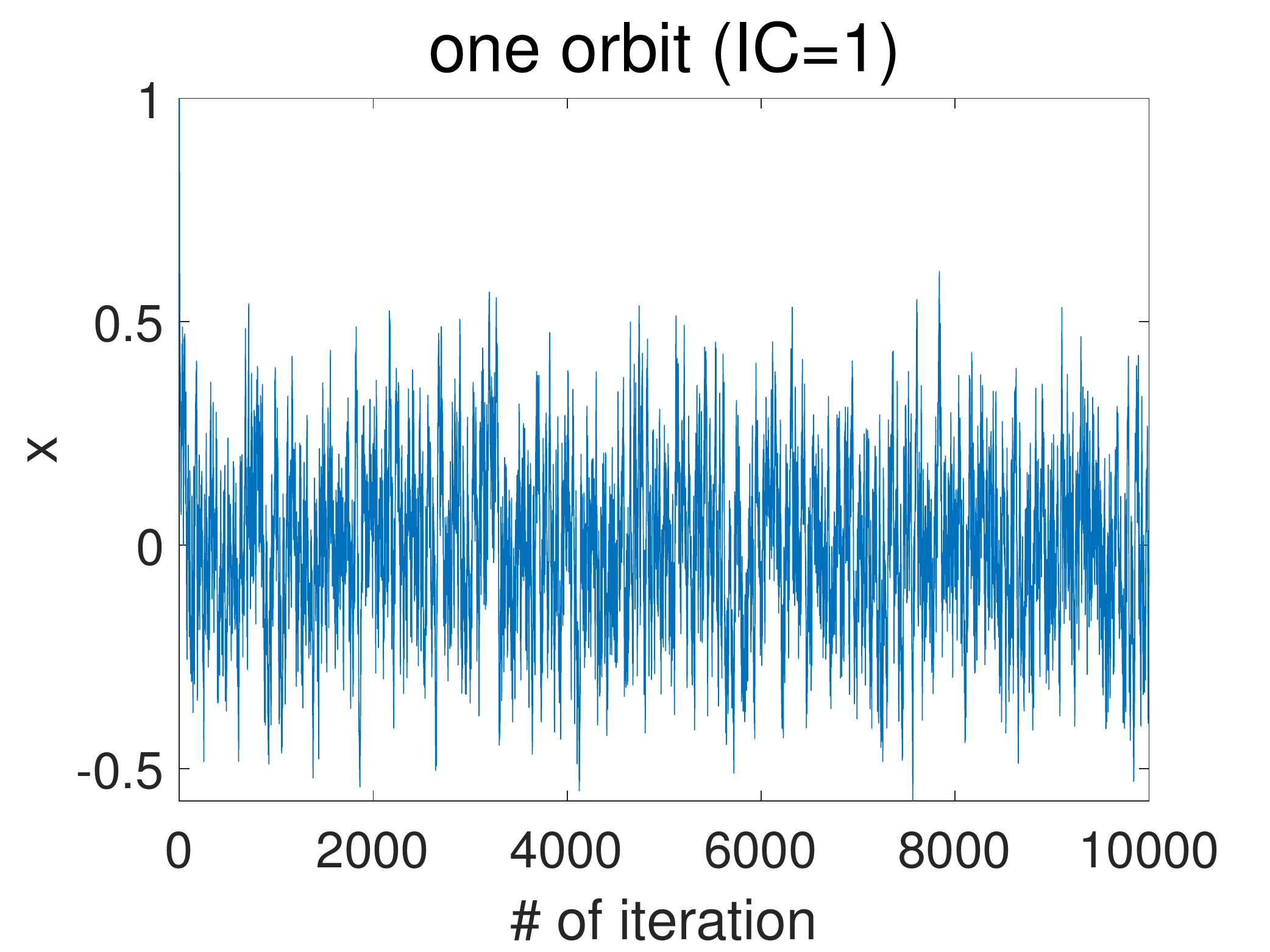}}
	\hfill
	\caption{Ergodicity and mixing of $\varphi$. Nonperiodic nor quasiperiodic 
	small scale. $\epsilon=10^{-4}$, $\eta=0.1$. }
	\label{fig_nonperiodic_generalization}
\end{figure}

\subsection{Stochasticity of deterministic GD: a neural network example}
\label{sec_ExampleNeuralNetwork}
To show that stochasticity can still exist in practical problems even when Cond.\ref{cond_1},\ref{cond_2} are hard to verify, we run a numerical test on a regression problem with a 2-layer neural network. We use a fully connected 5-16-1 MLP to regress UCI Airfoil Self-Noise Data Set \citep{Dua:2019}, with leaky ReLU activation, MSE as loss, and batch gradient. Fig.\ref{chaosInNN} shows large LR produces stochasticity and Fig.\ref{convergeNN} shows small LR doesn't, which are consistent with our study.

\begin{figure}
    \setlength{\belowcaptionskip}{-10pt}
    \centering
    \subfigure[loss orbit]{
    \includegraphics[width=0.3\textwidth]{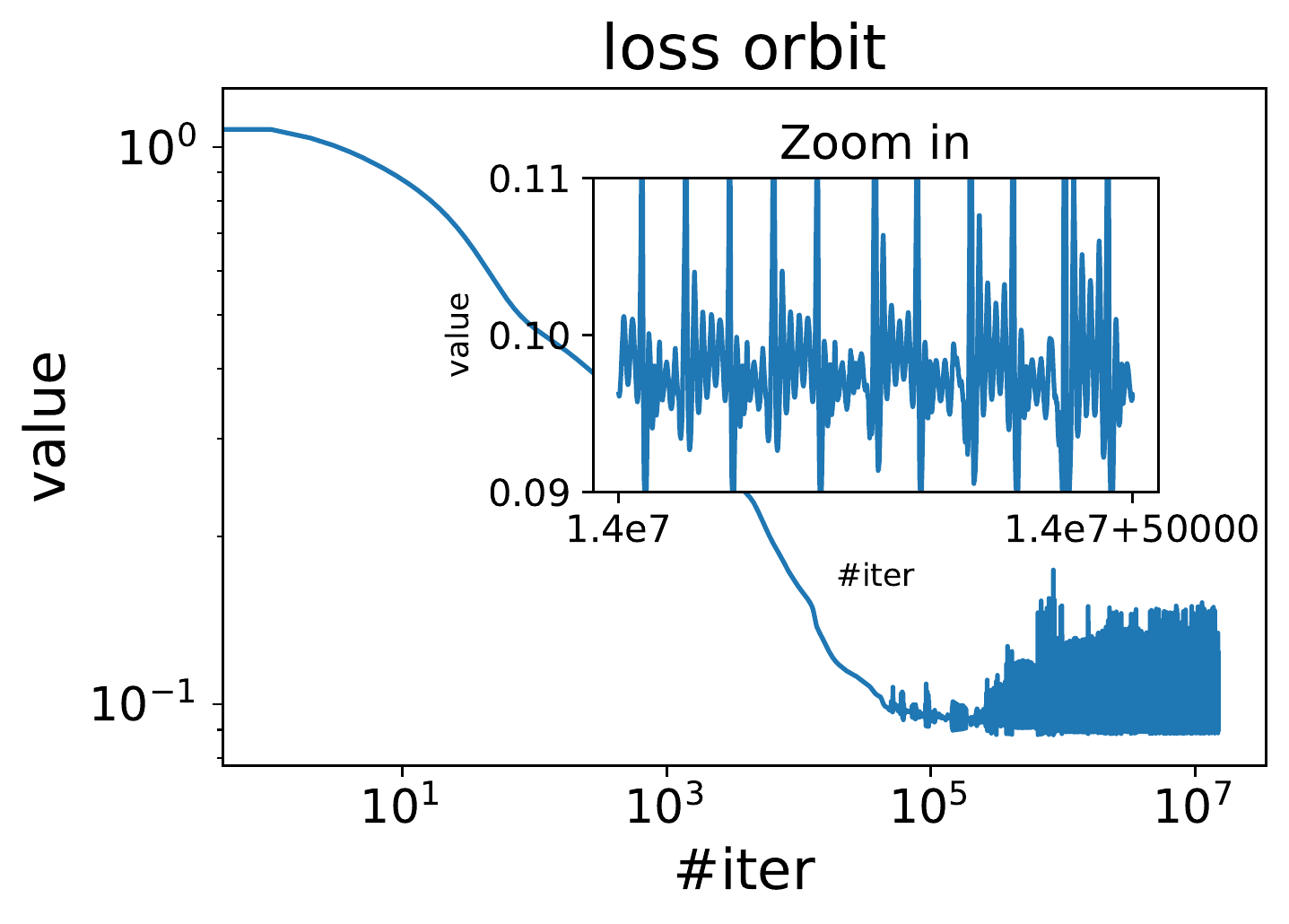}
    }
    \subfigure[orbit of a weight parameter]{
    \includegraphics[width=0.3\textwidth]{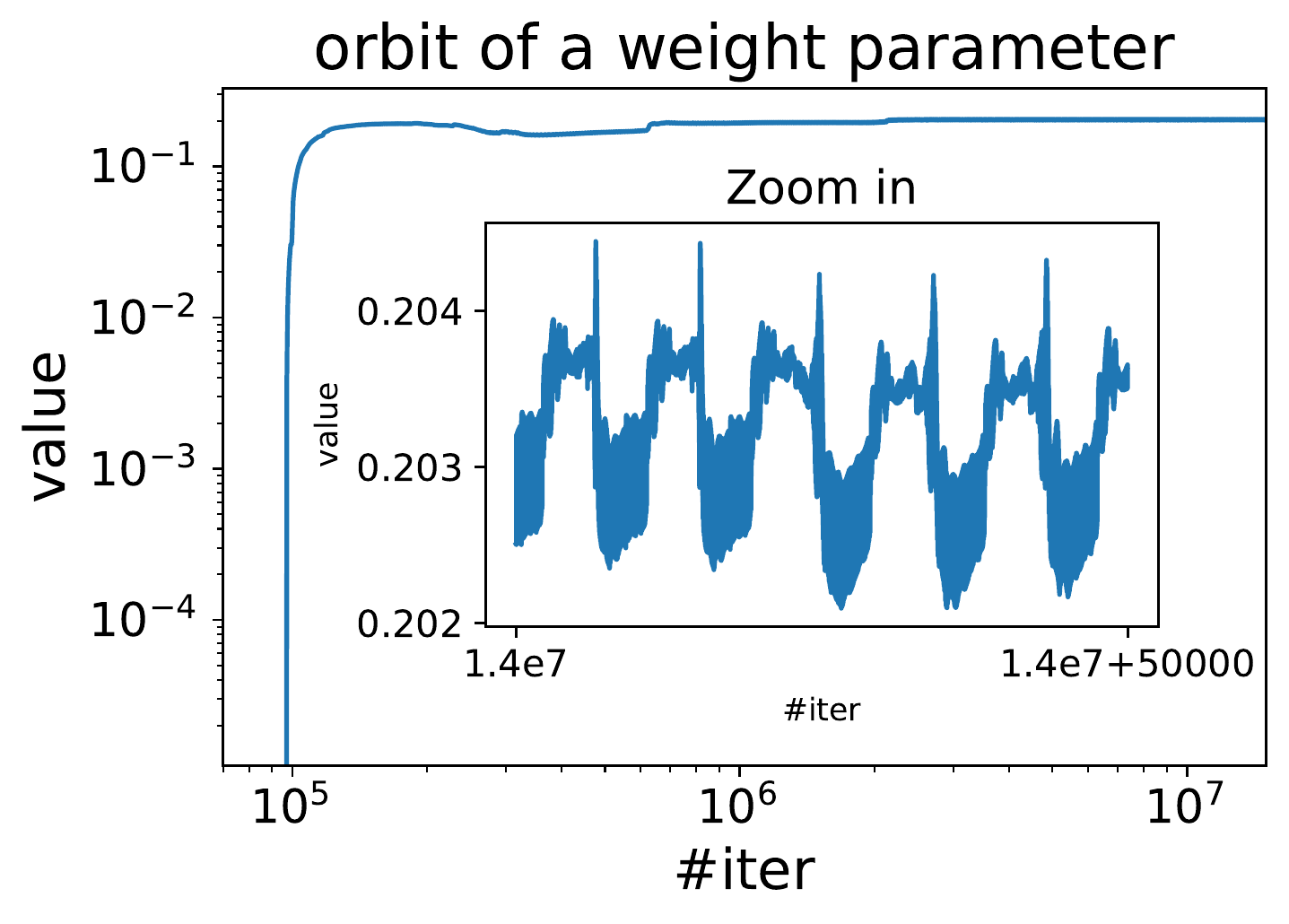}
    }
    \subfigure[orbit of a bias parameter]{
    \includegraphics[width=0.3\textwidth]{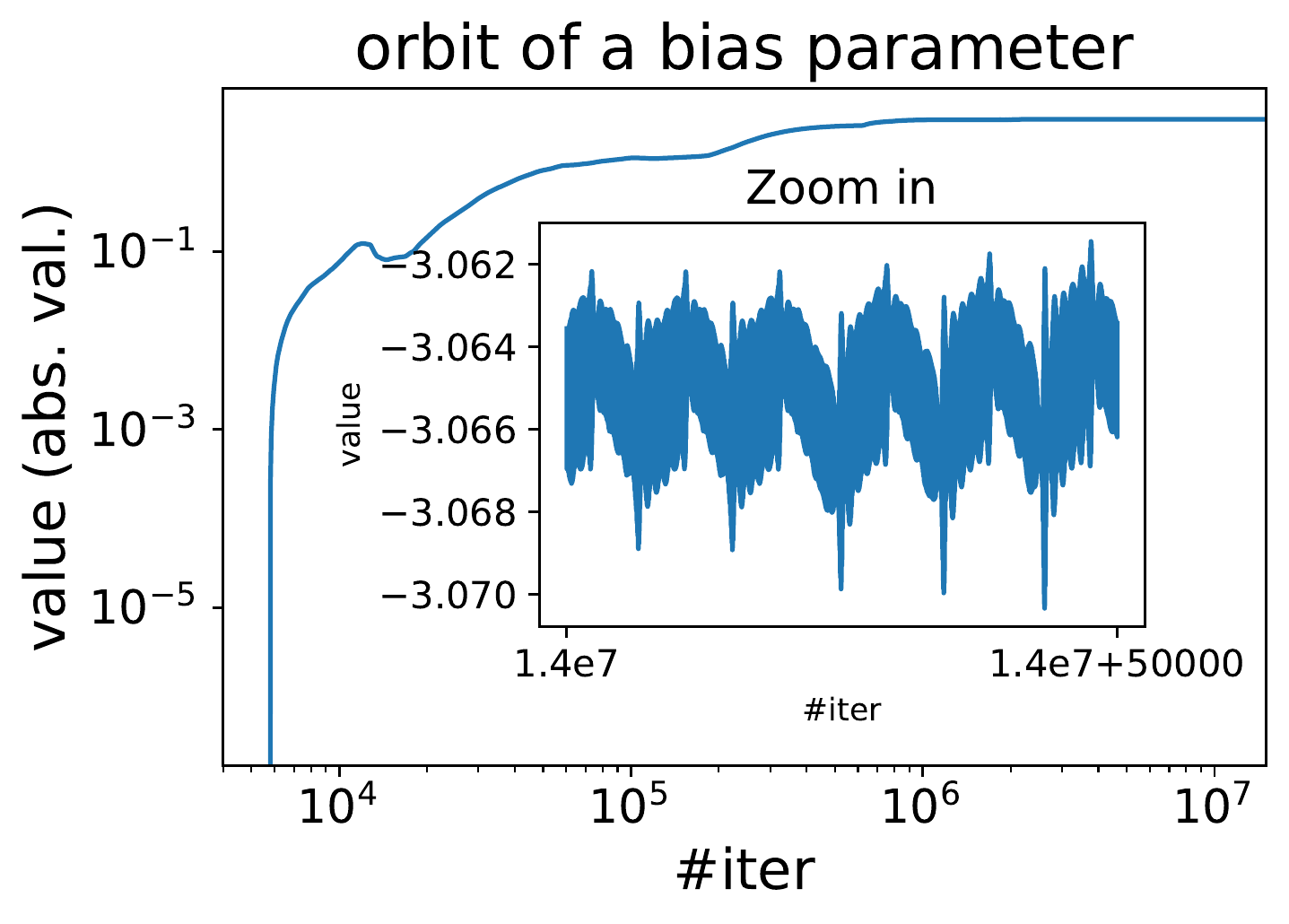}
    }
    \subfigure[loss histogram]{
    \includegraphics[width=0.3\textwidth]{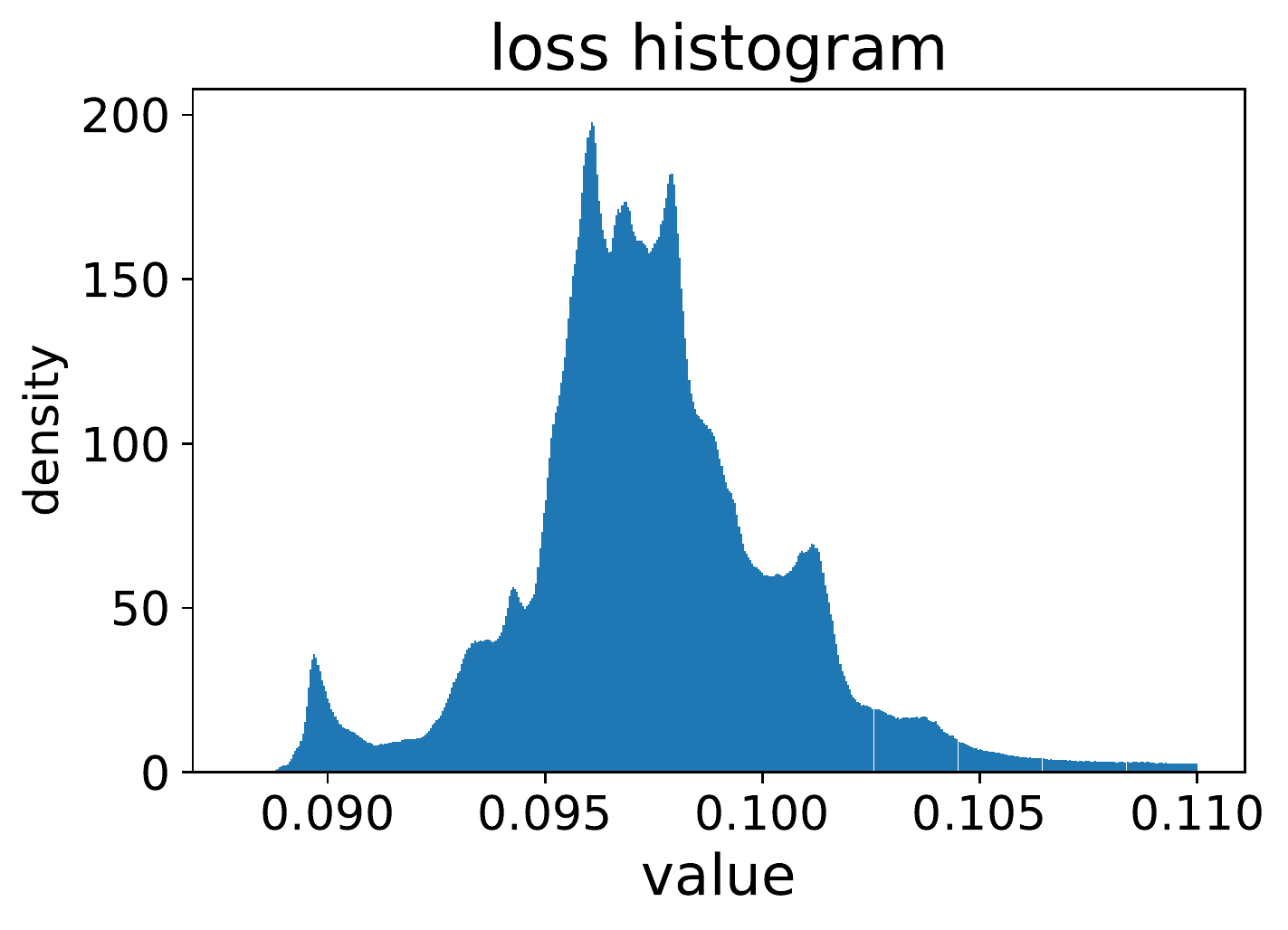}
    }
    \subfigure[hist of a weight parameter]{
    \includegraphics[width=0.3\textwidth]{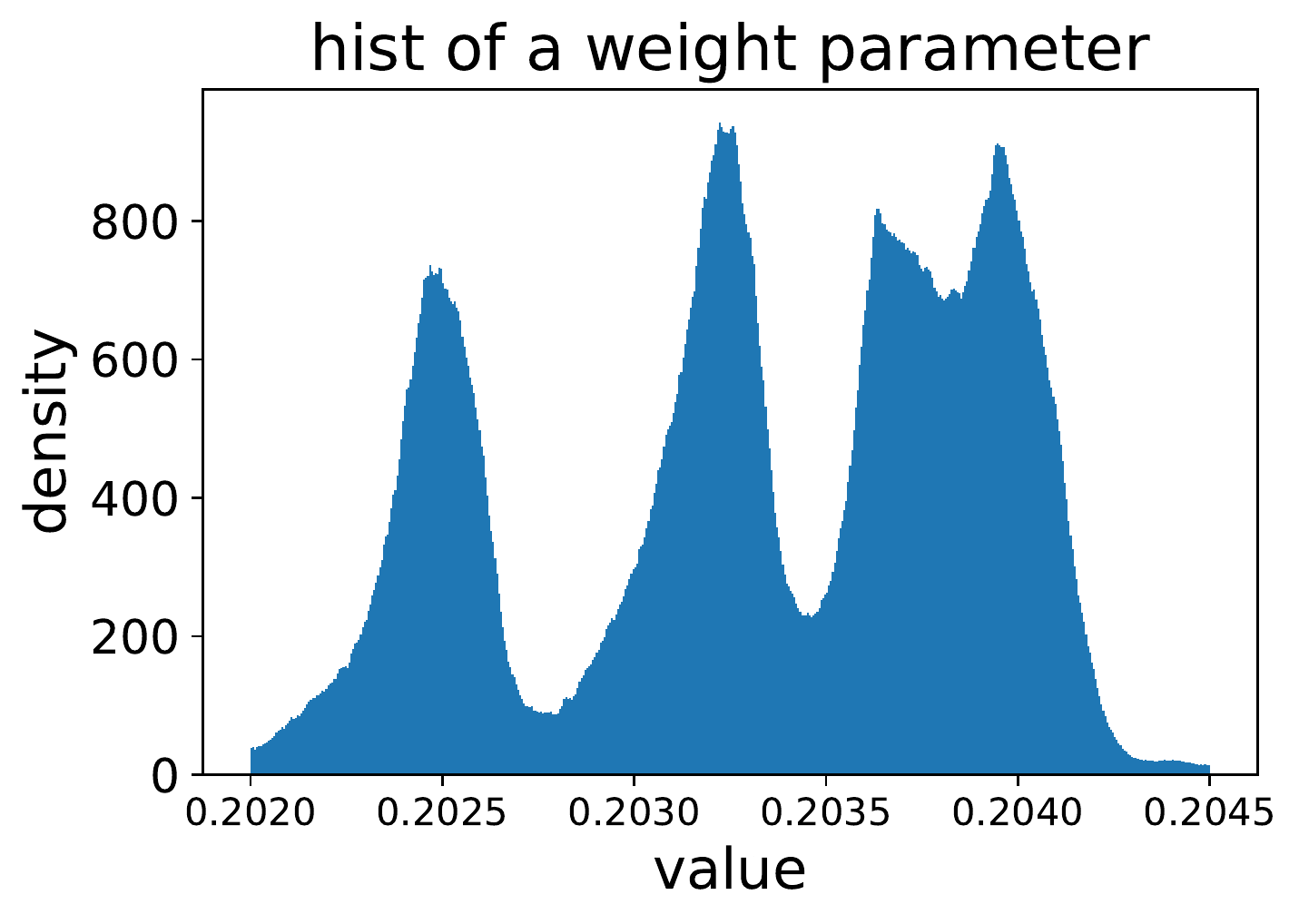}
    }
    \subfigure[hist of a bias parameter]{
    \includegraphics[width=0.3\textwidth]{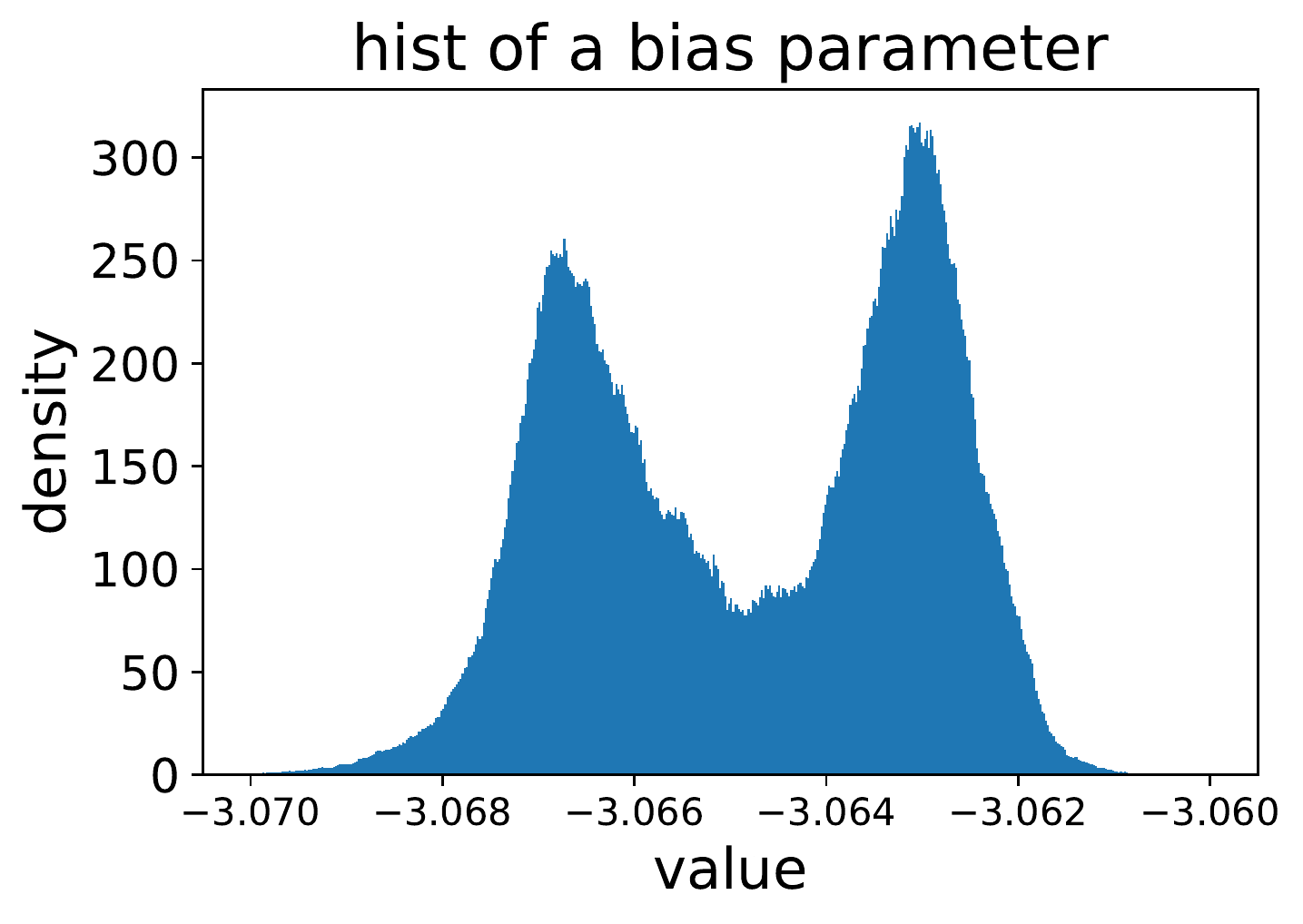}
    }
    \caption{LR=0.02 (large),which demonstrates stochasticity originated from chaos as GD converges to a statistical distribution rather than a local minimum.
    }
    \label{chaosInNN}
\end{figure}

\begin{figure}
    \setlength{\belowcaptionskip}{-10pt}
    \centering
    \includegraphics[width=0.3\textwidth]{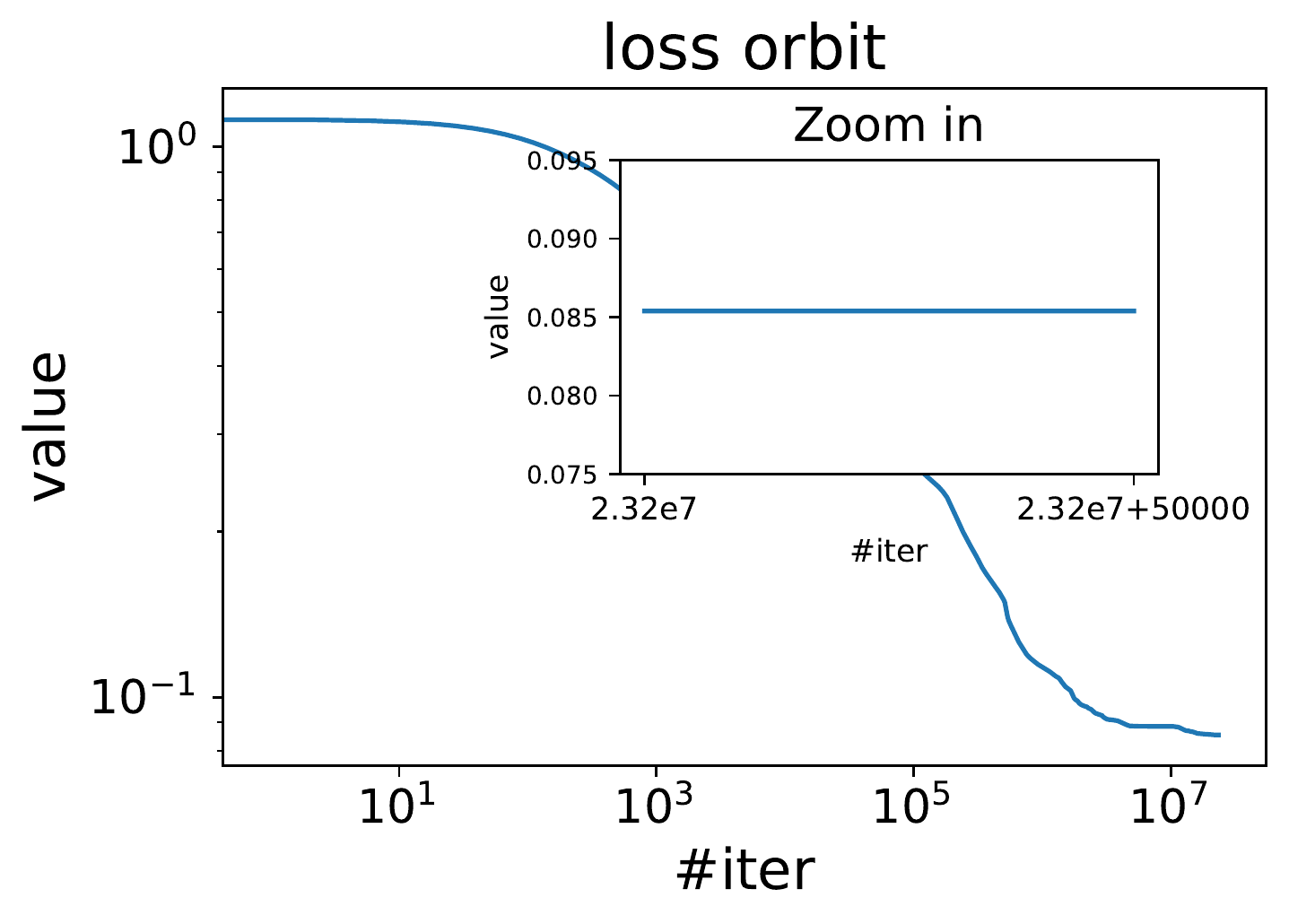}
    \caption{With the same loss function and initial condition, GD with LR=0.0005 (small) converges to a local minimum.
    }
    \label{convergeNN}
\end{figure}

\subsection{Persistence of stochasticity when momentum is added to GD}
Our theory is only for vanilla gradient decent, but also numerically observed is that deterministic GD with momentum still exhibits stochastic behaviors with large LR. See Appendix \ref{sec_withMomentum}.

\newpage
\section*{Broader Impact}
This theoretical work deepens our understanding of the performance of gradient descent, an optimization algorithm of significant importance to machine learning. This understanding could lead to the design of better optimization algorithms and improved learning models (either for encouraging or discouraging multiscale landscape, and for enabling or disabling stochasticity originated from determinism, depending on the application). It also helps tune the learning rate, and creates a new quantitative way for generating randomness (more precisely, sampling via determinism). Last but not least, analytical techniques developed and employed in this paper apply to a wide range of other problems.

\begin{ack}
This research was mainly conducted when LK was a visiting undergraduate student at Georgia Institute of Technology.
The authors thank 
Jacob Abernethy, Fryderyk Falniowski, Ruilin Li, and Tuo Zhao
for helpful discussions.
MT was partially supported by NSF DMS-1847802 and ECCS-1936776.
\end{ack}

\bibliographystyle{plainnat}
\bibliography{ref}

\begin{thebibliography}{76}
\providecommand{\natexlab}[1]{#1}
\providecommand{\url}[1]{\texttt{#1}}
\expandafter\ifx\csname urlstyle\endcsname\relax
  \providecommand{\doi}[1]{doi: #1}\else
  \providecommand{\doi}{doi: \begingroup \urlstyle{rm}\Url}\fi

\bibitem[Akin and Kolyada(2003)]{akin2003li}
Ethan Akin and Sergii Kolyada.
\newblock Li--yorke sensitivity.
\newblock \emph{Nonlinearity}, 16\penalty0 (4):\penalty0 1421, 2003.

\bibitem[Allen-Zhu et~al.(2019{\natexlab{a}})Allen-Zhu, Li, and
  Liang]{allen2019learning}
Zeyuan Allen-Zhu, Yuanzhi Li, and Yingyu Liang.
\newblock Learning and generalization in overparameterized neural networks,
  going beyond two layers.
\newblock In \emph{Advances in neural information processing systems}, pages
  6155--6166, 2019{\natexlab{a}}.

\bibitem[Allen-Zhu et~al.(2019{\natexlab{b}})Allen-Zhu, Li, and
  Song]{allen2019convergence}
Zeyuan Allen-Zhu, Yuanzhi Li, and Zhao Song.
\newblock A convergence theory for deep learning via over-parameterization.
\newblock In \emph{International Conference on Machine Learning}, pages
  242--252, 2019{\natexlab{b}}.

\bibitem[Alligood et~al.(1997)Alligood, Sauer, and Yorke]{alligood1997chaos}
Kathleen~T Alligood, Tim~D Sauer, and James~A Yorke.
\newblock Chaos: An introduction to dynamical systems. 1996, 1997.

\bibitem[Arora et~al.(2018)Arora, Ge, Neyshabur, and Zhang]{arora2018stronger}
Sanjeev Arora, R~Ge, B~Neyshabur, and Y~Zhang.
\newblock Stronger generalization bounds for deep nets via a compression
  approach.
\newblock In \emph{35th International Conference on Machine Learning, ICML
  2018}, 2018.

\bibitem[Aulbach and Kieninger(2001)]{aulbach2001three}
Bernd Aulbach and Bernd Kieninger.
\newblock On three definitions of chaos.
\newblock \emph{Nonlinear Dyn. Syst. Theory}, 1\penalty0 (1):\penalty0 23--37,
  2001.

\bibitem[Bartlett et~al.(2017)Bartlett, Foster, and
  Telgarsky]{bartlett2017spectrally}
Peter~L Bartlett, Dylan~J Foster, and Matus~J Telgarsky.
\newblock Spectrally-normalized margin bounds for neural networks.
\newblock In \emph{Advances in Neural Information Processing Systems}, pages
  6240--6249, 2017.

\bibitem[Block and Coppel(2006)]{block2006dynamics}
Louis~S Block and William~A Coppel.
\newblock \emph{Dynamics in one dimension}.
\newblock Springer, 2006.

\bibitem[Borkar and Mitter(1999)]{borkar1999strong}
Vivek~S Borkar and Sanjoy~K Mitter.
\newblock A strong approximation theorem for stochastic recursive algorithms.
\newblock \emph{Journal of optimization theory and applications}, 100\penalty0
  (3):\penalty0 499--513, 1999.

\bibitem[Cao and Gu(2020)]{cao2020generalization}
Yuan Cao and Quanquan Gu.
\newblock Generalization error bounds of gradient descent for learning
  overparameterized deep {ReLU} networks.
\newblock \emph{AAAI}, 2020.

\bibitem[Chae et~al.(2017)Chae, Walker, et~al.]{chae2017novel}
Minwoo Chae, Stephen~G Walker, et~al.
\newblock A novel approach to bayesian consistency.
\newblock \emph{Electronic Journal of Statistics}, 11\penalty0 (2):\penalty0
  4723--4745, 2017.

\bibitem[Cvitanovic(2017)]{cvitanovic2017universality}
Predrag Cvitanovic.
\newblock \emph{Universality in chaos}.
\newblock Routledge, 2017.

\bibitem[Devaney(2018)]{devaney2018introduction}
Robert Devaney.
\newblock \emph{An introduction to chaotic dynamical systems}.
\newblock CRC Press, 2018.

\bibitem[Dinh et~al.(2017)Dinh, Pascanu, Bengio, and Bengio]{dinh2017sharp}
Laurent Dinh, Razvan Pascanu, Samy Bengio, and Yoshua Bengio.
\newblock Sharp minima can generalize for deep nets.
\newblock In \emph{Proceedings of the 34th International Conference on Machine
  Learning-Volume 70}, pages 1019--1028. JMLR. org, 2017.

\bibitem[Draxler et~al.(2018)Draxler, Veschgini, Salmhofer, and
  Hamprecht]{draxler2018essentially}
Felix Draxler, Kambis Veschgini, Manfred Salmhofer, and Fred~A Hamprecht.
\newblock Essentially no barriers in neural network energy landscape.
\newblock \emph{arXiv preprint arXiv:1803.00885}, 2018.

\bibitem[Du et~al.(2019{\natexlab{a}})Du, Lee, Li, Wang, and
  Zhai]{du2019gradient2}
Simon Du, Jason Lee, Haochuan Li, Liwei Wang, and Xiyu Zhai.
\newblock Gradient descent finds global minima of deep neural networks.
\newblock In \emph{International Conference on Machine Learning}, pages
  1675--1685, 2019{\natexlab{a}}.

\bibitem[Du et~al.(2018)Du, Hu, and Lee]{du2018algorithmic}
Simon~S Du, Wei Hu, and Jason~D Lee.
\newblock Algorithmic regularization in learning deep homogeneous models:
  Layers are automatically balanced.
\newblock In \emph{Advances in Neural Information Processing Systems}, pages
  384--395, 2018.

\bibitem[Du et~al.(2019{\natexlab{b}})Du, Zhai, Poczos, and
  Singh]{du2019gradient}
Simon~S Du, Xiyu Zhai, Barnabas Poczos, and Aarti Singh.
\newblock Gradient descent provably optimizes over-parameterized neural
  networks.
\newblock \emph{ICLR}, 2019{\natexlab{b}}.

\bibitem[Dua and Graff(2017)]{Dua:2019}
Dheeru Dua and Casey Graff.
\newblock {UCI} machine learning repository, 2017.
\newblock URL \url{http://archive.ics.uci.edu/ml}.

\bibitem[Dziugaite and Roy(2017)]{dziugaite2017computing}
Gintare~Karolina Dziugaite and Daniel~M Roy.
\newblock Computing nonvacuous generalization bounds for deep (stochastic)
  neural networks with many more parameters than training data.
\newblock \emph{Uncertainty inArtificial Intelligence}, 2017.

\bibitem[E et~al.(2020)E, Ma, and Wu]{ma2020comparative}
Weinan E, Chao Ma, and Lei Wu.
\newblock A comparative analysis of the optimization and generalization
  property of two-layer neural network and random feature models under gradient
  descent dynamics.
\newblock \emph{Science China Mathematics}, 2020.

\bibitem[Eckmann and Ruelle(1985)]{eckmann1985ergodic}
J-P Eckmann and David Ruelle.
\newblock Ergodic theory of chaos and strange attractors.
\newblock In \emph{The theory of chaotic attractors}, pages 273--312. Springer,
  1985.

\bibitem[Falniowski et~al.(2015)Falniowski, Kulczycki, Kwietniak, and
  Li]{falniowski2015two}
Fryderyk Falniowski, Marcin Kulczycki, Dominik Kwietniak, and Jian Li.
\newblock Two results on entropy, chaos and independence in symbolic dynamics.
\newblock \emph{Discrete \& Continuous Dynamical Systems-B}, 20\penalty0
  (10):\penalty0 3487, 2015.

\bibitem[Franca et~al.(2018)Franca, Robinson, and Vidal]{franca2018admm}
Guilherme Franca, Daniel Robinson, and Rene Vidal.
\newblock Admm and accelerated admm as continuous dynamical systems.
\newblock In \emph{International Conference on Machine Learning}, pages
  1559--1567, 2018.

\bibitem[Golowich et~al.(2018)Golowich, Rakhlin, and Shamir]{golowich2018size}
Noah Golowich, Alexander Rakhlin, and Ohad Shamir.
\newblock Size-independent sample complexity of neural networks.
\newblock In \emph{Conference On Learning Theory}, pages 297--299, 2018.

\bibitem[Hairer et~al.(2006)Hairer, Lubich, and Wanner]{hairer2006geometric}
Ernst Hairer, Christian Lubich, and Gerhard Wanner.
\newblock \emph{Geometric numerical integration: structure-preserving
  algorithms for ordinary differential equations}, volume~31.
\newblock Springer Science \& Business Media, 2006.

\bibitem[Hennion and Herv{\'e}(2004)]{hennion2004central}
Hubert Hennion and Lo{\"\i}c Herv{\'e}.
\newblock Central limit theorems for iterated random lipschitz mappings.
\newblock \emph{The Annals of Probability}, 32\penalty0 (3):\penalty0
  1934--1984, 2004.

\bibitem[Iwanik(1991)]{iwanik1991independence}
Anzelm Iwanik.
\newblock Independence and scrambled sets for chaotic mappings.
\newblock In \emph{The mathematical heritage of CF Gauss}, pages 372--378.
  World Scientific, 1991.

\bibitem[Jastrz{\k{e}}bski et~al.(2017)Jastrz{\k{e}}bski, Kenton, Arpit,
  Ballas, Fischer, Bengio, and Storkey]{jastrzkebski2017three}
Stanis{\l}aw Jastrz{\k{e}}bski, Zachary Kenton, Devansh Arpit, Nicolas Ballas,
  Asja Fischer, Yoshua Bengio, and Amos Storkey.
\newblock Three factors influencing minima in sgd.
\newblock \emph{arXiv preprint arXiv:1711.04623}, 2017.

\bibitem[Jin et~al.(2017)Jin, Ge, Netrapalli, Kakade, and
  Jordan]{jin2017escape}
Chi Jin, Rong Ge, Praneeth Netrapalli, Sham~M Kakade, and Michael~I Jordan.
\newblock How to escape saddle points efficiently.
\newblock In \emph{Proceedings of the 34th International Conference on Machine
  Learning-Volume 70}, pages 1724--1732. JMLR. org, 2017.

\bibitem[Jin et~al.(2018)Jin, Liu, Ge, and Jordan]{jin2018local}
Chi Jin, Lydia~T Liu, Rong Ge, and Michael~I Jordan.
\newblock On the local minima of the empirical risk.
\newblock In \emph{Advances in Neural Information Processing Systems}, pages
  4896--4905, 2018.

\bibitem[Kovachki and Stuart(2019)]{kovachki2019analysis}
Nikola~B Kovachki and Andrew~M Stuart.
\newblock Analysis of momentum methods.
\newblock \emph{arXiv preprint arXiv:1906.04285}, 2019.

\bibitem[Lee et~al.(2016)Lee, Simchowitz, Jordan, and Recht]{lee2016gradient}
Jason~D Lee, Max Simchowitz, Michael~I Jordan, and Benjamin Recht.
\newblock Gradient descent only converges to minimizers.
\newblock In \emph{Conference on learning theory}, pages 1246--1257, 2016.

\bibitem[Lewkowycz et~al.(2020)Lewkowycz, Bahri, Dyer, Sohl-Dickstein, and
  Gur-Ari]{lewkowycz2020large}
Aitor Lewkowycz, Yasaman Bahri, Ethan Dyer, Jascha Sohl-Dickstein, and Guy
  Gur-Ari.
\newblock The large learning rate phase of deep learning: the catapult
  mechanism.
\newblock \emph{arXiv preprint arXiv:2003.02218}, 2020.

\bibitem[Li et~al.(2017)Li, Tai, and Weinan]{li2017stochastic}
Qianxiao Li, Cheng Tai, and E~Weinan.
\newblock Stochastic modified equations and adaptive stochastic gradient
  algorithms.
\newblock In \emph{International Conference on Machine Learning}, pages
  2101--2110, 2017.

\bibitem[Li et~al.(2019{\natexlab{a}})Li, Tai, and Weinan]{li2019stochastic}
Qianxiao Li, Cheng Tai, and E~Weinan.
\newblock Stochastic modified equations and dynamics of stochastic gradient
  algorithms i: Mathematical foundations.
\newblock \emph{Journal of Machine Learning Research}, 20\penalty0
  (40):\penalty0 1--40, 2019{\natexlab{a}}.

\bibitem[Li(1993)]{li1993}
Shi~Hai Li.
\newblock $\omega$-chaos and topological entropy.
\newblock \emph{Transactions of the American Mathematical Society},
  339\penalty0 (1):\penalty0 243--249, 1993.

\bibitem[Li and Yorke(1975)]{li1975period}
Tien-Yien Li and James~A Yorke.
\newblock Period three implies chaos.
\newblock \emph{The American Mathematical Monthly}, 82\penalty0 (10):\penalty0
  985--992, 1975.

\bibitem[Li et~al.(2018)Li, Lu, Wang, Haupt, and Zhao]{li2018tighter}
Xingguo Li, Junwei Lu, Zhaoran Wang, Jarvis Haupt, and Tuo Zhao.
\newblock On tighter generalization bound for deep neural networks: Cnns,
  resnets, and beyond.
\newblock \emph{arXiv preprint arXiv:1806.05159}, 2018.

\bibitem[Li and Liang(2018)]{li2018learning}
Yuanzhi Li and Yingyu Liang.
\newblock Learning overparameterized neural networks via stochastic gradient
  descent on structured data.
\newblock In \emph{Advances in Neural Information Processing Systems}, pages
  8157--8166, 2018.

\bibitem[Li et~al.(2019{\natexlab{b}})Li, Wei, and Ma]{li2019towards}
Yuanzhi Li, Colin Wei, and Tengyu Ma.
\newblock Towards explaining the regularization effect of initial large
  learning rate in training neural networks.
\newblock In \emph{Advances in Neural Information Processing Systems}, pages
  11669--11680, 2019{\natexlab{b}}.

\bibitem[Liu et~al.(2017)Liu, Shang, Cheng, Cheng, and
  Jiao]{liu2017accelerated}
Yuanyuan Liu, Fanhua Shang, James Cheng, Hong Cheng, and Licheng Jiao.
\newblock Accelerated first-order methods for geodesically convex optimization
  on {R}iemannian manifolds.
\newblock In \emph{Advances in Neural Information Processing Systems}, pages
  4868--4877, 2017.

\bibitem[Lyapunov(1992)]{lyapunov1992general}
Aleksandr~Mikhailovich Lyapunov.
\newblock The general problem of the stability of motion.
\newblock \emph{International journal of control}, 55\penalty0 (3):\penalty0
  531--534, 1992.

\bibitem[Ma et~al.(2019)Ma, Chatterji, Cheng, Flammarion, Bartlett, and
  Jordan]{ma2019there}
Yi-An Ma, Niladri Chatterji, Xiang Cheng, Nicolas Flammarion, Peter Bartlett,
  and Michael~I Jordan.
\newblock Is there an analog of {N}esterov acceleration for {MCMC}?
\newblock \emph{arXiv preprint arXiv:1902.00996}, 2019.

\bibitem[Mandt et~al.(2017)Mandt, Hoffman, and Blei]{mandt2017stochastic}
Stephan Mandt, Matthew~D Hoffman, and David~M Blei.
\newblock Stochastic gradient descent as approximate bayesian inference.
\newblock \emph{The Journal of Machine Learning Research}, 18\penalty0
  (1):\penalty0 4873--4907, 2017.

\bibitem[Mei et~al.(2018)Mei, Bai, and Montanari]{mei2018}
Song Mei, Yu~Bai, and Andrea Montanari.
\newblock The landscape of empirical risk for nonconvex losses.
\newblock \emph{Ann. Statist.}, 46\penalty0 (6A):\penalty0 2747--2774, 12 2018.
\newblock \doi{10.1214/17-AOS1637}.
\newblock URL \url{https://doi.org/10.1214/17-AOS1637}.

\bibitem[Misiurewicz(2010)]{misiurewicz2010horseshoes}
Micha{\l} Misiurewicz.
\newblock Horseshoes for continuous mappings of an interval.
\newblock In \emph{Dynamical systems}, pages 125--135. Springer, 2010.

\bibitem[Moser(1973)]{moser1973stable}
J{\"u}rgen Moser.
\newblock \emph{Stable and random motions in dynamical systems: With special
  emphasis on celestial mechanics}, volume~1.
\newblock Princeton University Press, 1973.

\bibitem[Moulines and Bach(2011)]{moulines2011non}
Eric Moulines and Francis~R Bach.
\newblock Non-asymptotic analysis of stochastic approximation algorithms for
  machine learning.
\newblock In \emph{Advances in Neural Information Processing Systems}, pages
  451--459, 2011.

\bibitem[Nesterov(2013)]{nesterov2013introductory}
Yurii Nesterov.
\newblock \emph{Introductory lectures on convex optimization: A basic course},
  volume~87.
\newblock Springer Science \& Business Media, 2013.

\bibitem[Neyshabur and Li(2019)]{neyshabur2019towards}
Behnam Neyshabur and Zhiyuan Li.
\newblock Towards understanding the role of over-parametrization in
  generalization of neural networks.
\newblock In \emph{International Conference on Learning Representations
  (ICLR)}, 2019.

\bibitem[Neyshabur et~al.(2015)Neyshabur, Tomioka, and
  Srebro]{neyshabur2015norm}
Behnam Neyshabur, Ryota Tomioka, and Nathan Srebro.
\newblock Norm-based capacity control in neural networks.
\newblock In \emph{Conference on Learning Theory}, pages 1376--1401, 2015.

\bibitem[Ornstein and Weiss(1973)]{ornstein1973geodesic}
Donald Ornstein and Benjamin Weiss.
\newblock Geodesic flows are bernoullian.
\newblock \emph{Israel Journal of Mathematics}, 14\penalty0 (2):\penalty0
  184--198, 1973.

\bibitem[Oseledec(1968)]{oseledec1968multiplicative}
Valery~Iustinovich Oseledec.
\newblock A multiplicative ergodic theorem. liapunov characteristic number for
  dynamical systems.
\newblock \emph{Trans. Moscow Math. Soc.}, 19:\penalty0 197--231, 1968.

\bibitem[Ott(2002)]{ott2002chaos}
Edward Ott.
\newblock \emph{Chaos in dynamical systems}.
\newblock Cambridge university press, 2002.

\bibitem[Pavliotis and Stuart(2008)]{pavliotis2008multiscale}
Grigoris Pavliotis and Andrew Stuart.
\newblock \emph{Multiscale methods: averaging and homogenization}, volume~53.
\newblock Springer, 2008.

\bibitem[Polyak(1964)]{polyak1964some}
Boris~T Polyak.
\newblock Some methods of speeding up the convergence of iteration methods.
\newblock \emph{USSR Computational Mathematics and Mathematical Physics},
  4\penalty0 (5):\penalty0 1--17, 1964.

\bibitem[Polyak and Juditsky(1992)]{polyak1992acceleration}
Boris~T Polyak and Anatoli~B Juditsky.
\newblock Acceleration of stochastic approximation by averaging.
\newblock \emph{SIAM journal on control and optimization}, 30\penalty0
  (4):\penalty0 838--855, 1992.

\bibitem[Robbins and Monro(1951)]{robbins1951stochastic}
Herbert Robbins and Sutton Monro.
\newblock A stochastic approximation method.
\newblock \emph{The Annals of Mathematical Statistics}, pages 400--407, 1951.

\bibitem[Roux et~al.(2012)Roux, Schmidt, and Bach]{roux2012stochastic}
Nicolas~L Roux, Mark Schmidt, and Francis~R Bach.
\newblock A stochastic gradient method with an exponential convergence \_rate
  for finite training sets.
\newblock In \emph{Advances in neural information processing systems}, pages
  2663--2671, 2012.

\bibitem[Sanders et~al.(2010)Sanders, Verhulst, and Murdock]{SaVeMu10}
J.~A. Sanders, F.~Verhulst, and J.~Murdock.
\newblock \emph{Averaging Methods in Nonlinear Dynamical Systems}.
\newblock Springer, 2010.

\bibitem[Sharkovski{\u\i}(Original 1962; Translated
  1995)]{sharkovskiui1995coexistence}
AN~Sharkovski{\u\i}.
\newblock Coexistence of cycles of a continuous map of the line into itself.
\newblock \emph{International Journal of Bifurcation and Chaos}, 5\penalty0
  (05):\penalty0 1263--1273, Original 1962; Translated 1995.

\bibitem[Shi et~al.(2018)Shi, Du, Jordan, and Su]{shi2018understanding}
Bin Shi, Simon~S Du, Michael~I Jordan, and Weijie~J Su.
\newblock Understanding the acceleration phenomenon via high-resolution
  differential equations.
\newblock \emph{arXiv preprint arXiv:1810.08907}, 2018.

\bibitem[Sinai(1970)]{sinai1970dynamical}
Yakov~G Sinai.
\newblock Dynamical systems with elastic reflections.
\newblock \emph{Russian Mathematical Surveys}, 25\penalty0 (2):\penalty0 137,
  1970.

\bibitem[Sitzmann et~al.(2020)Sitzmann, Martel, Bergman, Lindell, and
  Wetzstein]{sitzmann2020implicit}
Vincent Sitzmann, Julien~NP Martel, Alexander~W Bergman, David~B Lindell, and
  Gordon Wetzstein.
\newblock Implicit neural representations with periodic activation functions.
\newblock \emph{NeurIPS}, 2020.

\bibitem[Smith(2017)]{smith2017cyclical}
Leslie~N Smith.
\newblock Cyclical learning rates for training neural networks.
\newblock In \emph{2017 IEEE Winter Conference on Applications of Computer
  Vision (WACV)}, pages 464--472. IEEE, 2017.

\bibitem[Smith and Topin(2019)]{smith2019super}
Leslie~N Smith and Nicholay Topin.
\newblock Super-convergence: Very fast training of neural networks using large
  learning rates.
\newblock In \emph{Artificial Intelligence and Machine Learning for
  Multi-Domain Operations Applications}, volume 11006, page 1100612.
  International Society for Optics and Photonics, 2019.

\bibitem[Strogatz(2018)]{strogatz2018nonlinear}
Steven~H Strogatz.
\newblock \emph{Nonlinear dynamics and chaos: with applications to physics,
  biology, chemistry, and engineering}.
\newblock CRC Press, 2018.

\bibitem[Su et~al.(2014)Su, Boyd, and Candes]{su2014differential}
Weijie Su, Stephen Boyd, and Emmanuel Candes.
\newblock A differential equation for modeling {N}esterov’s accelerated
  gradient method: Theory and insights.
\newblock In \emph{Advances in Neural Information Processing Systems}, pages
  2510--2518, 2014.

\bibitem[Tao and Ohsawa(2020)]{tao2020Lie}
Molei Tao and Tomoki Ohsawa.
\newblock Variational optimization on {L}ie groups, with examples of leading
  (generalized) eigenvalue problems.
\newblock \emph{International Conference on Artificial Intelligence and
  Statistics}, 2020.

\bibitem[Tu et~al.(2015)Tu, Boczar, Simchowitz, Soltanolkotabi, and
  Recht]{tu2015low}
Stephen Tu, Ross Boczar, Max Simchowitz, Mahdi Soltanolkotabi, and Benjamin
  Recht.
\newblock Low-rank solutions of linear matrix equations via procrustes flow.
\newblock \emph{arXiv preprint arXiv:1507.03566}, 2015.

\bibitem[Wei et~al.(2019)Wei, Lee, Liu, and Ma]{wei2019regularization}
Colin Wei, Jason~D Lee, Qiang Liu, and Tengyu Ma.
\newblock Regularization matters: Generalization and optimization of neural
  nets vs their induced kernel.
\newblock In \emph{Advances in Neural Information Processing Systems}, pages
  9709--9721, 2019.

\bibitem[Wibisono et~al.(2016)Wibisono, Wilson, and Jordan]{wibisonoe7351}
Andre Wibisono, Ashia~C. Wilson, and Michael~I. Jordan.
\newblock A variational perspective on accelerated methods in optimization.
\newblock \emph{Proceedings of the National Academy of Sciences}, 113\penalty0
  (47):\penalty0 E7351--E7358, 2016.
\newblock ISSN 0027-8424.

\bibitem[Yi et~al.(2016)Yi, Park, Chen, and Caramanis]{yi2016fast}
Xinyang Yi, Dohyung Park, Yudong Chen, and Constantine Caramanis.
\newblock Fast algorithms for robust pca via gradient descent.
\newblock In \emph{Advances in neural information processing systems}, pages
  4152--4160, 2016.

\bibitem[Young(1998)]{young1998statistical}
Lai-Sang Young.
\newblock Statistical properties of dynamical systems with some hyperbolicity.
\newblock \emph{Annals of Mathematics}, 147:\penalty0 585--650, 1998.

\bibitem[Zhang(2004)]{zhang2004solving}
Tong Zhang.
\newblock Solving large scale linear prediction problems using stochastic
  gradient descent algorithms.
\newblock In \emph{Proceedings of the twenty-first international conference on
  Machine learning}, page 116. ACM, 2004.

\end{thebibliography}

\newpage

\appendix

\section{On the insufficiency of modified equation}
\label{sec_AppendixModifiedEq}

Recently there has been an extremely interesting line of research in which discrete algorithms are studied through their continuum limits (e.g., \cite{su2014differential,wibisonoe7351,liu2017accelerated,franca2018admm,ma2019there,tao2020Lie}); these limits, however, correspond to a small LR (denoted by $\eta$) regime.

It is possible to slightly extend this regime by writing down a limiting ODE that includes additional correction terms (e.g., \cite{shi2018understanding,li2019stochastic,kovachki2019analysis}). The classical notion for systematically doing so is backward error analysis and modified equation (e.g., \cite{hairer2006geometric}). For example, the GD map $\varphi$ can be formally approximated, via an application of the modified equation theory, by $\dot{x}=-\nabla \tilde{f} (x)$, where the modified loss
\[
    \tilde{f}(x)=f(x)+\frac{\eta}{4}\|\nabla f(x)\|_2^2+\mathcal{O}(\eta^2).
\]
While informative, this result does not help us understand the large LR regime. Take $f_{1,\epsilon}=\epsilon f_1(x/\epsilon)$ for periodic $f_1$ as an example. When $\eta \geq C \epsilon$ for some $C>0$, the formal series expansion used in modified equation does not converge (see Appendix \ref{sec_AppendixModifiedEq}), which renders it inapplicable.

More precisely, as detailed in \cite{hairer2006geometric} Chap IX.1, in order for a discrete map
\[
    \Phi_\eta(x) = x + \eta g(x)    \qquad (\text{in our case $g(x)=f'(x)= f'_0(x)+ f'_1(x/\epsilon)$})
\]
to be the $\eta$-time flow of
\begin{equation}
    \dot{x}=g(x)+\eta g_2(x) + \eta^2 g_3(x) + \cdots,
    \label{eq_modifiedEq}
\end{equation}
we need
\begin{align*}
    g_2(x) &= -\frac{1}{2!} g'g(x) \\
    g_3(x) &= -\frac{1}{3!} (g''(g,g)(x)+g'g'g(x)) - \frac{1}{2!} (g'g_2(x)+g_2'g(x)) \\
    & \cdots
\end{align*}
Note each derivative of $g$ gives a factor of $1/\epsilon$, and thus $g_n = \mathcal{O}(\epsilon^{-(n-1)})$. Therefore, RHS of \eqref{eq_modifiedEq} diverges if $\eta \geq C\epsilon$ for some $C>0$, in which case the more higher-order correction terms are included, the worse approximation power the modified ODE will have. 

\emph{This paper thus develops a completely different framework to understand the large LR regime.}

\section{Proofs and additional remarks}
\label{sec_appendixProofs}

\subsection{On the relation between stochastic and deterministic map}
\begin{remark}[On Theorem \ref{thm_sameLimitStats}]
    ~
	\begin{itemize}
		\item The purpose for using an open set $\mathcal{E}$ accumulating at 0 but does not use a interval such as $(0,1]$ directly here. In the later Theorem \ref{counter-example2}, we proved that for a fixed $f_0$ and $\eta$, there exists periodic $f_{1,\epsilon}$ and arbitrary small $\epsilon$ to make the non trivial invariant distribution doesn't exist. We can use the set $\mathcal{E}$ to eliminate this bad case that we doesn't want to see.
		
		\item 
		Lemma \ref{thm_geometricErgodicity} gives a sufficient condition for $\hat\varphi$ to have a unique fixed point, denoted by $X$. When this happens, the conclusion will be if $\{X_{\epsilon_i}\}_{i=1}^\infty$ has a weak limit, $\{X_{\epsilon_i}\}_{i=1}^\infty\rightarrow X$. We do numerical tests on this situation in Sec.\ref{sec_Matyas}. When $\hat\varphi$ have multiple fixed points, please see related numerical test in Sec.\ref{sec_NT_nonvonvex}.
		
		\item Intuitively, condition (*) means $\varphi_\epsilon$ is continuous in $\mathcal{F}$. This property is used in the proof of lemma \ref{lemma_3}. Condition (*) is strong, but we can hardly prove it or find a condition that easy to test. The 2-order derative of $f_0$ goes to infinity, which is pathological, but also make the whole problem interesting and nontrivial. See Thm. \ref{counter-example1} and \ref{counter-example2} for 2 examples. However, some necessary conditions could be useful, such as the r.v.'s in $\mathcal{F}$ cannot have atom points (which means all the variables are nondegenerate).
	\end{itemize}
\end{remark}
In order to prove Theorem \ref{thm_sameLimitStats}, we need the following lemmas.
\begin{lemma}
	\label{lemma_1}
	Under the condition of Thm. \ref{thm_sameLimitStats}, $\forall X$, there exists $\tilde{X}$, such that $\sup_{\omega\in\Omega}\|\tilde{X}(\omega)-X(\omega)\|_2<\delta(\epsilon)$ where $\Omega$ is the sample space
	and $\varphi_\epsilon(\tilde{X})\stackrel{w}{\longrightarrow}\hat{\varphi}(\tilde{X})$ when $\epsilon\rightarrow 0$.
\end{lemma}
\begin{proof}
	Let $\tilde{X}:=X+Y_{X,\epsilon}$,
	where $Y_{X,\epsilon}$ is defined as in Cond. \ref{cond_1}. Without causing confusion, the dependence of $Y_{x,\epsilon}$ on $\epsilon$ is omitted in this proof, as well as in lemma \ref{lemma_2} and \ref{lemma_3}. So $\sup_\omega\|Y_X(\omega)\|_2<\delta(\epsilon)$. ($\delta(\epsilon)$ is given in Cond. \ref{cond_1})

	Arbitrarily choosing a test function $g$, we have
	\begin{align*}
		&\lim_{\epsilon\rightarrow 0}\mathbb{E}\left[g(\varphi_\epsilon(\tilde{X}))-g(\hat{\varphi}(\tilde{X}))\right]\\
		=&\lim_{\epsilon\rightarrow 0}\mathbb{E}\left[g(\tilde{X}-\eta\nabla f_0(\tilde{X})-\eta\nabla f_{1,\epsilon}(\tilde{X}))-g(\tilde{X}-\eta\nabla f_0(\tilde{X})-\eta\zeta)\right]\\
		=&\lim_{\epsilon\rightarrow 0}\mathbb{E}_X[\mathbb{E}_{Y_X}[g(X+Y_X-\eta\nabla f_0(X+Y_X)\\
		&\quad-\eta\nabla f_{1,\epsilon}(X+Y_X))-g(X+Y_X-\eta\nabla f_0(X+Y_X)-\eta\zeta)|X]]
	\end{align*}
	We use the nice property of $g$ and $f_0$ to have some of the $Y_X$'s.
	$$g(x+Y_x-\eta\nabla f_0(x+Y_x)-\eta\nabla f_{1,\epsilon}(x+Y_x))=g(x-\eta\nabla f_0(x)-\eta\nabla f_{1,\epsilon}(x+Y_x))+\mathcal{O}(\delta(\epsilon))$$
	$$g(x+Y_x-\eta\nabla f_0(x+Y_x)-\eta\zeta)=g(x-\eta\nabla f_0(x)-\eta\zeta)+\mathcal{O}(\delta(\epsilon))$$
	Due to the uniform weak convergence condition in condition \ref{cond_1}, we calculate the limit first and then compute the expectation regarding $X$, which means
	\begin{align*}
		&\lim_{\epsilon\rightarrow 0}\mathbb{E}\left[g(\varphi_\epsilon(\tilde{X}))-g(\hat{\varphi}(\tilde{X}))\right]\\
		=&\mathbb{E}_X\left[\lim_{\epsilon\rightarrow 0}\mathbb{E}_{Y_X}\left[g(X-\eta\nabla f_0(X)-\eta\nabla f_{1,\epsilon}(X+Y_X))-g(X-\eta\nabla f_0(X)-\eta\zeta)|X\right]\right]\\
		=&0
	\end{align*}
\end{proof}

\begin{lemma}
	\label{lemma_2}
	Let $\tilde{X}:=X+Y_X$ (as in the proof of Lemma \ref{lemma_1}). Then $\hat{\varphi}(\tilde{X})\stackrel{w}{\longrightarrow}\hat{\varphi}(X)$ as $\epsilon\rightarrow 0$.
\end{lemma}
\begin{proof}
	For an arbitrary test function $g$, we have
	\begin{align*}
		&\lim_{\epsilon\rightarrow 0}\mathbb{E}\left[g(\hat{\varphi}(\tilde{X}))-g(\hat{\varphi}(X))\right]\\
		=&\lim_{\epsilon\rightarrow 0}\mathbb{E}\left[g(\tilde{X}-\eta\nabla f_0(\tilde{X})-\eta\zeta)-g(X-\eta\nabla f_0(X)-\eta\zeta)\right]\\
		\le&\lim_{\epsilon\rightarrow 0}\mathbb{E}\left[\sup\|\nabla g\| \|(\tilde{X}-\eta\nabla f_0(\tilde{X}))-(X-\eta\nabla f_0(X)) \|\right]\\
		\le&\lim_{\epsilon\rightarrow 0}\mathbb{E}\left[\sup\|\nabla g\|(1+\eta L)\|\tilde{X}-X\|_2\right]\\
		\le&\lim_{\epsilon\rightarrow 0}(1+\eta L)\sup\|\nabla g\|\,\delta(\epsilon)\\
		=&0
	\end{align*}
	The 3rd last line is due to $L$-smoothness of $f_0$.
\end{proof}

\begin{lemma}
	\label{lemma_3}
	$\forall X\in\mathcal{F}$, $\varphi_\epsilon(X)\stackrel{w}{\longrightarrow}\hat{\varphi}(X)$ when $\epsilon\rightarrow 0$.
\end{lemma}
\begin{proof}
	We define $\tilde{X}:=X+Y_X$, like we did in the proof for lemma \ref{lemma_1}. Fix a $g$ as the test function.
	\begin{align*}
		&\mathbb{E}\left[g(\varphi_\epsilon(X))-g(\hat{\varphi}(X))\right]\\
		=&\mathbb{E}\left[g(\varphi_\epsilon(X))-g(\varphi_\epsilon(\tilde{X}))\right]+\mathbb{E}\left[g(\hat{\varphi}(X))-g(\hat{\varphi}(\tilde{X}))\right]+\mathbb{E}\left[g(\varphi_\epsilon(\tilde{X}))-g(\hat{\varphi}(\tilde{X}))\right]
	\end{align*}
	The first term converges to 0 due to condition (*) in Thm. \ref{thm_sameLimitStats}, which ensures the continuity in the weak sense of $\varphi_\epsilon$. The second term goes to 0 according to lemma \ref{lemma_2}. The third term converges to 0 according to lemma \ref{lemma_1}. So we have $\mathbb{E}\left[g(\varphi_\epsilon(X))-g(\hat{\varphi}(X))\right]\rightarrow 0$.
\end{proof}
This lemma prepares us to finish the following proof.

\begin{proof}[Proof of Thm.\ref{thm_sameLimitStats}]
Suppose $X_{\epsilon_i}\in\mathcal{F}$ is a sequence of r.v. , which are fixed points for $\varphi_{\epsilon_i}$, and have a limit point $X\in\mathcal{F}$ in the weak sence. Then we have
$$\varphi_\epsilon(X_\epsilon)\overset{w}{=}X_\epsilon,\quad\forall \epsilon=\epsilon_i$$
$$X_{\epsilon_i}\stackrel{w}{\longrightarrow}X$$
$$\varphi_{\epsilon_i}(X_{\epsilon_i})\stackrel{w}{\longrightarrow}\hat{\varphi}(X)$$
So $\hat{\varphi}(X)\overset{w}{=}X$.
\end{proof}

\subsection{On the stochastic map $\hat\varphi$}
\subsubsection{Some quantitative results about its ergodicity}
\begin{proof}[Proof of Lemma \ref{GeoErgo}]
	Here we use the machinery provided by \cite{hennion2004central}. Regard $\hat\varphi$ as a random action on $\mathbb{R}^d$. In this proof, we write the dependence of $\hat\varphi$ on $\zeta$ explicitly as $\hat\varphi_\zeta$. Choose a fixed point $x_0$ and let
	\begin{align*}
		c(\zeta)&:=\sup\left\{\frac{d(\hat\varphi_\zeta x,\hat\varphi_\zeta y)}{d(x,y)}:x,y\in \mathbb{R}^d,x\neq y\right\}\\
		\mathcal{M}_{\gamma+1}&:=\int_G(1+c(\zeta)+d(\varphi_\zeta x_0,x_0))^{\gamma}\, d\pi(\zeta)\\
		\mathcal{C}_{\gamma+1}^{(n_0)}&:=\int_G c(\varphi_\zeta)\max\{c(\varphi_\zeta),i\}^{\gamma}\, d\pi^{*n}(\zeta)
	\end{align*}
	In $\hat{\varphi}$ and the our interested chaotic regime of learning rate, since $f_0$ is strongly convex and $L$-smooth, we choose $\eta_0$ small to ensure $c(\varphi_\zeta)=1-\eta_0 L<1$, and we choose $\gamma=0$, $n_0=1$ to get $\mathcal{M}_{\gamma+1}=E_\zeta[1+c(\varphi_\zeta)+d(\hat\varphi_\zeta(x_0),x_0)]<+\infty$ and $C_{\gamma+1}^{(1)}=E_\zeta[c(\varphi_\zeta)]<1$.
	
	Under these facts, Theorem 1 in \cite{hennion2004central} ensures that there is a unique $\hat\varphi$-invariant probability distribution $\hat\mu_0$. Moreover, geometric ergodicity holds in the Prokhorov distance $d_P$. Namely, there exists positive real number $C$ and $\kappa_0<1$, such that, for any probability distribution $\mu$ on $M$ satisfying $\mu(d(\cdot,x_0))<+\infty$, and all $n\ge1$,
	\begin{equation*}
		d_P(\hat\varphi_\sharp^{(n)}\mu,\hat\mu_0)\le C\kappa_0^{n/2}
	\end{equation*}
	where $\hat\varphi_\sharp^{(n)}$ stands for apply the push forward of measure $n$ times.
\end{proof}
\begin{remark}
	In a separable metric space, which is our case, convergence of measures in the Prokhorov metric is equivalent to weak convergence of measures, which is also equivalent to the convergence of cumulative distribution functions.
\end{remark}
The following two remarks show that convexity and $L$-smoothness of $f_0$ are necessary for geometric ergodicity established by Lemma \ref{GeoErgo}.
\begin{remark}
	\label{rmk_nonconvex}
	Here we will explain in 1-dim, what can happen when the function $f_0$ is not convex. Since the random variable $\zeta$ is bounded, denote it by $[a,b]$. Unlike in a standard overdamped Langevin case, there can be potential barriers in $f_0$ that $\hat{\varphi}$ cannot cross, because the noise is of a finite strength. To make this quantitative, we assume the existence of an invariant distribution with density $\mu_0$, and calculate what kind of points are not in the support of $\mu_0$. When $\eta<1/L$, for a point $x\in\text{supp}\hat\mu_0$, we have $\eta f_0'(x)\in\eta[a,b]$. So if $\left\{ x| f_0'(x)\in[a,b] \right\}$ is not a connected set (note that it is independent from $\eta$), then the support of the invariant density will be separated in to disjoint components, and no orbit can jump between them. An example explains why the set can be disconnected:
	
	Suppose $f_0=k(x^2-1)^2$, $k>0$ for example, and $f_{1,\epsilon}=\epsilon\sin(x/\epsilon)$. Calculate the set $S:=\{x: f_0'(x)\in[-1,1]\}=\{x: |4kx(x^2-1)|<1\}$. We have that when $k<\frac{3\sqrt{3}}{8}$, $S$ is connected. But when $k>\frac{3\sqrt{3}}{8}$, the set $S$ is not connected. In this case, a point cannot jump from one well to another as $\hat{\varphi}$ is closed in each connected component of $S$, which means ergodicity on $S$ is lost. Which distribution the system converges to (if existent) relies on which well the initial condition belongs to. 
	
	In multi-dimension case, connectedness is different from simply connectedness, which complicates the intuition. We won't discuss it here.

	See also Sec. \ref{sec_NT_nonvonvex} on jumping between potential wells by the deterministic map.
\end{remark}
\begin{remark}
	When $f_0$ is not $L$-smooth, such as $f_0(x)=(x^2+1)^2$ and $f_{1,\epsilon}=\epsilon\sin(x/\epsilon)$. For a fixed $\eta$, it is easy to see that when the absolute value of initial condition is greater than $x_0$, where $x_0$ is the greatest solution of $x-4\eta x(x^2+1)+\eta+x=0$, we know $P(|\hat\varphi(x)|>|x|)=1$, so the system will explode and never converge to any distribution. This is because $\mathcal{M}_{\gamma+1}<\infty$ in the proof of Lemma \ref{thm_geometricErgodicity} is not satisfied.
\end{remark}

\begin{theorem}[coupling estimation of the exponential convergence rate of $\hat{\varphi}$]
	Consider the iteration $x_{k+1}=x_k-\eta \nabla f_0(x_k)+\eta \zeta_k$ for i.i.d. $\zeta_k\sim \zeta$. Denote by $\rho_k$ the density of $x_k$. Assume $f_0$ is $\mathcal{C}^2$, $\nu$-smooth and $\mu$-strongly convex, and $f_1$ is $\mathcal{C}^1$. Then the limiting distribution $\rho_\infty$ exists and the 2-Wasserstein distance satisfies the nonasymptotic bound
	\begin{equation}
	W_2(\rho_k, \rho_\infty) \leq \left( \max\{|1-\eta\mu|,|1-\eta\nu|\} \right)^k C
	\label{eq_geomErgInW2}
	\end{equation}
	for some constant $C\geq 0$.
	\label{thm_geomErgInW2}
\end{theorem}

\begin{proof}
	Existence of $\rho_\infty$ is guaranteed by Lemma \ref{GeoErgo}.
	
	Let $\hat{x}_0$ be a random variable distributed according to $\rho_\infty$ and define
	\[
	\hat{x}_{k+1} = \hat{x}_k - \eta \nabla f_0(\hat{x}_k)+\eta \zeta_k
	\]
	using the same noise $\zeta_k$. Then
	\[
	    x_{k+1}-\hat{x}_{k+1}=x_k-\hat{x}_k-\eta \left( \nabla f_0(x_k)-\nabla f_0(\hat{x}_k) \right)
	\]
	Since $f_0$ is $\mathcal{C}^2$, $\nu$-smooth and $\mu$-strongly convex, it is easy to see that the mapping $x \mapsto x-\eta \nabla f_0(x)$ is a contraction with rate$=\max\{|1-\eta\mu|,|1-\eta\nu|\}$.
	Therefore,
	\[
		\|x_{k+1}-\hat{x}_{k+1}\| \leq \max\{|1-\eta\mu|,|1-\eta\nu|\} \|x_k-\hat{x}_k\|
	\]
	Thus,
	\[
	\mathbb{E} \| x_{k+1} - \hat{x}_{k+1} \|^2 \leq \max\{|1-\eta\mu|,|1-\eta\nu|\}^{2k} \mathbb{E} \|x_0-\hat{x_0}\|^2
	\]
	Note $\hat{x}_k$ is distributed according to $\rho_\infty$ because that is the invariant distribution and $\hat{x}_0 \sim \rho_\infty$. By definition,
	\begin{align*}
		W_2(\rho_k,\rho_\infty)^2 &= \inf_{\pi \in \Pi (\rho_k, \rho_\infty)} \int \|y_1-y_2\|^2 d\pi(y_1,y_2) \\
		& \leq \mathbb{E} \|x_k - \hat{x}_k \|^2.
	\end{align*}
	Therefore, the choice of $C=\sqrt{\mathbb{E}\|x_0-\hat{x}_0\|^2}$ leads to eq.\ref{eq_geomErgInW2}.
\end{proof}

\begin{corollary}[Spectral gap of $\hat{\varphi}$ is at least at the order of $\eta$]
	Consider the setup of Thm.\ref{thm_geomErgInW2} and $\eta < \frac{1}{\nu}$. Denote by $L$ the transition operator of the Markov process generated by $\hat{\varphi}$, i.e., $L\rho_k = \rho_{k+1} \quad \forall k$. Then $L$ has a single eigenvalue of 1, and any other eigenvalue $\lambda$ satisfies $|1-\lambda| \geq \eta\mu$.
	\label{cor_spectralGap}
\end{corollary}

\begin{proof}
	Since $\hat{\varphi}$ generates a Markov process, any eigenvalue has modulus bounded by 1.
	
	The single eigenvalue of 1 is guaranteed by geometric ergodicity (Lemma \ref{GeoErgo}). Thus, for any other eigenvalue $\lambda$, $|\lambda| < 1$.
	
	Let $\rho_\perp$ be the eigenfunction corresponding to $\lambda$. Since $L$ preserves the normalization of probability density, $\int \rho_\perp = 0$.
	
	For any $\alpha \neq 0$, let $x_0$ be a random variable distributed according to density $\rho_\infty+\alpha \rho_\perp$. We have
	\[
	\rho_{x_k} = L^k (\rho_\infty+\alpha \rho_\perp) = \rho_\infty + \alpha \lambda^k \rho_\perp
	\]
	and therefore the $L_1$ distance satisfies
	\[
	d_1(\rho_{x_k},\rho_\infty) = \alpha \lambda^k \|\rho_\perp\|_1
	\]
	Since densities exist, we have the total variation distance
	\[
	d_{TV}(\rho_{x_k},\rho_\infty) = \frac{1}{2} d_1(\rho_{x_k},\rho_\infty) = \frac{1}{2} \alpha \|\rho_\perp\|_1 \lambda^k
	\]
	Although in general total variation distance cannot be upper bounded by Wasserstein distance, it was shown in \cite{chae2017novel} Lemma 5.1 that such an upper bound exists when both probability distributions admit smooth densities, i.e.,
	\[
	d_{TV}(\rho_{x_k},\rho_\infty) \leq C W_2(\rho_{x_k},\rho_\infty)
	\]
	for some $C\geq 0$. Combined with Thm. \ref{thm_geomErgInW2}, this thus gives
	\[
	d_{TV}(\rho_{x_k},\rho_\infty) \leq \hat{C} \left(\max\{|1-\eta\mu|,|1-\eta\nu|\}\right)^k
	\]
	for some $\hat{C}\geq 0$. Therefore, $|\lambda| \leq \max\{|1-\eta\mu|,|1-\eta\nu|\}=1-\eta \mu$ (the last equality is due to $\mu\leq \nu$ and $\eta<1/\nu$). This leads to $|1-\lambda|\geq \eta\mu$.
\end{proof}

\subsubsection{On Proposition \ref{prop_rescaledGibbsError}}
To prove the bound of difference between  $\mathbb{E}h(\hat{\varphi}(X_0))$ and $\mathbb{E}h(X_0)$, we first prove the following lemma:
\begin{lemma}[gradient estimate of rescaled Gibbs]
	\label{lma_estimate_nablag}
	Suppose $f_0$ is $L$-smooth. Let $x_0$ be the global minimizer of $f_0$. If
	$$f_0(x)-f_0(x_0)\ge C_1||x-x_0||^{k_1}\ \text{and}\ ||\nabla f_0(x)||\le C_2\|x-x_0\|^{k_2}, \quad \forall x\in \mathbb{R}^d,$$
	Then we have, for $X_0$ following rescaled Gibbs \eqref{eq_rescaledGibbs},
	$$\mathbb{E}||\nabla f_0(X_0)||_2^2=\mathcal{O}(\eta^\frac{2k_2-1}{k_1})\quad\text{when}\;\eta\rightarrow0 .$$.
\end{lemma}

\begin{proof}
	\begin{align*}
		\mathbb{E}||\nabla f_0(X_0)||_2^2&=\frac{1}{Z_1}\int||\nabla f_0(x)||_2^2\exp\left(-\frac{2f_0(x)}{\eta}\right)\,dx\\
		&\le\frac{\sqrt[k]{\eta}}{Z_2}\int||\nabla f_0(x)||_2^2\exp\left(-2C_1 (\frac{||x||}{\sqrt[k_1]{\eta}})^{k_1}\right)\,d\frac{x}{\sqrt[k]{\eta}}\\
		&=\frac{\sqrt[k_1]{\eta}}{Z_2}\int||\nabla f_0(\sqrt[k_1]{\eta} u)||_2^2\exp(-2C_1||u||^{k_1})\,du
	\end{align*}
	Since $$||\nabla f_0(x)||\le C_2\|x-x_0\|^{k_2}$$
	So
	\begin{align*}
		\mathbb{E}||\nabla f_0(Y_0)||_2^2&=\frac{\sqrt[k_1]{\eta}}{Z_4}\int C_2(\sqrt[k_1]{\eta}||u||)^{2k_2}\exp(-2C_1||u||^{k_1})\,du\\
		&=\eta^\frac{2k_2-1}{k_1}\frac{1}{Z_4}\int C_2||u||^{2k_2}\exp(-2C_1||u||^{k_1})\,du
	\end{align*}
	The integral converges and is a constant, so we have
	$$\mathbb{E}||\nabla f_0(X_0)||_2^2=\mathcal{O}(\eta^\frac{2k_2-1}{k_1})$$
\end{proof}

\begin{proof}[Proof of Prop. \ref{prop_rescaledGibbsError}]
	Because $\tilde{\zeta}$ is compactly supported and $||\nabla f_0||$ is bounded, Taylor expansion of $h$ in $\eta$ gives, $\forall X$,
	\begin{align*}
		&\mathbb{E}(h(\hat{\varphi}(X)))=\mathbb{E}_X\left[\mathbb{E}_{\tilde\zeta}[h(X-\eta\nabla f_0(X)+\eta\tilde\zeta)|X]\right]\\
		&=\mathbb{E}_Xh(X-\eta\nabla f_0(X))+\eta \mathbb{E}\tilde{\zeta}^\top\mathbb{E}_X\left[\nabla h(X-\eta\nabla f_0(X))\right] \\
		& \qquad\qquad\qquad +\frac{\eta^2}{2}\mathbb{E}_X\left[\mathbb{E}_{\tilde{\zeta}} [\tilde\zeta^\top \text{Hess}\,h(X-\eta\nabla f_0(X))\tilde\zeta|X]\right]+\mathcal{O}(\eta^3)\\
		&=\mathbb{E}_X\left[h(X)-\eta\nabla f_0(X)^\top\cdot\nabla h(X)+\frac{\eta^2}{2}\nabla f_0(X)^\top \text{Hess}\,h(X)\nabla f_0(X)+\frac{\eta^2}{2}\mathbb{E}\tilde{\zeta}^\top \text{Hess}\,h(X)\mathbb{E}\tilde\zeta\right]+\mathcal{O}(\eta^3)
	\end{align*}
	When $X=X_0$, we first estimate the 3rd term. Since $\text{Hess} h$ is bounded and due to the $L$-smoothness and strong convexity of $f_0$, we know it is $\mathcal{O}(\eta^3)$ using Lemma \ref{lma_estimate_nablag} in the case $k_1=k_2=2$.
	So we get
	\begin{align*}
	    &\mathbb{E}(h(\hat{\varphi}(X_0)))-\mathbb{E}h(X_0)\\
	    =&\frac{\eta^2}{2Z}\int\left[-\frac{2}{\eta}\nabla f_0(x)^\top\cdot\nabla h(x)+\sigma^2\text{Tr}\, \text{Hess}\,h(x)\right]\exp\left(-\frac{2f_0(x)}{\eta\sigma^2}\right)\,dx+\mathcal{O}(\eta^3)
	\end{align*}
	And then we use Stokes' theorem to prove the integration in RHS vanishes.
	Denote
	$$\omega:=\sum_i(-1)^i\nabla_i h(x)\exp\left(-\frac{2f_0(x)}{\eta\sigma^2}\right)dx_1\wedge\cdots\wedge\widehat{dx_i}\wedge\cdots\wedge dx_n$$\\
	where $\widehat{dx_i}$ means dropout $dx_i$. Then
	\begin{align*}
		d\omega&=\sum_i\nabla_i^2h(x)\exp\left(-\frac{2f_0(x)}{\eta\sigma^2}\right)-\frac{2}{\eta\sigma^2}\nabla_i h(x)\nabla_i f_0(x)\exp\left(-\frac{2f_0(x)}{\eta\sigma^2}\right)dx_1\wedge...\wedge dx_n\\
		&=(\text{Tr}\,\text{Hess}\,h-\frac{2}{\eta\sigma^2}\nabla h^\top\cdot\nabla f_0)\exp\left(-\frac{2f_0(x)}{\eta\sigma^2}\right)dx_1\wedge\cdots\wedge dx_n\\
	\end{align*}
	According Stokes' formula,
	\begin{align*}
		\mathbb{E}(h(\hat{\varphi}(X)))-\mathbb{E}h(X)&=\frac{\eta^2\sigma^2}{2Z}\int_{\mathbb{R}^d} d\omega+\mathcal{O}(\eta^3)\\
		&=\frac{\eta^2\sigma^2}{2Z}\lim_{r\rightarrow\infty}\int_{B(0,r)}d\omega+\mathcal{O}(\eta^3)\\
		&=\frac{\eta^2\sigma^2}{2Z}\lim_{r\rightarrow\infty}\int_{\partial B(0,r)}\omega+\mathcal{O}(\eta^3)
	\end{align*}
	The first term vanishes since $h(x)$ is compacted supported, which gives us the conclusion that
	$$\mathbb{E}(h(\hat{\varphi}(X_0)))-\mathbb{E}h(X_0)=\mathcal{O}(\eta^3)$$
\end{proof}
\begin{remark}
Note that strong convexity and $L$-smoothness of $f_0$ are sufficient to satisfy the condition of Lemma \ref{lma_estimate_nablag}, but they may not be necessary. In fact, Prop. \ref{prop_rescaledGibbsError} is also correct for any $f_0$ that satisfies
	$$f_0(x)-f_0(x_0)\ge C_1||x-x_0||^{k_1}\ \text{and}\ ||\nabla f_0(x)||\le C_2\|x-x_0\|^{k_2}, \quad \forall x\in \mathbb{R}^d,$$
	where $2k_2-1\ge k_1$. Although we only proved that the rescaled Gibbs approximates the invariant distribution when $f_0$ is strongly convex functions, the fact that rescaled Gibbs nearly satisfies the invariance equation does not require strong convexity. In fact, we conjecture that rescaled Gibbs also approximates the invariant distribution for convex and even nonconvex $f_0$. See numerics in Sec.\ref{NT_Ergodicity} ($f_0=x^4/4$, with $k_1=4$, $k_2=3$) and Appendix \ref{sec_NT_nonvonvex} (nonconvex and multimodal $f_0$).
\end{remark}
\subsubsection{On Theorem \ref{thm_rescaledGibbsIsApprx}}
\begin{proof}
    Denote (as before) by $L$ the transition operator of the Markov process generated by $\hat{\varphi}$.
	Consider a deviation function
	\[
	r:=\rho_\infty - \tilde{\rho}.
	\]
	Decompose $r$ as an orthogonal sum
	\[
	r=r_1+r_0   \quad \text{where } r_1 \in \ker(I-L) \text{ and } r_0 \perp \ker(I-L)
	\]
	Since $\hat{\varphi}$ induces a geometric ergodic process, $\dim \ker (I-L)=1$, and thus
	\[
	r=\gamma \rho_\infty + r_0  \quad \text{for some scalar }\gamma.
	\]
	Since $L\rho_\infty = \rho_\infty$ and $L\tilde{\rho}=\tilde{\rho}+\mathcal{O}(\eta^3)$ (Prop.\ref{prop_rescaledGibbsError}; note weak-* topology is metrizable on a separable space), we have $(I-L)r=\mathcal{O}(\eta^3)$, and consequently
	\[
	(I-L)r_0 = \mathcal{O}(\eta^3)
	\]
	Since $r_0$ is orthogonal to $\ker(I-L)$ which is the eigenspace associated with eigenvalue 1 of $L$, and all eigenvalues of $I-L$, except for the the irrelevant 0, satisfy $|\lambda| \geq \mu\eta$ due to Cor.\ref{cor_spectralGap}, we obtain
	\[
	r_0=\mathcal{O}(\eta^2).
	\]
	This means $\rho_\infty-\tilde{\rho}=\gamma \rho_\infty + \mathcal{O}(\eta^2)$. Since $\rho_\infty$ and $\tilde{\rho}$ are both density functions that normalize to 1, applying a uniform test function and letting its support go to infinity give $0=\gamma+\mathcal{O}(\eta^2)$. This yields eq.\ref{eq_rescaledGibbsIsApprx}.
\end{proof}

\begin{remark}
	The invariant distribution can be approximated by not only rescaled Gibbs but a Gaussian if $f_0$ is strongly convex. Here is the intuition of a more general result: \\
	Consider rescaled Gibbs \eqref{eq_rescaledGibbs}. Due to the small $\eta$ at the denominator, $X_0$ assumes small values with exponentially large probability. We thus can formally Taylor expand $f_0(x)$ about $x=0$, which we assumed WLOG to be the minimizer. Denote the first nonzero derivative of $f_0$ at 0 by the k$th$ one. Then $f_0(x)\approx \frac{1}{k!} f_0^k(0) x^k$. So, from the density of rescaled Gibbs, we see the density of $\frac{X_0}{\sqrt[k]\eta}$ can be approximated by
	$$\frac{X_0}{\sqrt[k]\eta}\sim\frac{1}{Z}\exp\left(\frac{-2f_0^k(0)}{k!\sigma^2}x^k\right)$$
	Note that iff $f_0$ is strongly convex, $k=2$, and one gets a Gaussian approximation.
	\label{rmk_invDis_Taylor}
\end{remark}

\begin{remark}
	If one considers another stochastic map $\tilde{\varphi}(x):=x-\eta\nabla f_0(x)+\eta \sigma \xi$ where $\xi$ is standard i.i.d. Gaussian, $\tilde{\varphi}(x)$ admits, under the same Lipschitz and convexity conditions, a similar limiting invariant distribution $\frac{1}{Z}\exp\left(-\frac{2f_0(x)}{\eta \sigma ^2}\right)$ will be obtained. The key difference is, unlike $\tilde{\varphi}$ which uses unbounded noise and is the discretization of an SDE, our stochastic map $\hat{\varphi}$ uses only bounded noise as it mimicks the deterministic map $\varphi$.
\end{remark}

\subsection{On the deterministic map $\varphi$}
\subsubsection{counter-examples}
Here are the complete version of the 2 counter-examples given in Sec. \ref{sec_deterministic_map}.

\begin{theorem}[a sufficient condition for the nonexistence of nondegenerate invariant distribution]
	\label{counter-example1}
	When $d=1$, for any fixed $\epsilon$ and fixed periodic $f_1\in \mathcal{C}^2(\mathbb{R})$, for any $\eta_0$, there exists $\eta>\eta_0$ and $f_0\in \mathcal{C}^2$ such that $|f_0'|$ and $|f_0''|$ (but 3-order or more derivative will explode) are arbitrarily small. For such $f_0$, the orbit starting at any point is bounded but $\varphi$ does not admit a nontrivial invariant distribution.
\end{theorem}

\begin{proof}
	\begin{equation*}
		\varphi'(x)=1-\eta f_0''(x)-\frac{\eta}{\epsilon}f_1''\left(\frac{x}{\epsilon}\right)
	\end{equation*}
	Because of the continuity of $f_1''$, $1-\frac{\eta}{\epsilon}f_1''(\frac{x}{\epsilon})$ has a zero point, denote as $x_0$. So we can choose $\delta$ to make $\frac{1-\eta/\epsilon f_1''(x/\epsilon)}{\eta}$ arbitrarily small on the interval $I=[x_0-\delta,x_0+\delta]$. Then construct $f_0|_I$ and $\eta$ making $\varphi'\equiv0$ on $I$. After that, we adjust $f_0$ to make $\varphi(x_0)$, which is not in $I$,  be a fixed point of $\varphi$. According to the property of Li-Yorke chaos, all the point will be finally mapped to $I$, and then to $\varphi(x_0)$ and never move. So the nontrivial invariant distribution does not exist.
\end{proof}

\begin{theorem}[another sufficient condition for the nonexistence of invariant distribution]
	\label{counter-example2}
	When $d=1$, $\forall$ fixed $f_0\in \mathcal{C}^2$ and $\eta>0$, there exists periodic $f_1\in \mathcal{C}^2$ whose period is $1$ and 0,1,2-order derivative is arbitrary small, together with an $\epsilon$ arbitrarily small, making nontrivial invariant distribution not exist.
\end{theorem}

\begin{proof}
	Choose $f_1$ s.t. $\nabla^2 f_1(\frac{x}{\epsilon})\equiv\frac{\epsilon}{\eta}(1-\eta\nabla^2f_0(x))$ on a interval $[0,\delta]$ where $\delta\ll\epsilon$ and make $f_1$ and $f_1'$ arbitrarily small on $[0,\delta/\epsilon]$, and choose $f_1$ on $[\delta/\epsilon]$ to ensure continuity and smoothness. We can make $\epsilon\rightarrow0$ to make $f_1''$ small. Then choose a specific $\epsilon$ to make $\varphi(0)$ is a fix point. According to the property of Li-Yorke chaos, all the point will be finally mapped to $[0,\delta]$, then to $\varphi(0)$ and never move. So the nontrivial invariant distribution does not exist.
\end{proof}

\begin{remark}
The requirements for $\eta$ to be arbitrarily large in Theorem \ref{counter-example1} and $\epsilon$ to be arbitrarily small in Theorem \ref{counter-example2} ensure the system won't converge to a local minimum created by $f_1$, and from the construction of the counter-examples, we know the system is not the other trivial one, which means the system explodes because $\eta$ is too large. 
\end{remark}
\begin{remark}
	Here we give some intuition of Thm.\ref{counter-example1} and \ref{counter-example2}. Thm.\ref{LYThm} will show that in 1-dim case, if we have a period-3 orbit, then there exists a subset $S$ of the whole space $J$ satisfying: For every $x_1,x_2\in S$ with $x_1\neq x_2$, $\liminf_{n\rightarrow\infty}|\varphi^{(n)}(x_1)-\varphi^{(n)}(x_2)|=0$. So the intuition for proving Thm. \ref{counter-example1} and \ref{counter-example2} is to make $\varphi\equiv0$ on a small interval, then all the points that drop in this interval will be mapped to a single fixed point of $\varphi$.
\end{remark}

\subsubsection{Period Doubling}
\label{sec_PD}
When $\eta$ is small, each (local) minimizer of $f$ corresponds to a stable fixed point of $\varphi$, which is thus also a periodic orbit of $\varphi$ with period 1. As $\eta$ increases, this point remains as a fixed point but will become unstable. Instead, the previously stable periodic orbit bifurcates into a stable periodic orbit with period 2, and the period similarly keeps doubling as $\eta$ further increases. Eventually, the period becomes arbitrarily large before a finite value of $\eta$, as will be numerically illustrated in Sec.\ref{sec_experimentPeriodDoubling}. This phenomenon is known as period doubling, which is a common route to chaos (e.g., \cite{alligood1997chaos,ott2002chaos}); after the appearance of arbitrarily large period, the system enters $\eta$ regime that corresponds to chaotic dynamics.

We now explain how this relates to what we call global and local chaos, which are specific to our multiscale problem.

When $\eta \ll \epsilon$, we know GD converges to a local minimum of $f$ corresponding to one of the many potential wells of created by $f_{1,\epsilon}$. This is the non-interesting case. 

When $\eta$ approaches some order function of $\epsilon$ 
describing the width of microscopic potential wells of $f_{1,\epsilon}$ (for the periodic case, this is $\mathcal{O}(\epsilon)$), the orbit is still trapped in a single microscopic potential well, but it starts making jumps within the well. In fact, restricted to any potential well, $\varphi$ becomes a unimodal map (see e.g., \cite{strogatz2018nonlinear}) and its dynamics is known to eventually become chaotic as $\eta$ exceeds a critical value. This is where the period of a periodic orbit keeps on doubling and becomes arbitrarily large. The classical method for studying the invariant distribution of unimodal chaotic maps applies here (see e.g., \cite{cvitanovic2017universality}). This is the local chaos regime.

Even more interesting is the case when $\eta$ gets even larger, large enough for the orbit to jump out of a single potential well created by $f_{1,\epsilon}$ and navigate the landscape of $f_0$. This is what we call global chaos. For this, Thm.\ref{thm_sameLimitStats} and \ref{thm_geometricErgodicity} characterize the combined effect of chaos and global behavior of $f_0$.

\subsubsection{About Li-Yorke Chaos}

\begin{definition}[Li-Yorke chaos]
\label{def_LY}
		Let $J$ be an interval and let $F: J\rightarrow J$ be continuous. The dynamical system generated by $F$ exhibits Li-Yorke chaos if
		\begin{enumerate}
			\item For any $k=1,2,...$, there is a periodic point in $J$ having period $k$.
			\item There is an uncountable set $S\subset J$ containing no periodic points, that satisfies:\\
			(A) For every $p,q\in S$ with $p\neq q$, $\limsup_{n\rightarrow\infty}|F^n(p)-F^n(q)|>0$ and $\liminf_{n\rightarrow\infty}|F^n(p)-F^n(q)|=0$.
			
			(B) For every $p\in S$ and periodic point $q\in J$, $\limsup_{n\rightarrow\infty}|F^n(p)-F^n(q)|>0$.
		\end{enumerate}
	\end{definition}
	\begin{theorem}[period 3 implies chaos]
		\label{LYThm}
		If there exists $a\in J$ for which $b=F(a)$, $c=F^2(a)$, and $d=F^3(a)$ satisfy $d\le a<b<c\text{ or }d\ge a>b>c$, then $F$ induces Li-Yorke chaos.
\end{theorem}
\begin{remark}
About Thm.\ref{LYThm}, see \cite{sharkovskiui1995coexistence, li1975period} for rigorous theorems and proofs. This is one of the most celebrated result in chaotic dynamics, which tells us that period 3 implies chaos. The 1st conclusion is named after Sharkovskii. The 2nd conlusion in this theorem is also generalized to be the definition of Li-Yorke Chaos in multi-dim case.
\end{remark}
\begin{proof}[Proof of Thm.\ref{thm_condition4LiYorke}]
	\begin{figure}[H]
		
		\centering
		\includegraphics[width=0.8\linewidth]{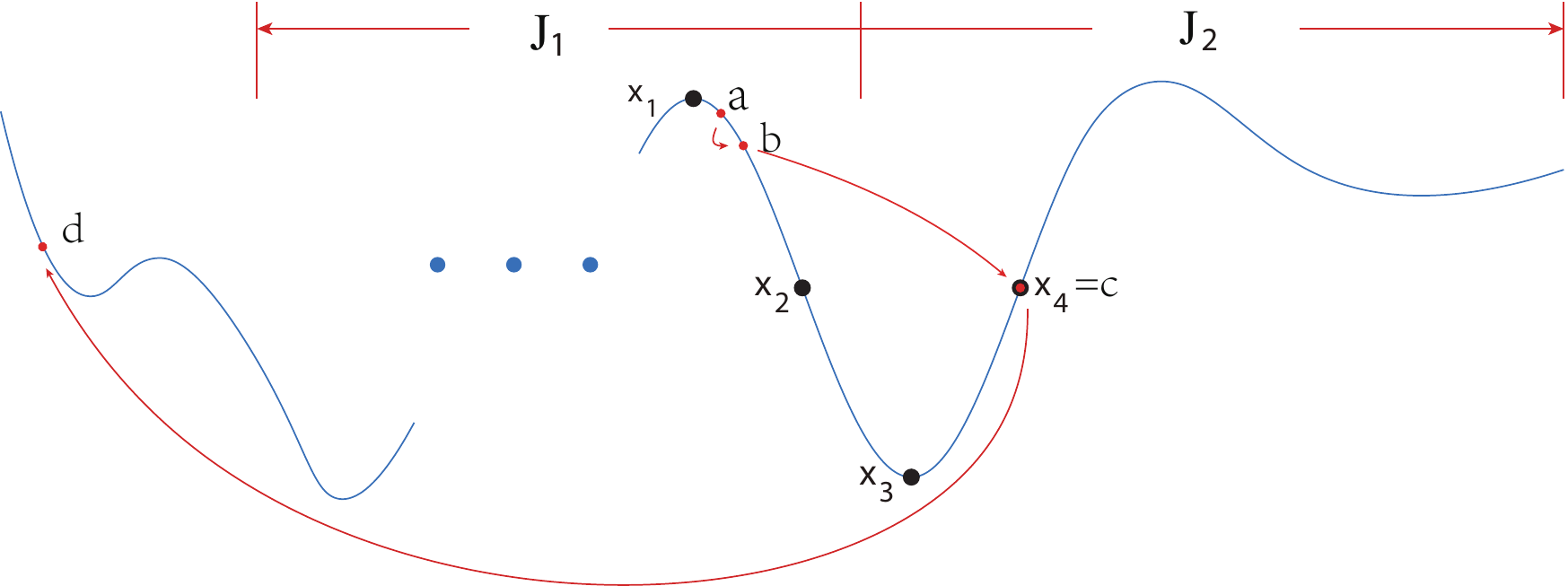}
		\caption{Guideline to finding a period-3 orbit}
		\label{fig_period3}
	\end{figure}
	First we show there exists an interval $J$, such that when $0<\eta<1/L$, $\varphi(J)\subset J$. WLOG, suppose $f_0(0)=0$. According to Cond. \ref{cond_1}, there exists $\epsilon_1$, when $\epsilon<\epsilon_1$, $\sup_x\|\nabla f_{1,\epsilon}(x)\|$ is uniformly bounded w.r.t. $\epsilon$. Denote the upper bound as $R$. Due to the $L$-smoothness of $f_0$, 
	\begin{align*}
	    \limsup_{x\rightarrow+\infty}[\varphi(x)-x]\le&\limsup_{x\rightarrow+\infty}[-\eta f_0'(x)+\eta R]<-C<0\\
	    \liminf_{x\rightarrow+\infty}[\varphi(x)+x]\ge&\liminf_{x\rightarrow+\infty}[2x-\eta f_0'(x)+\eta R]\ge\liminf_{x\rightarrow+\infty}[(2-\eta L)x+\eta R]>C>0
	\end{align*}
	where $C>0$ is a constant. So there exists $M_1$ such that $-x<\varphi(x)<x$ when $x>M_1$. Similarly, we have $M_2$ such that $x<\varphi(x)<-x$ when $x<-M_2$.
	
	So there exists $M:=\max(M_1,M_2)$, so when $|x|>M$,  $-|x|<\varphi(x)<|x|$. Set $J:=[\inf_{x\in[-M,M]}\varphi(x),\sup_{x\in[-M,M]}\varphi(x)]$ and we have $\varphi(J)\subset J$ when $\epsilon<\epsilon_1$.
	
	Next, we try to find $a$, $b$, $c$ and $d$ in Thm. \ref{LYThm}. Because $P(\zeta=0)<1$, $\exists\delta_0>0$ s.t.$P(\zeta>\delta_0)>0$ and $P(\zeta<-\delta_0)>0$. Since $\nabla f_0$ have a zero point, we can find an interval $\tilde{J}$ on which $|\nabla f_0|<\delta_0/3$. Denote the middle point of $x_0$. Find a subinterval of $\tilde{J}$, whose length $\le\eta/\frac{\delta_0}{3}$ and denote as $J$. Divide $J$ into 2 parts of similar length $J_1$ and $J_2$. $\exists \epsilon_1$, s.t. when $\epsilon<\epsilon_1$, $|\min_{J_i}\nabla f_{1,\epsilon}|,|\max_{J_i}\nabla f_{1,\epsilon}|>\frac{2}{3}\delta_0,i=1,2$. So now we have that $|\inf_{J_i}\nabla f|,|\sup_{J_i}\nabla f|>\delta_0/3$. Which means we can find $x_1,x_2\in J_1$, $x_3,x_4\in J_2$ and $x_1<x_2<x_3<x_4$ satisfying $\varphi(x_1)=x_1$, $\varphi(x_2)>x_4$, $\varphi(x_3)=x_3$, $\varphi(x_4)<x_1$.
	
	Let $c=x_4$, and $d=\varphi(c)$. So we have $\varphi(x_2)>c$. And since $\varphi(x_1)=x_1$ and continuity, $b\in[x_1,x_2]$ s.t.$\varphi(b)=c$. By the same way we get $a\in[x_1,b]$ s.t. $\varphi(a)=b$. Let $\epsilon_0:=\min(\epsilon_1,\epsilon_2)$. Based on Thm.\ref{LYThm}, we deduct that the discrete dynamical system induced by $\varphi$ is chaotic in Li-Yorke sense when $\epsilon<\epsilon_0$ and $0<\eta<1/L$.
\end{proof}

\begin{remark}[Beyond Li-Yorke Chaos]
  (\emph{Thanks to valuable comments from Fryderyk Falniowski.}) Here the 3-periodic orbit of $\varphi$ can be used to establish a positive topological entropy \citep{misiurewicz2010horseshoes}, which implies not only Li-Yorke chaos but also distributional chaos, as well as the existence of a subsystem chaotic in the sense of Devaney \citep{li1993} (see e.g., \cite{aulbach2001three,falniowski2015two} for their differences). So far these are only known in 1D though.
\end{remark}

\subsubsection{On the Lyapunov exponent}
\begin{proof}[Proof of Thm.\ref{thm_LyapExp}]
	All the norms for matrix in this proof is 2-norm (for simplicity, we omit its subscript).
	
	Denoted by $\nu$ the invariant distribution of the deterministic map. Denote the special map where is $f_0\equiv 0$ as $\varphi_0$:
	$$\varphi_0(x)=x-\eta\nabla f_{1,\epsilon}(x)$$
	With ergodicity, when $\epsilon\rightarrow0$, we have
	\begin{align*}
		\lambda(x)&=\lim_{n\rightarrow\infty}\frac{1}{n}\sum_{i=1}^{n}\ln||\nabla\varphi|_{\varphi^{(i)}(x)}||\\
		&=\int\ln||\nabla\varphi|_x||\,\nu(dx)\\
		&=\int\ln||\nabla\varphi_0|_x+\eta\text{Hess} f_0(x)||\,\nu(dx)\\
	\end{align*}
	Since  $\text{Hess}f_0$ is bounded, we know that
	\begin{align*}
		\lambda(x)&=\int\ln||\nabla\varphi_0|_x||\,\nu(dx)+\mathcal{O}(\eta)\\
	\end{align*}
	And then, we choose a bounded set $T$ and a mesh of which, denoted as $\Delta=\bigsqcup_{i\in\mathcal{I}}\Gamma_i$, $\forall\delta>0$, we have $\mu$ is a simple function which is constant on each $Gamma_i$, where $\text{supp}\mu\subset T$, $\int|\mu-\nu|\,dx<\delta$. Denoted the bound of $\epsilon \nabla^2f_{1,\epsilon}=A$, then
	\begin{align*}
		\lambda(x)&=\sum_{i\in\mathcal{I}}\int_{\Gamma_i}\ln||\nabla\varphi_0|_x||\,\nu(dx)+\mathcal{O}(\eta)\\
		&=\sum_{i\in\mathcal{I}}\int_{\Gamma_i}\ln||\nabla\varphi_0|_x||\,(\mu+(\nu-\mu))dx+\mathcal{O}(\eta)\\
		&=\ln\left(\frac{\eta}{\epsilon}\right)+\sum_{i\in\mathcal{I}}\int_{\Gamma_i}\ln||\epsilon\nabla^2 f_1(y)||\,(\mu+(\nu-\mu))dx+\mathcal{O}(\eta)
	\end{align*}
	where $\sum_{i\in\mathcal{I}}\int_{\Gamma_i}\ln||\nabla\varphi_0|_x||\mu(dx)\rightarrow m$ and $\sum_{i\in\mathcal{I}}\int_{\Gamma_i}\ln||\nabla\varphi_0|_x||(\nu-\mu)(dx)<\delta A\rightarrow0$. So we know that $\lambda(x)-\ln\left(\frac{\eta}{\epsilon}\right)\rightarrow m$ when $\epsilon\rightarrow0$ first and then $\eta\rightarrow0$.
\end{proof}
\begin{remark}
	Here we need $\varphi$ to be ergodic, which means the distribution of a single trajectory converges to the invariant distribution of the chaotic dynamical system. We don't have a reference, but please see section \ref{NT_Ergodicity} for numerical test.
\end{remark}
\begin{remark} 
    One may ask why $f_0$ doesn't appear in $m$. The reason is, the microstructure creates both local and global chaos, not the macrostructure; in fact, since $L \ll 1/\epsilon$, $L$ for the $L$-smooth $f_0$ gets absorbed in the high-order term in the proof.
\end{remark}
\begin{remark}
When $f_1$ is periodic and $f_{1,\epsilon}=\epsilon f_1(x/\epsilon)$, we have an estimation of the order of convergence. 

We divide the support of the invariant distribution into small parts according to the period of $\epsilon f_1(x/\epsilon)$, and enumerate them with $A_j,j\in \mathbb{N}$.
\begin{align*}
	\lambda(x)&=\sum_{i}\int_{A_j}\ln||\nabla\varphi|x||\,\nu(dx)+\mathcal{O}(\eta)\\
	&=\sum_{i}\int_{A_j}\frac{1}{\epsilon|\Gamma|}\left(\int_{\epsilon\Gamma}\ln||\nabla^2 f_{1,\epsilon}(y)||dy+\mathcal{O}(\epsilon)\right)\,\nu(dx)+\mathcal{O}(\eta)\\
	&=\ln\left(\frac{\eta}{\epsilon}\right)+\frac{1}{|\Gamma|}\int_{\Gamma}\ln||\nabla^2 f_1(y)||\,dy+\mathcal{O}(\epsilon+\eta) \\
	&=\ln\left(\frac{\eta}{\epsilon}\right) + m + \mathcal{O}(\epsilon+\eta).
\end{align*}
\end{remark}

\section{A possible origin of multiscale landscape from neural networks}
\label{sec_MultiscaleLandscape}
It is possible that the (training) loss of a neural network satisfies the multiscale requirement of the presented theory. Here is an illustration in which multiscale training data together with periodic activation leads to a multiscale loss:

Consider the training of a 2-layer neural network to fit data $\{x^k,y^k\}_k$, where the output $y^k=y_0^k+y_1^k+\xi^k$ admits a decomposition into large scale behavior $y_0^k=g_0(x^k)$, microscopic detail $y_1^k=\epsilon g_1(\epsilon x^k)$, and i.i.d. noise $\xi_k$. Assume $g_0$ and $g_1$ are regular enough so that universal approximation (UA) works and they can be approximated by wide enough neural networks with $\mathcal{O}(1)$ weights. Consider MSE loss $\sum_k \|y^k-\sum_i a_i \sigma(W_i x^k+b_i)\|^2$ with $\sigma$ being the periodic activation in a recent progress \citep{sitzmann2020implicit}. Then the loss admits a minimizer and in its neighborhood the loss satisfies Cond.\ref{cond_1}\&\ref{cond_2} for the following reason: omit $k$ without loss of generality, absorb bias into weight, and rewrite the loss as (denoting $\theta=[a_i,W_i]_i$)
\begin{align*}
    & f(\theta)=\Big\| y_0 - \textstyle\sum_{i\in I} a_i \sigma(W_i x) + \epsilon y_1 - \textstyle\sum_{j \not\in I} a_j \sigma(W_j x) \Big\|^2
= \Big\|g_0(x) - \textstyle\sum_{i\in I} a_i \sigma(W_i x)\Big\|^2 \\
& \qquad +2\epsilon \Big\langle g_0(x) - \textstyle\sum_{i\in I} a_i \sigma(W_i x), g_1(\epsilon x) - \textstyle\sum_{j \not\in I} a_j \sigma(W_j x) \Big\rangle + \epsilon^2 \Big\|g_1(\epsilon x) - \textstyle\sum_{j \not\in I} a_j \sigma(W_j x)\Big\|^2
\end{align*}
where $I$ and $I^c$ are sets of nodes, each large enough for UA to ensure vanishing loss. Renormalize by letting $\hat{x}=\epsilon x$ so that UA works for $g_1(\cdot)$, then the 2nd term rewrites as
\[
2\epsilon \Big\langle g_0(x) - \textstyle\sum_{i\in I} a_i \sigma(W_i x), g_1(\hat{x}) - \textstyle\sum_{j \not\in I} a_j \sigma\Big(\frac{W_j}{\epsilon} \hat{x}\Big) \Big\rangle.
\]
This is in the form of $\epsilon \hat{f}_1(\theta/\epsilon,\theta)$ for some $\hat{f}_1(\phi,\varphi)$ that is quasiperiodic in $\phi$ (quasiperiodic because $\hat{x}$ is multi-dim). The 3rd term rewrites similarly. Thus, we see $f(\theta) = f_0(\theta) + f_{1,\epsilon}(\theta)$ where $f_0$ is the 1st term and $f_{1,\epsilon}(\theta)=\epsilon \hat{f}_1 (\theta/\epsilon,\theta) + \epsilon^2 \hat{f}_2 (\theta/\epsilon,\theta)$ for some $\hat{f}_1,\hat{f}_2$ quasiperiodic in the 1st argument. Such $f_{1,\epsilon}$ satisfies Cond.\ref{cond_1}\&\ref{cond_2} due to its quasiperiodic micro-scale. \hfill$\square$\vspace{4pt}

\section{More numerical evidence}
\label{sec_moreNumerics}
\subsection{Period doubling}
\begin{wrapfigure}{r}{0.50\textwidth}
    \vspace{-20pt}
	\centering
	\includegraphics[width=0.55\textwidth]{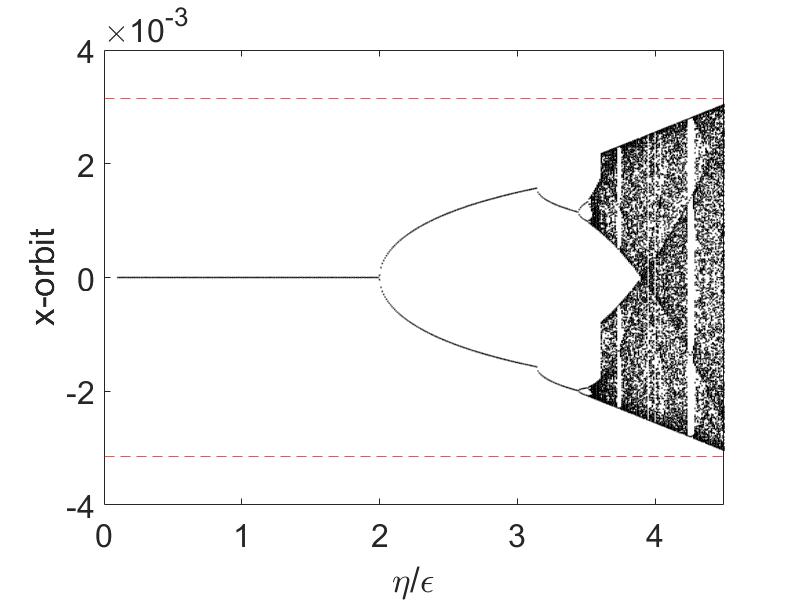}
	\caption{Bifurcation diagram of GD with $\epsilon=10^{-3}$, $f_0=x^4/4$ and $f_{1,\epsilon}=-\epsilon\cos(x/\epsilon)$.}
\label{sec_experimentPeriodDoubling}
\end{wrapfigure}
We illustrate numerically that $\varphi$, when viewed as a family of maps indexed by LR $\eta$, keeps undergoing period doubling bifurcation as $\eta$ increases, and the period of $\eta$ eventually approaches infinite at a finite $\eta$ value, which is the chaos threshold (e.g., \cite{alligood1997chaos}, Chap 11). This observation is rather robust to $f_0$, and we choose a convex but not strongly-convex example for an illustration.

The bifurcation diagram is plotted in Fig.\ref{sec_experimentPeriodDoubling}. For each $\eta$ value, we start with a fixed initial condition and iterate it using GD dynamics ($\varphi$) for sufficiently long so that the dynamics settle into an attractor, and then draw each of the thereafter iterations as a point on the diagram. For example, one can read from Fig.\ref{sec_experimentPeriodDoubling} that there are two points at $\eta=2.5\epsilon$, corresponding to an orbit of period 2. Although limited by the numerical resolution, one can see that the chaos threshold in this case is around $\eta \approx 3.5\epsilon$.

Worth mentioning is that the chaos that first onsets is a local one, happening in a (and every) small potential well created by $f_{1,\epsilon}$. In other words, before global chaos for which LR is so large that GD can escape local well, arbitrarily large period already appears and chaos already onsets. This can be seen from Fig.\ref{sec_experimentPeriodDoubling} as the boundaries of a small potential well, which is approximately $[-\epsilon\pi,\epsilon\pi]$, are marked by red dashed lines.

\subsection{A multi-dimensional demonstration}
\label{sec_Matyas}
Our sufficient condition for chaos (Thm.\ref{thm_condition4LiYorke}) is restricted to 1D problems, although our connection between $\varphi$ and $\hat{\varphi}$ limiting statistics (Sec.\ref{sec_SB}) and the approximation of $\hat{\varphi}$ limiting statistics (Sec.\ref{sec_rescaledGibbs}) work for any finite dimension. We conjecture that stochasticity also appears in large LR GD for multidimensional multiscale objective functions. A numerical experiment consistent with this conjecture is presented, based on a classical strongly convex test function of Matyas:

Let $f_0$ be defined as
$$f_0(x,y)=0.26(x^2+y^2)+0.48xy.$$
The small scale is arbitarily chosen to be
$$f_{1,\epsilon}(x,y)=\epsilon\sin(x/\epsilon)+\epsilon\cos(y/\epsilon),\epsilon=10^{-7}.$$

The evolution of the empirical distribution of an ensemble, respectively under GD $\varphi$ and the stochastic map $\hat{\varphi}$, is shown in Fig.\ref{fig_Matyas_StochVsChaotoc}, where good agreement is observed. The GD empirical distribution is also compared with rescaled Gibbs in Fig.\ref{fig_Matyas_InvDis}, where results again agree.

\begin{figure}[H]
	\centering
	\hfill
	\subfigure[Deterministic map]{\includegraphics[width=0.45\linewidth]{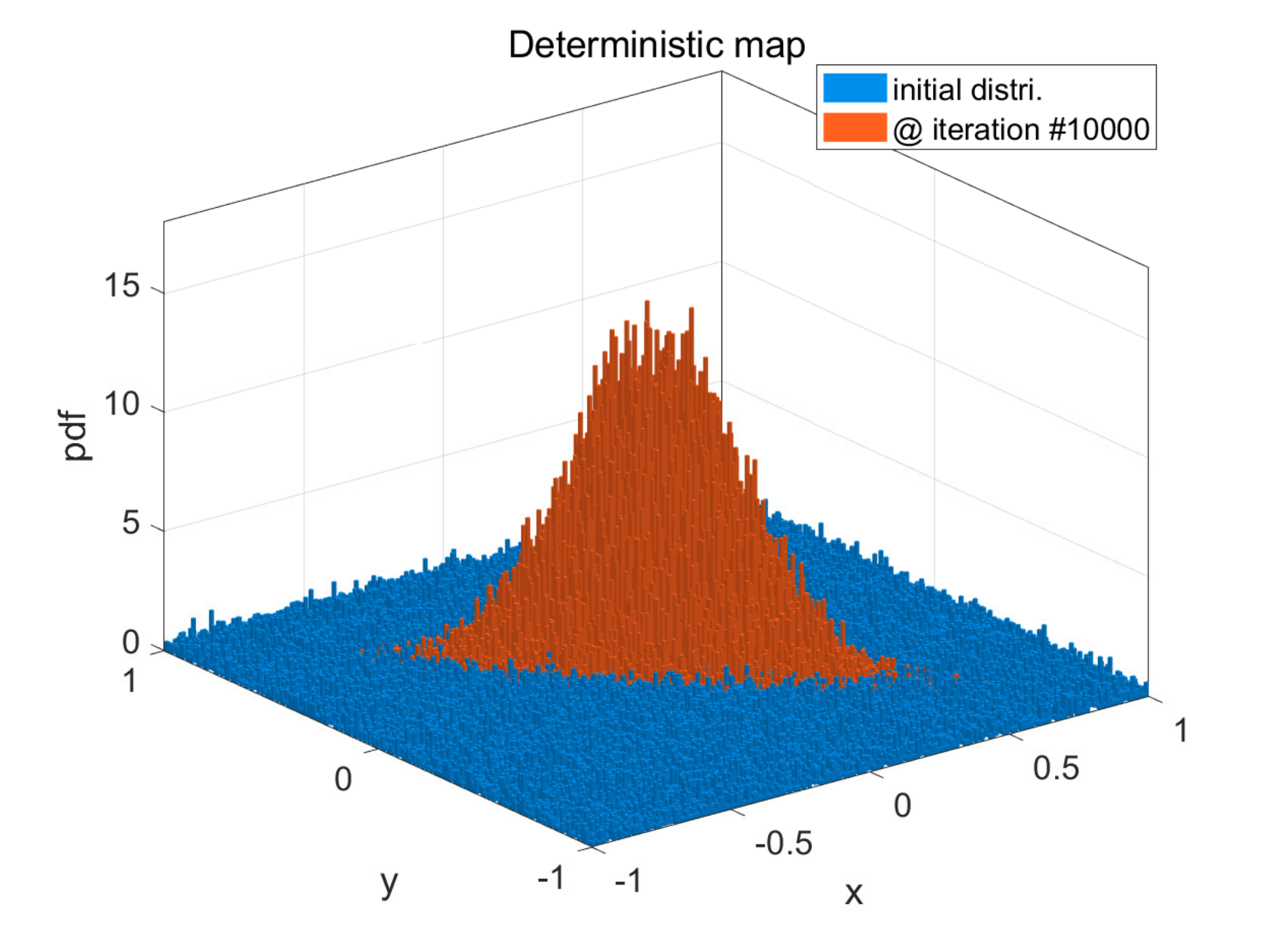}}
	\hfill
	\subfigure[Stochastic map\label{fig_Matyas_StochVsChaotoc_b}]{\includegraphics[width=0.45\linewidth]{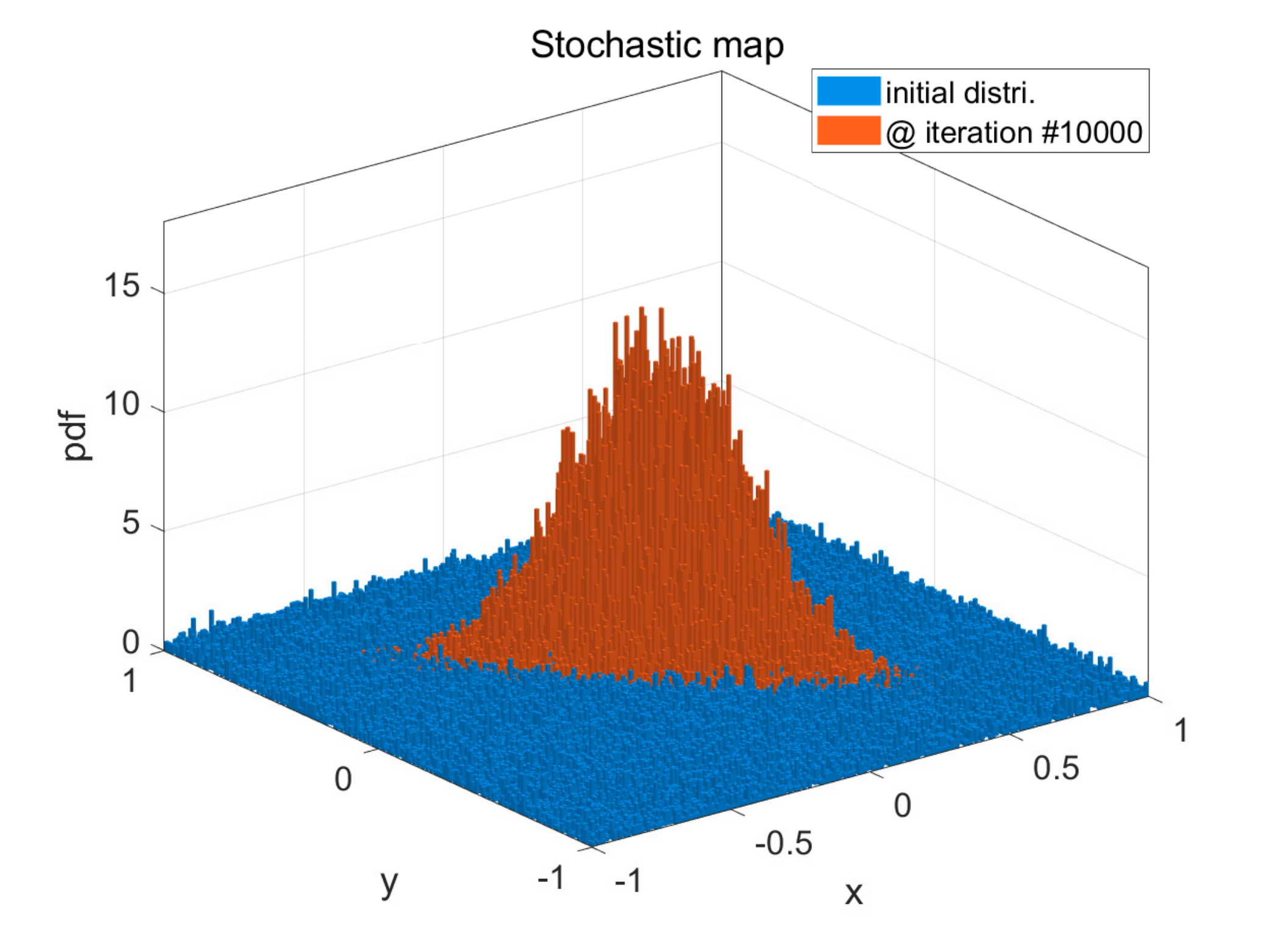}}
	\hfill
	\caption{Comparison between the deterministic map and the stochastic map on Matyas function ($\eta=0.01$) for  testing Thm.\ref{thm_sameLimitStats}. Agreed histograms suggests that the limiting distributions of the two maps are close.}
	\label{fig_Matyas_StochVsChaotoc}
\end{figure}

\begin{figure}[H]
	\centering
	\hfill
	\subfigure[$\eta=0.1$]{\includegraphics[width=0.3\linewidth]{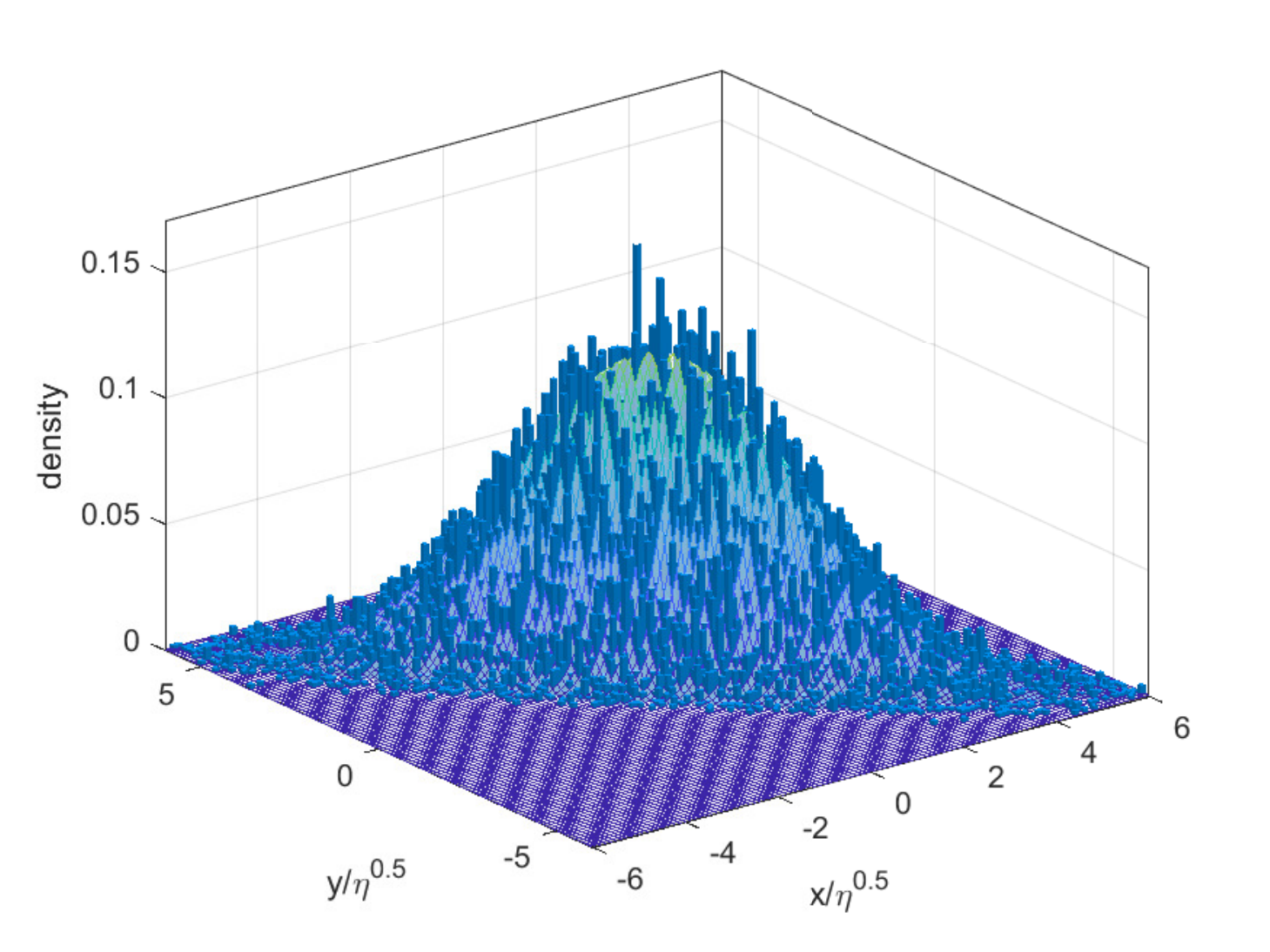}}
	\hfill
	\subfigure[$\eta=0.01$]{\includegraphics[width=0.3\linewidth]{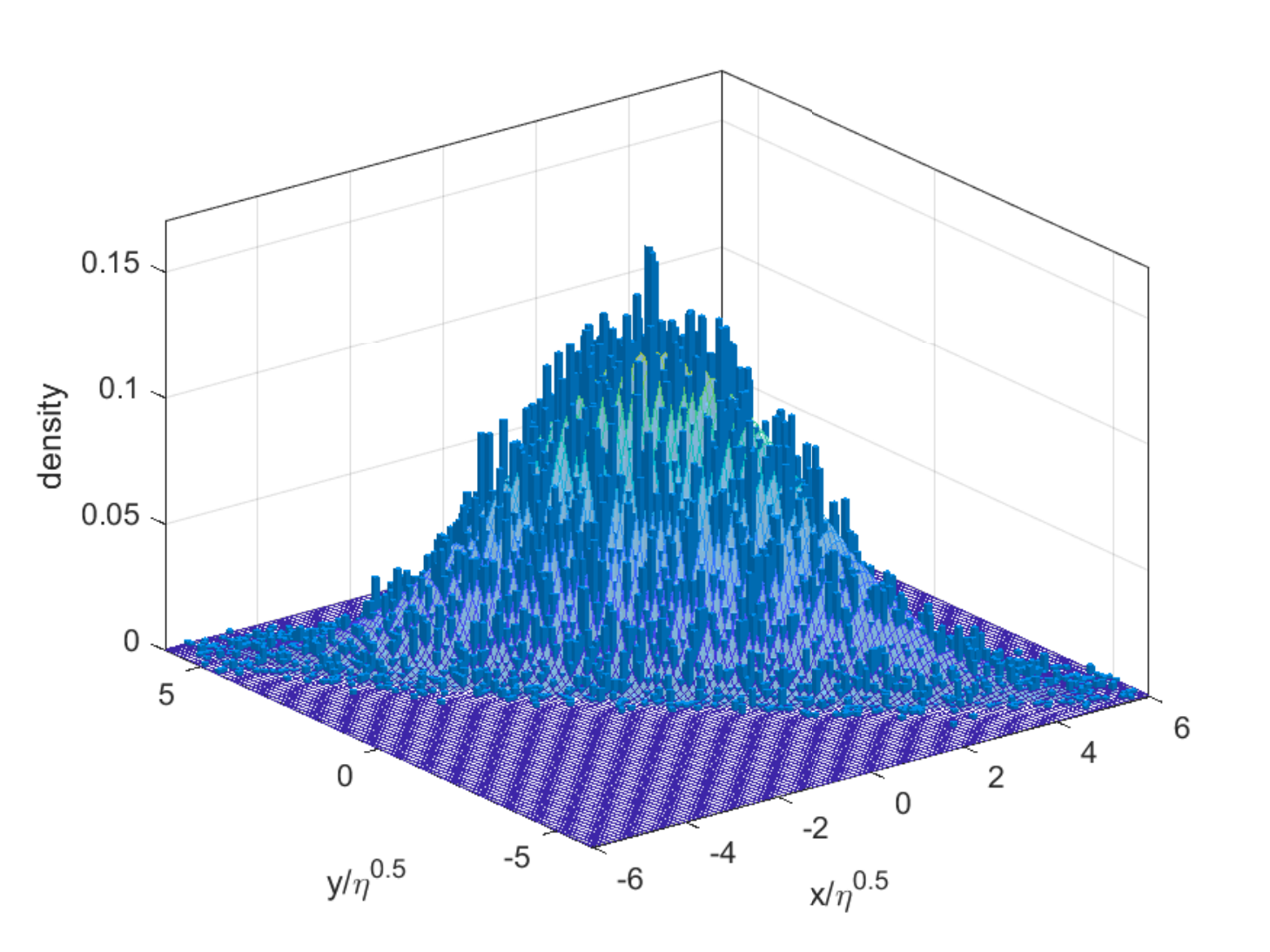}}
	\hfill
	\subfigure[$\eta=0.001$]{\includegraphics[width=0.3\linewidth]{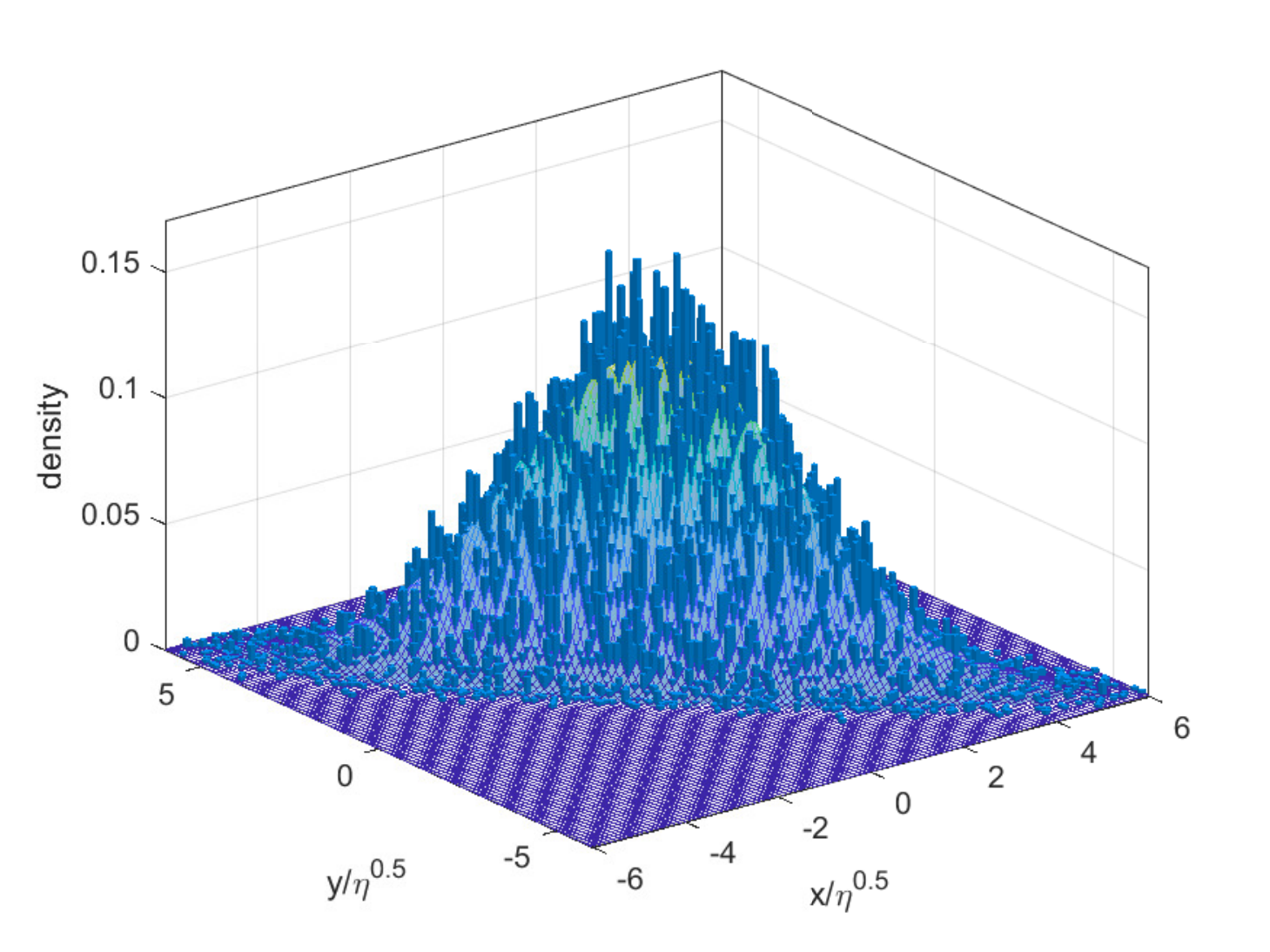}}
	\hfill
	\caption{Test for the explicit expression of the invariant distribution. The surface is rescaled Gibbs and the histogram is the experiment result. They are overplotted after a rescaling by $\sqrt\eta$ in both axis. Obersved agreement is consistent with the rescaled Gibbs approximation.}
	\label{fig_Matyas_InvDis}
\end{figure}
In terms of deterministic chaos, although our sufficient condition for chaos (Thm.\ref{thm_condition4LiYorke}) is only for 1-dim., the Lyapunov exponent estimate (Thm.\ref{thm_LyapExp}) works for any finite dimension as it assumes already ergodicity. Here we observe numerically that the deterministic map is chaotic and mixing (thus ergodic) despite of the $\geq 2$ dimension: see Fig.\ref{fig_Matyas1orbit} for the statistical behavior of a single orbit. A comparison with Fig.\ref{fig_Matyas_StochVsChaotoc} gives agreement in the statistics.
\begin{figure}[H]
	\centering
	\subfigure[Histogram of a trajectory]{\includegraphics[width=0.3\linewidth]{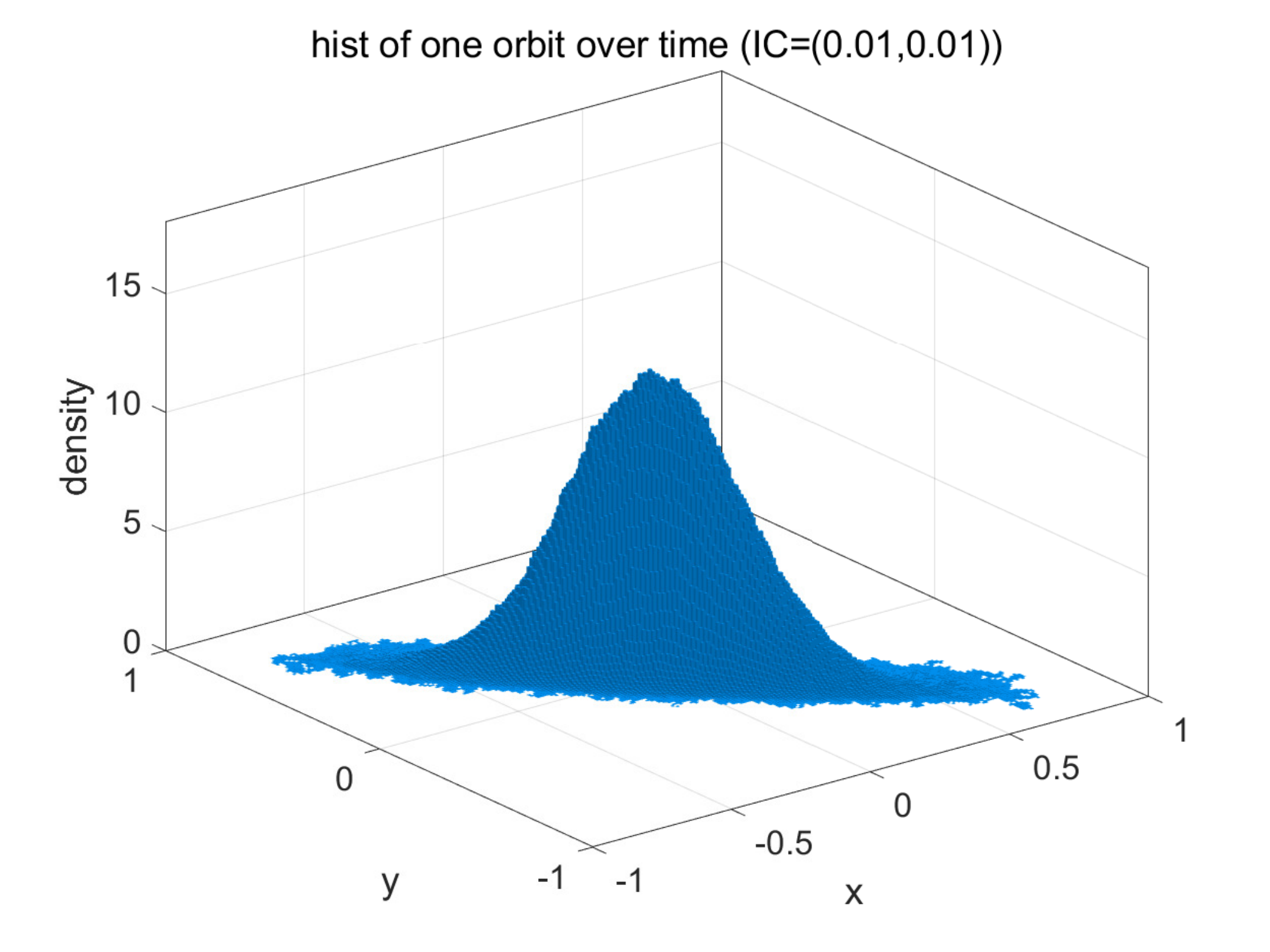}}
	\hfill
	\subfigure[x value of a trajectory]{\includegraphics[width=0.3\linewidth]{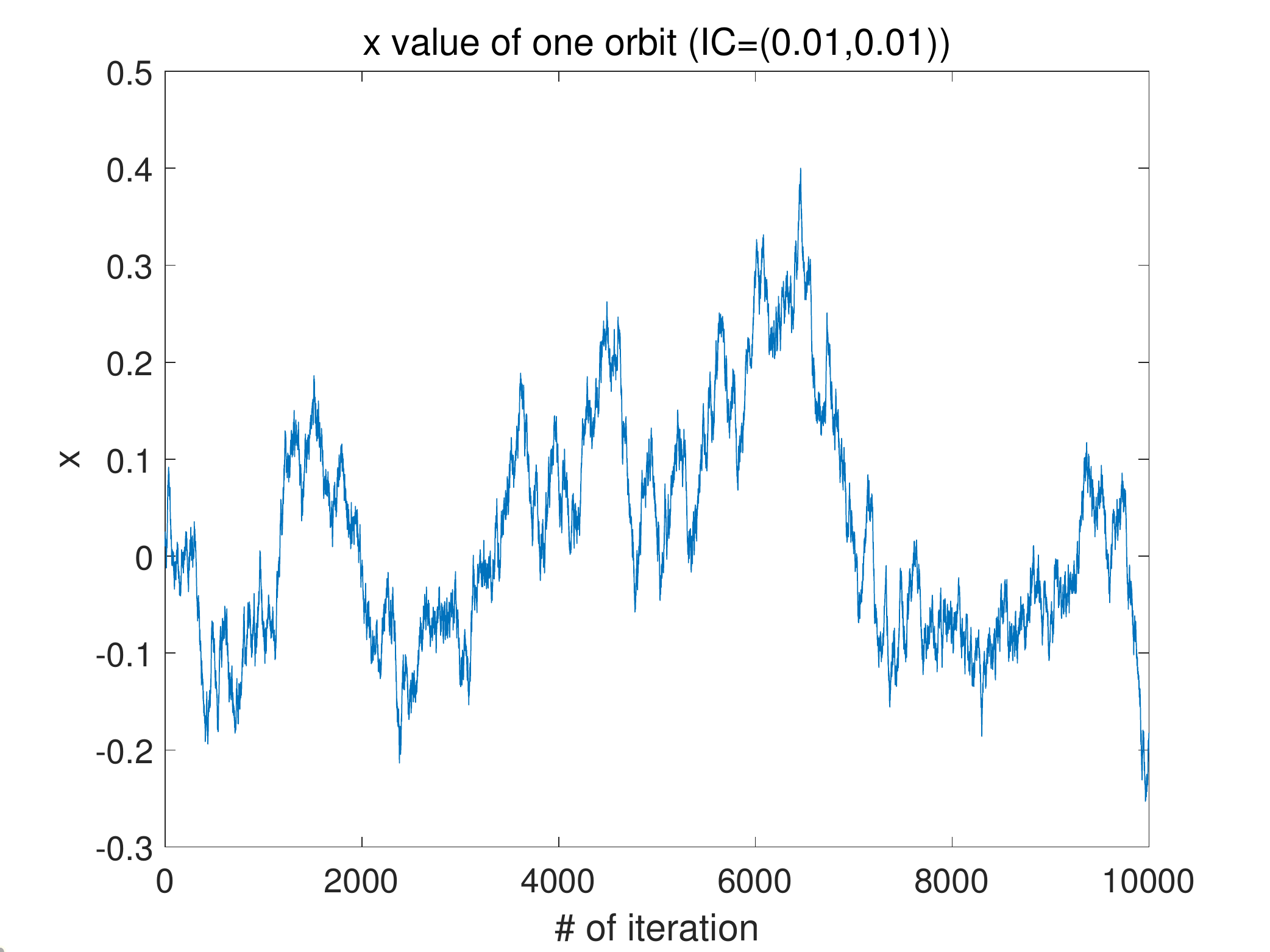}}
	\hfill
	\subfigure[y value of a trajectory]{\includegraphics[width=0.3\linewidth]{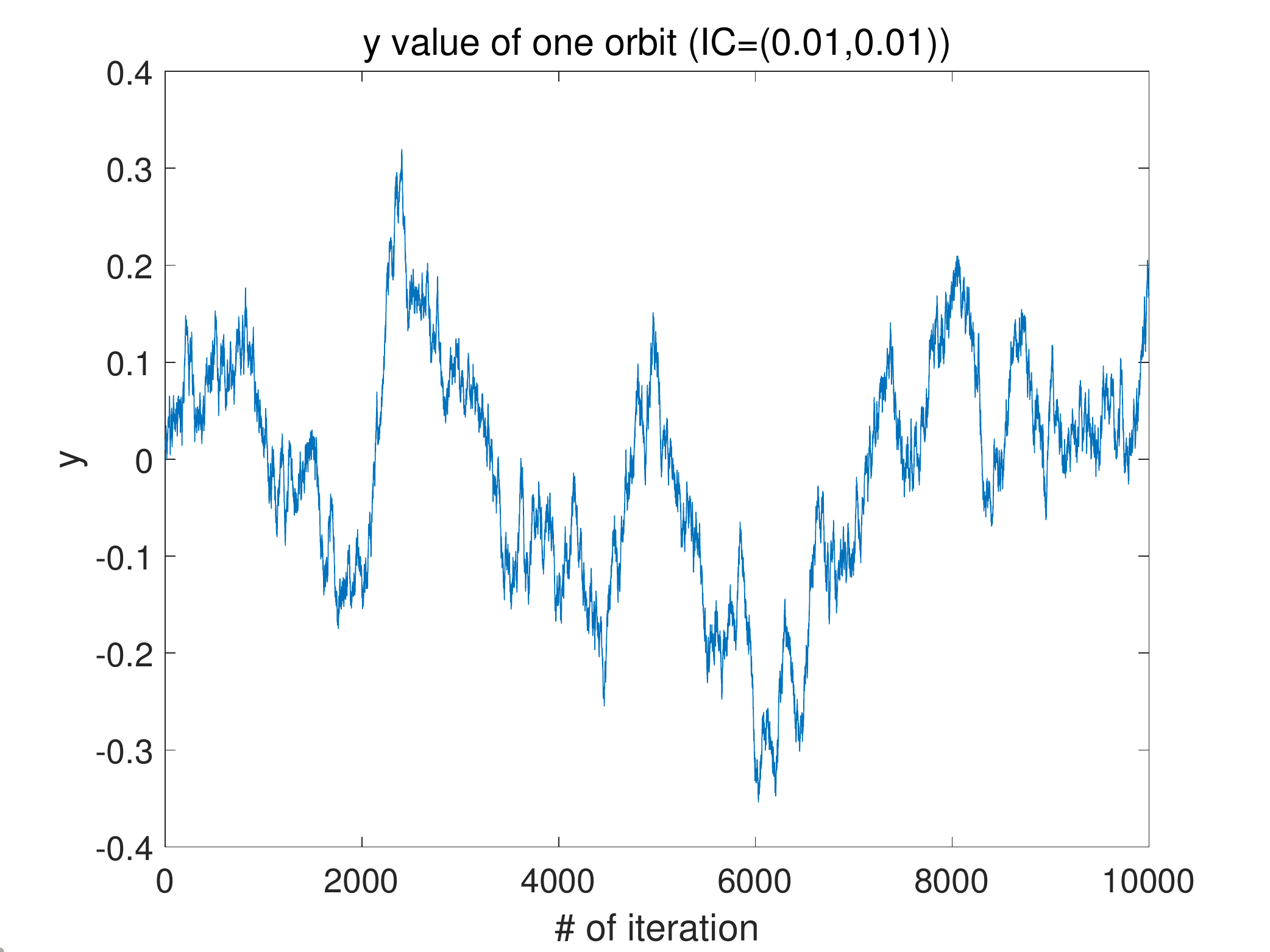}}
	\caption{The histogram of a single trajectory. We can see that it is the same as the experimental result for the invariant distribution in Fig.\ref{fig_Matyas_StochVsChaotoc_b}.}
	\label{fig_Matyas1orbit}
\end{figure}

\subsection{Lyapunov exponent}


Thm.\ref{thm_LyapExp} provides a quantitative estimate of the Lyapunov exponent of the deterministic GD map $\varphi$. Although we required an additional strong convexity condition on $f_0$ for the geometric ergodicity of the stochastic map $\hat{\varphi}$, this result about the deterministic map does not have this requirement.

\subsubsection{On 1-dim periodic $f_{1,\epsilon}$}
As an illustration, we pick multimodal nonconvex $f_0=(x^2-1)^2$, together with $f_{1,\epsilon}(x)=\epsilon\sin\left(\frac{x}{\epsilon}\right)$. Fig.'s \ref{fig_lyapExpVsEta} and \ref{fig_lyapExpVsEpsilon} respectively plot how the numerically computed Lyapunov exponent (computed by eq.\ref{LyapExpDef} with a random initial point) depends on $\eta$ (with fixed $\epsilon$) and on $\epsilon$ (with fixed $\eta$). The constant $m\approx \lambda(x)-\ln(\eta/\epsilon)$ is around 0.7 in both plots, which agrees with our theoretical estimate of $m=\frac{1}{2\pi}\int_0^{2\pi}\ln|\sin(y)|\,dy \approx-0.6931$.

\begin{figure}[H]
	\centering
	\hfill
	\subfigure[$\lambda(x)$ against $\eta$]{\includegraphics[width=0.45\linewidth]{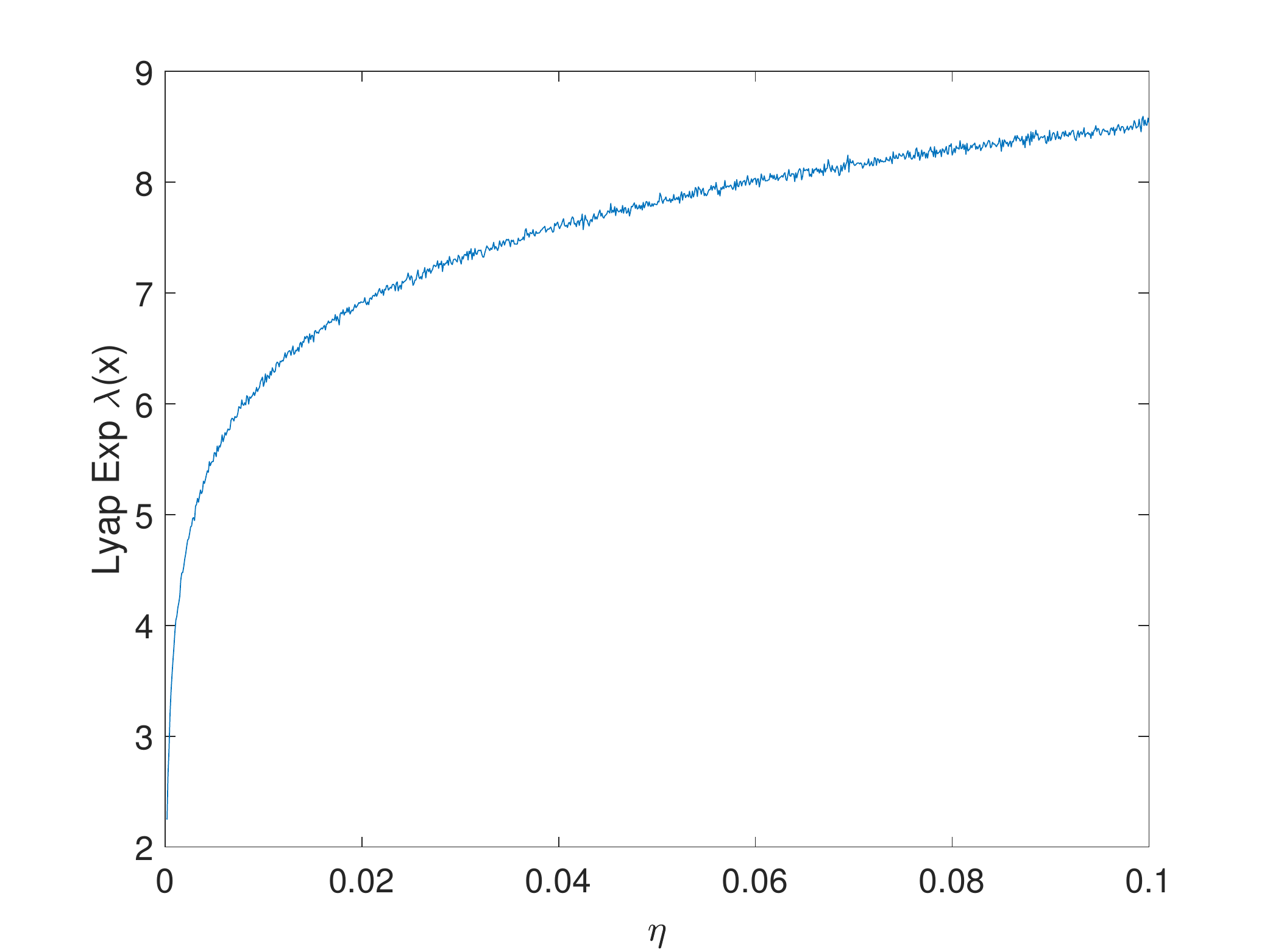}}
	\hfill
	\subfigure[$\lambda(x)-\ln(\eta/\epsilon)$ against $\eta$]{\includegraphics[width=0.45\linewidth]{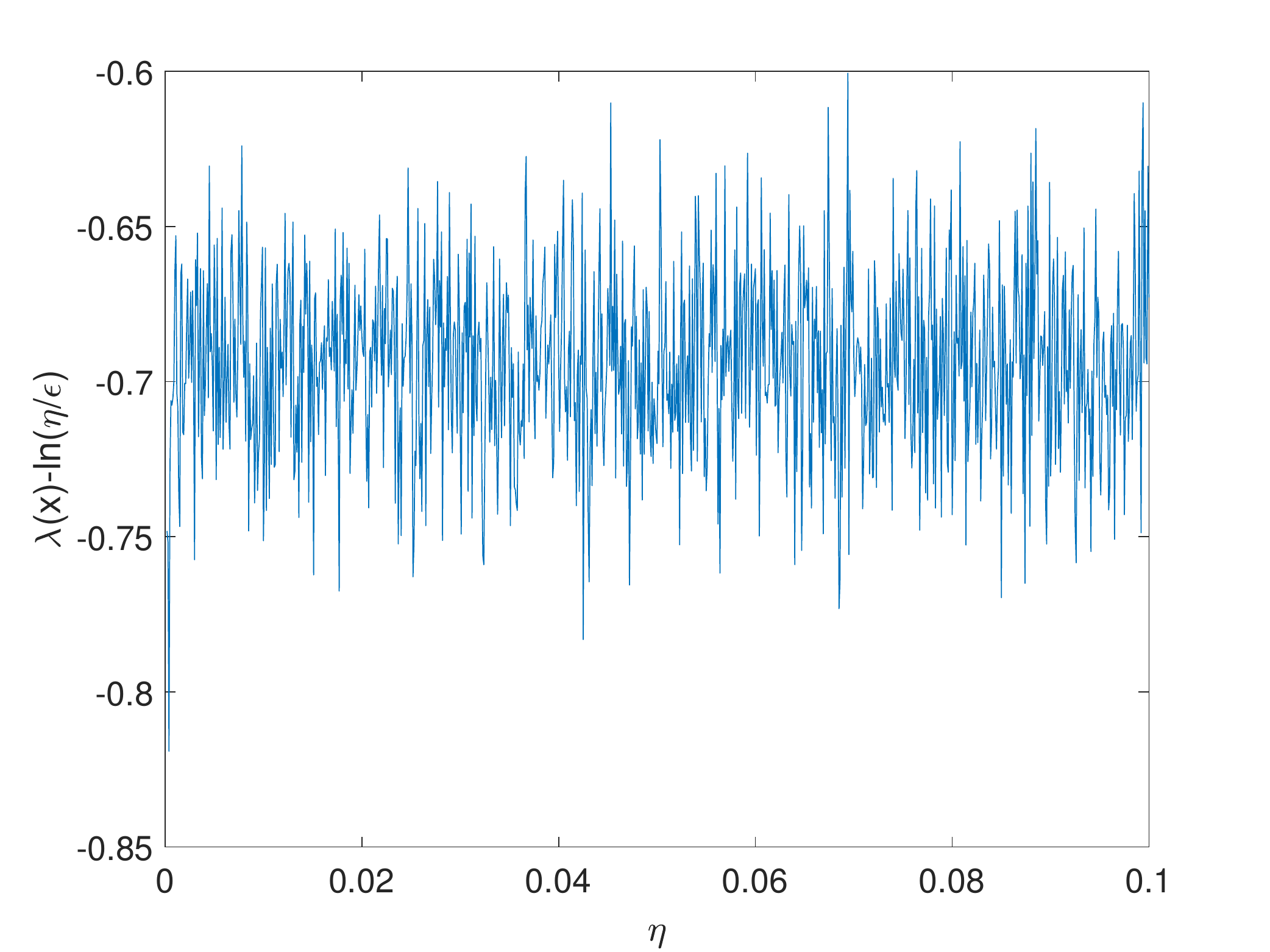}}
	\hfill
	\caption{Dependence of the Lyapunov exponent on $\eta$}
	\label{fig_lyapExpVsEta}
\end{figure}
\begin{figure}[H]
	\centering
	\hfill
	\subfigure[$\lambda(x)$ against $\epsilon$]{\includegraphics[width=0.45\linewidth]{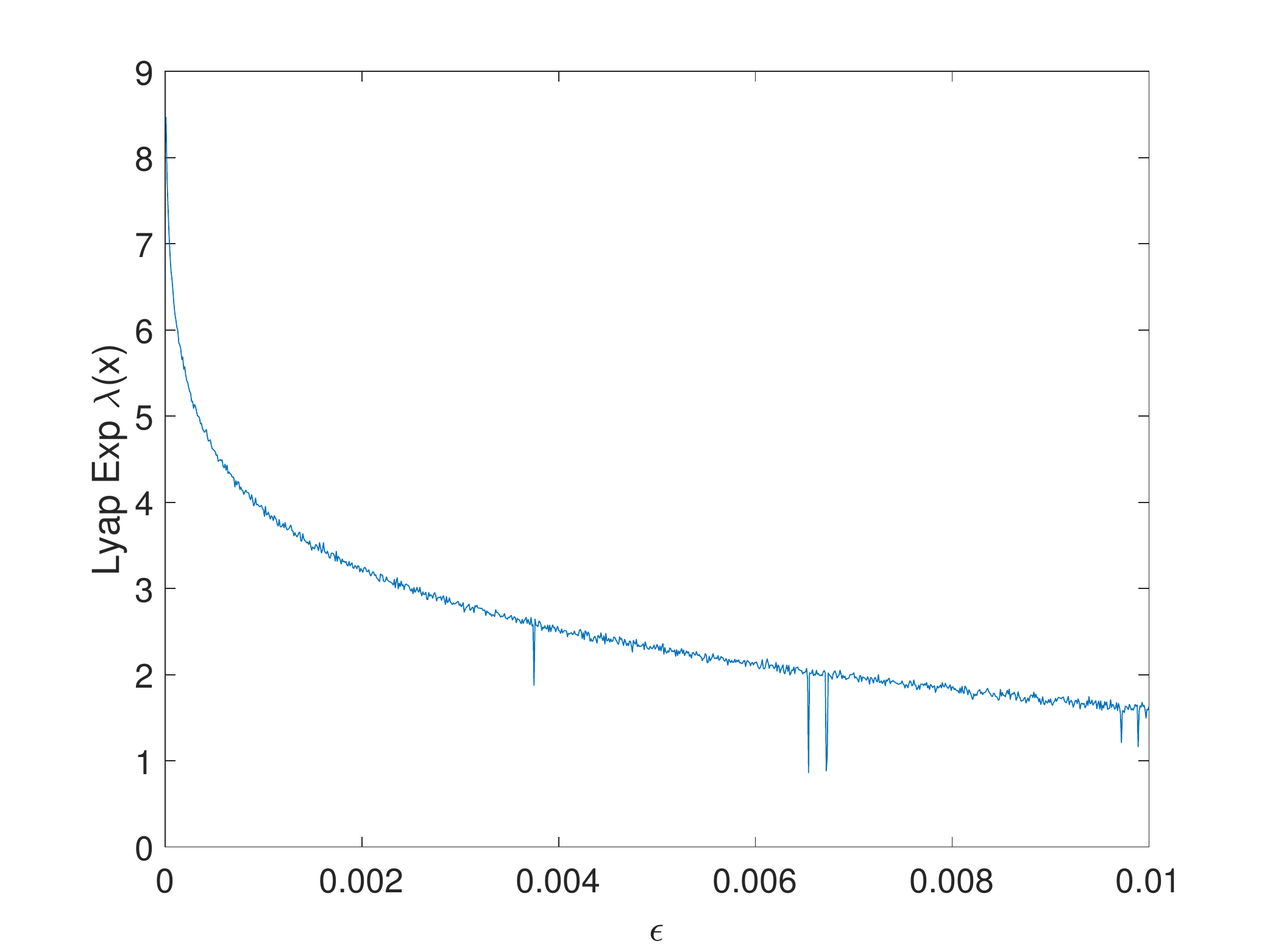}}
	\hfill
	\subfigure[$\lambda(x)-\ln(\eta/\epsilon)$ against $\epsilon$]{\includegraphics[width=0.45\linewidth]{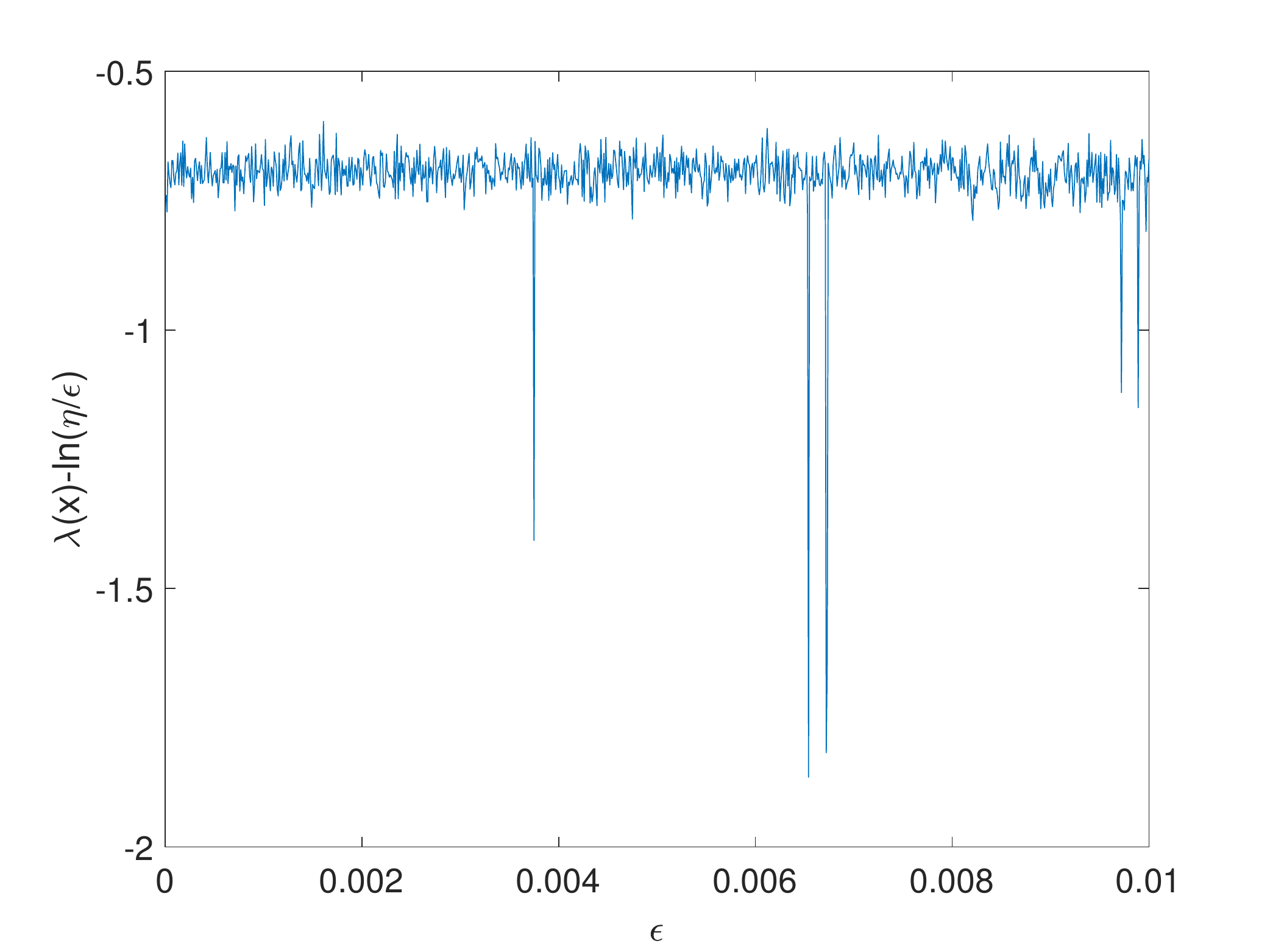}}
	\hfill
	\caption{Dependence of the Lyapunov exponent on $\epsilon$}
	\label{fig_lyapExpVsEpsilon}
\end{figure}

\subsubsection{On 1-dim non-periodic $f_{1,\epsilon}$}
The following experiment shows that Thm. \ref{thm_LyapExp} works for non-periodic $f_{1,\epsilon}$. Fig. \ref{fig_lyapExp_quasiperiodic} is the test on the quasiperiodic $f_{1,\epsilon}$ given in Fig.  \ref{fig_nonperiodic_generalization} and Example \ref{example_aperiodic_new}. The theoritical value for $m$ in Cond. 2 is $\lim_{n\rightarrow\infty}\int_0^n\ln|\sin(x)+2\sin(\sqrt{2}x)|\,dx\approx-0.0117$, is the same as the experiment shows. 
\begin{figure}[H]
	\centering
	\hfill
	\subfigure[$\lambda(x)-\ln(\eta/\epsilon)$ against $\eta$]{\includegraphics[width=0.45\linewidth]{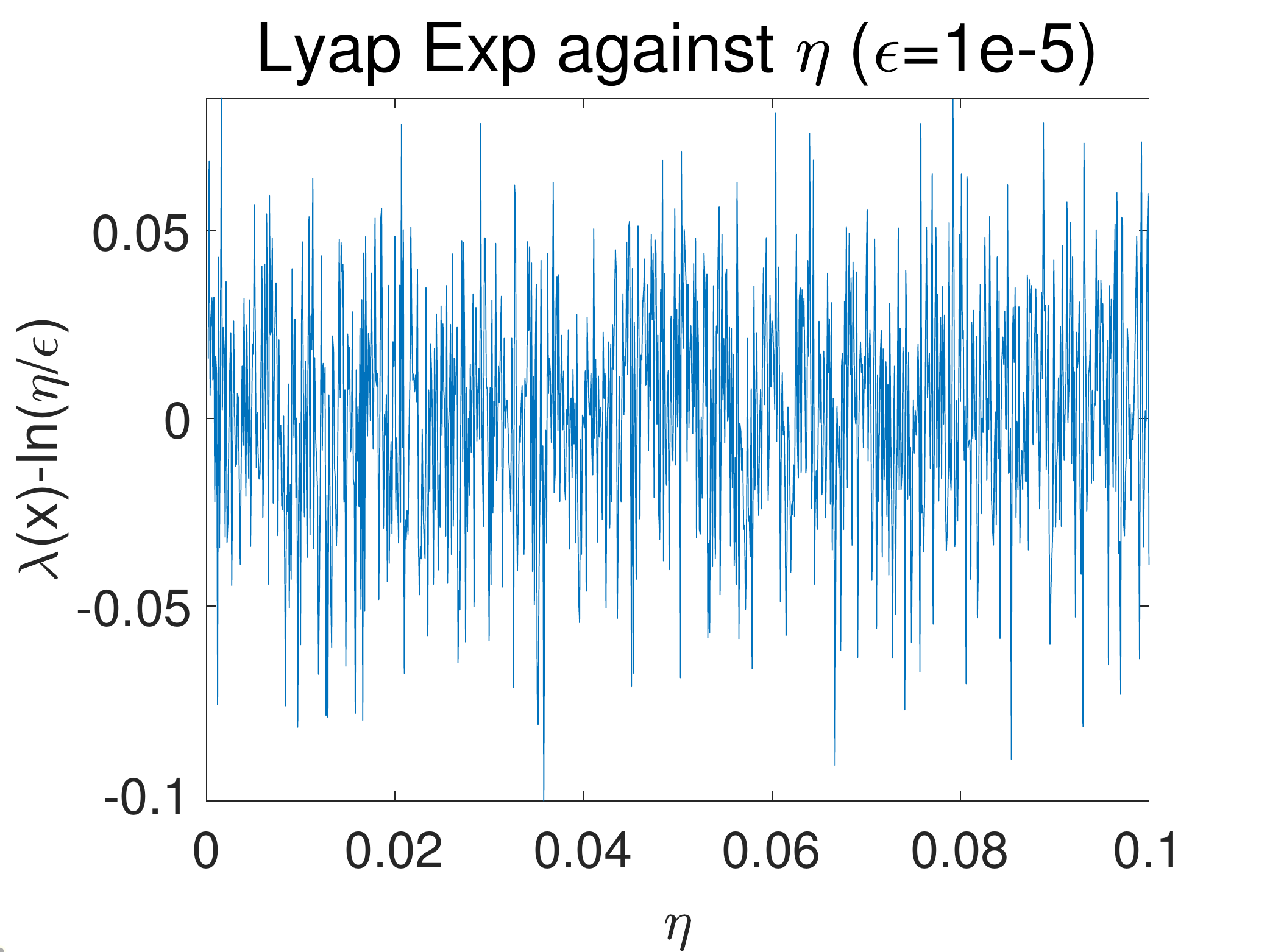}}
	\hfill
	\subfigure[$\lambda(x)-\ln(\eta/\epsilon)$ against $\epsilon$]{\includegraphics[width=0.45\linewidth]{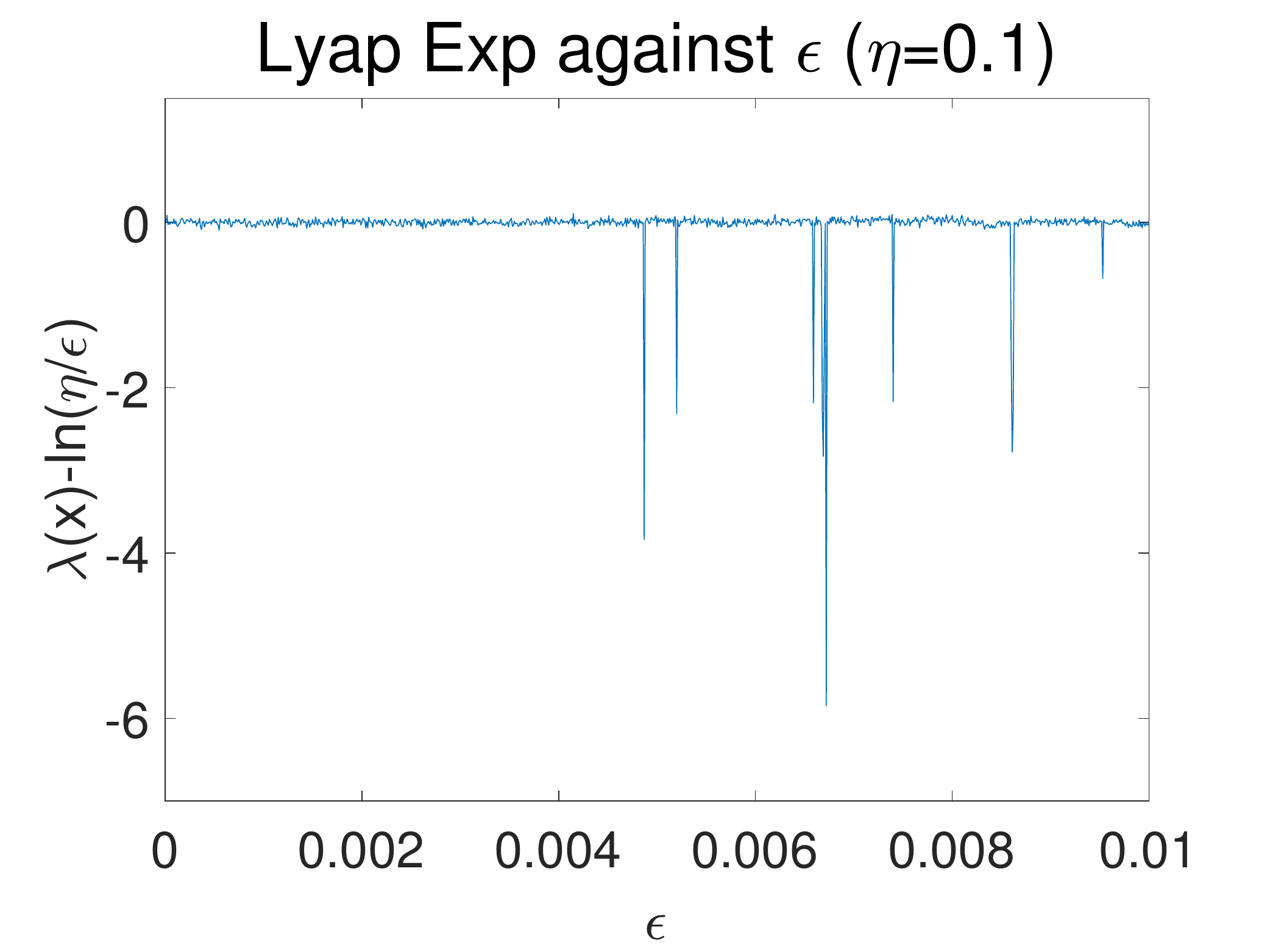}}
	\hfill
	\caption{Dependence of the Lyapunov exponent on $\epsilon$ and $\eta$ for non-periodic $f_{1,\epsilon} ($m=-0.0117$)$.}
	\label{fig_lyapExp_quasiperiodic}
\end{figure}

\subsubsection{On the multi-dim case}
Then we also test the theorem in a multi-dim case, whose $f_0$ is Matyas function and $f_{1,\epsilon}$ is periodic function, same as we did in Sec. \ref{sec_Matyas}. We chose a random initial point, run sufficiently many iterations, and use eq.\ref{LyapExpDef} to compute it. At the same time, Thm.\ref{thm_LyapExp} gives a theoretical estimation, with $m=\frac{1}{4\pi^2}\int_{[0,2\pi]^2}\ln\max(|\sin(x)|,|\cos(y)|)\,dx\,dy \approx -0.2669$. Fig.'s \ref{fig_LyapExp_Matyas_eta} and \ref{fig_LyapExp_Matyas_eps} show that this estimation, namely $\lambda(x)\approx m+\ln\left(\frac{\eta}{\epsilon}\right)$, is correct up to $\mathcal{O}(\epsilon+\eta)$ error.
\begin{figure}[H]
	\centering
	\hfill
	\subfigure[$\lambda(x)$ against $\eta$]{\includegraphics[width=0.45\linewidth]{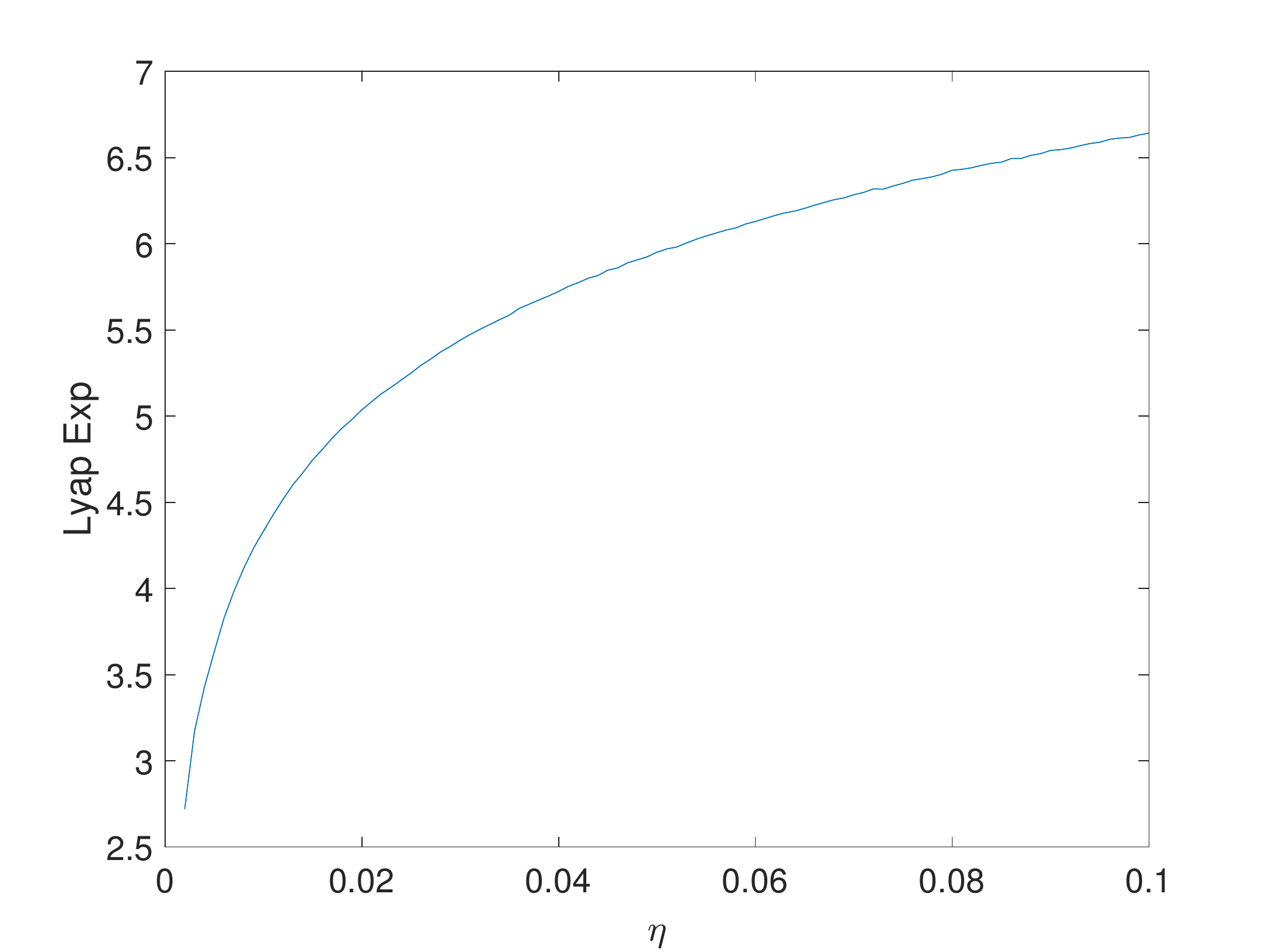}}
	\hfill
	\subfigure[$\lambda(x)-\ln(\eta/\epsilon)$ against $\eta$\label{fig_LyapExp_Matyas_eta_b}]{\includegraphics[width=0.45\linewidth]{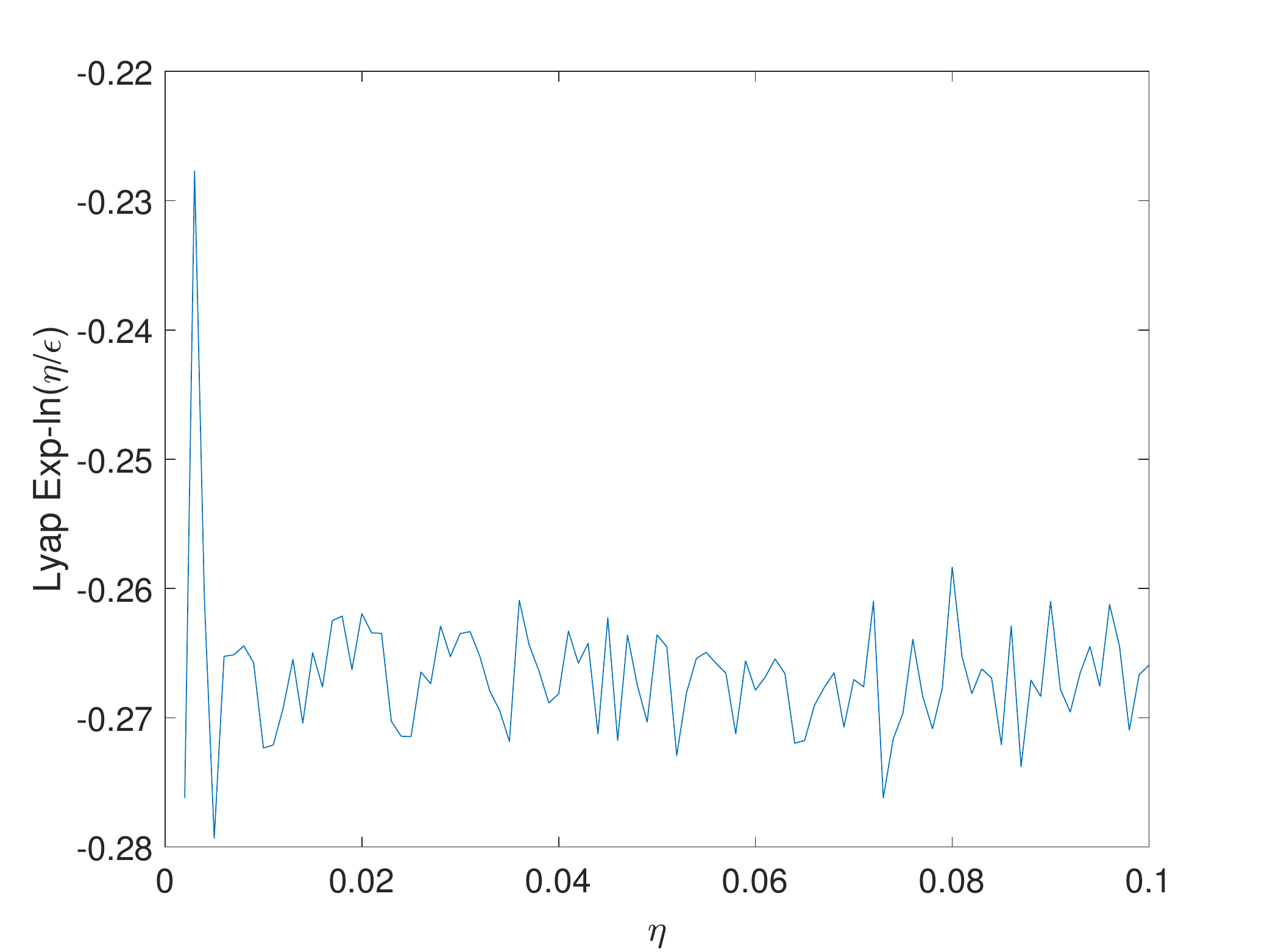}}
	\hfill
	\caption{Dependence of $\lambda(x)$ on $\eta$\label{fig_LyapExp_Matyas_eta} ($\epsilon=0.00001$)}
\end{figure}
\begin{figure}[H]
	\centering
	\hfill
	\subfigure[$\lambda(x)$ against $\epsilon$]{\includegraphics[width=0.45\linewidth]{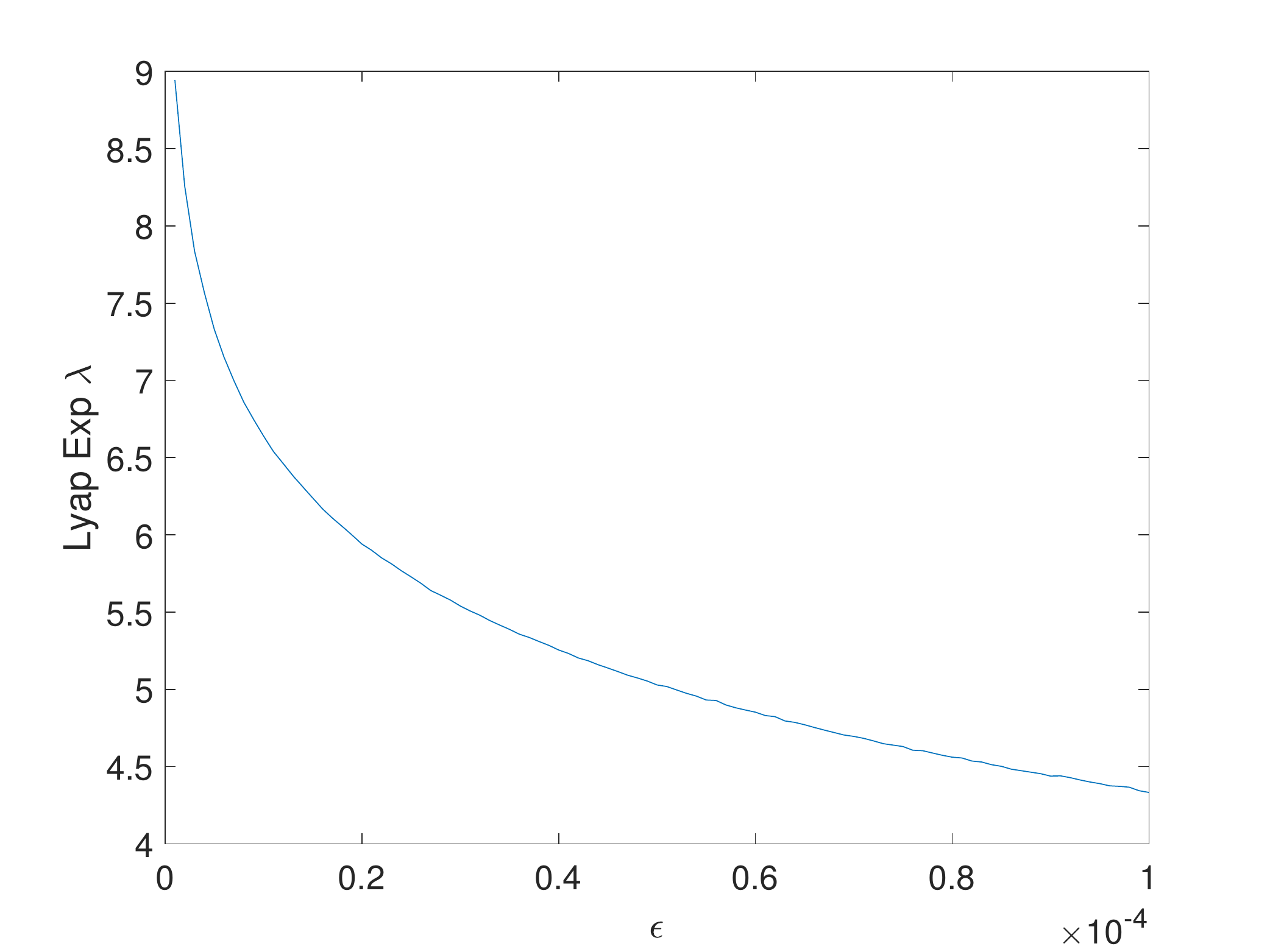}}
	\hfill
	\subfigure[$\lambda(x)-\ln(\eta/\epsilon)$ against $\epsilon$\label{fig_LyapExp_Matyas_eps_b}]{\includegraphics[width=0.45\linewidth]{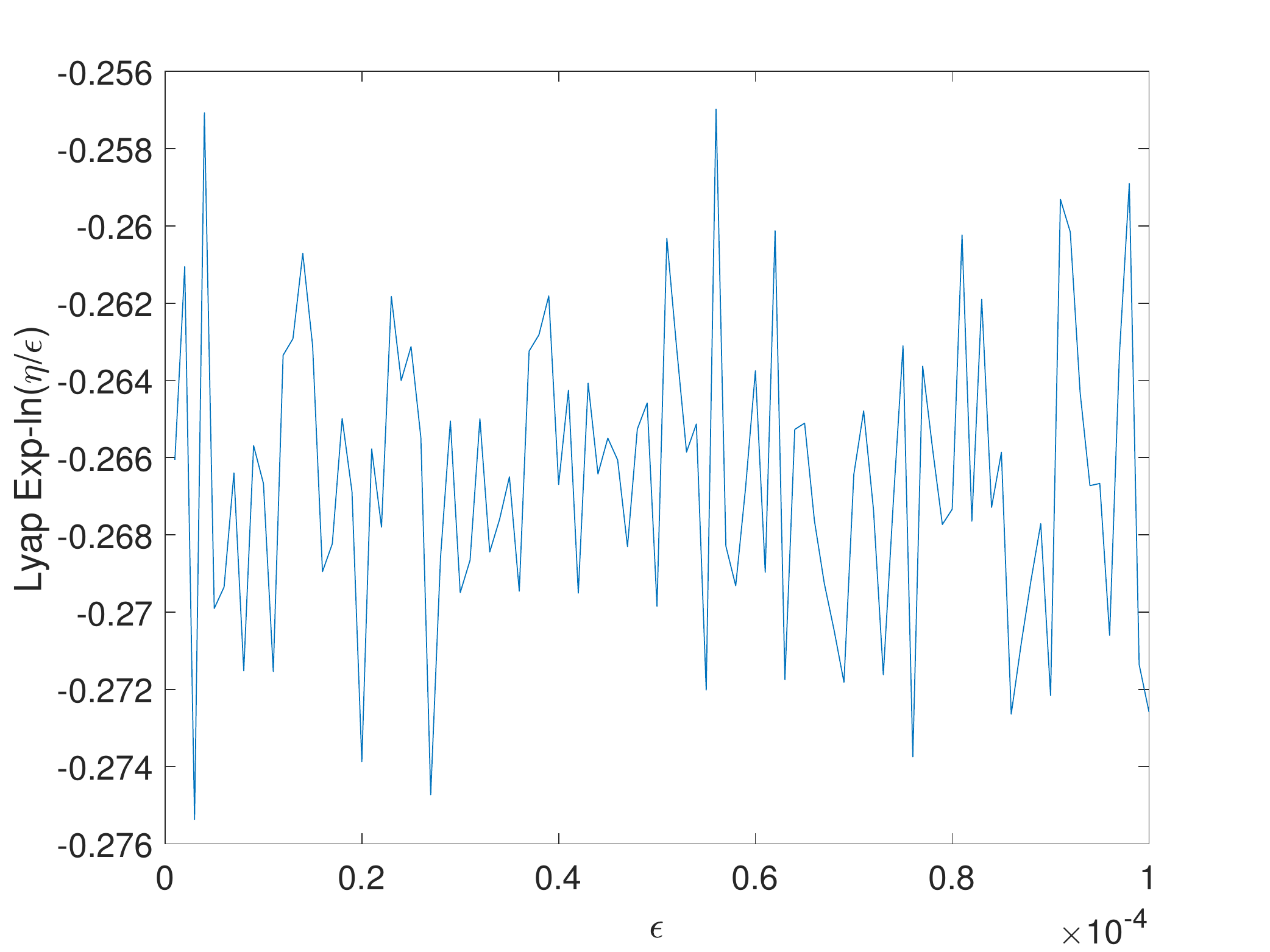}}
	\hfill
	\caption{Dependence of $\lambda(x)$ on $\epsilon$ ($\eta=0.1$)\label{fig_LyapExp_Matyas_eps}}
\end{figure}

\subsection{Stochasticity of deterministic gradient descent with momentum}
\label{sec_withMomentum}
Just for illustrations, consider $f_0=x^2/2$, $f_{1,\epsilon}(x)=\epsilon\sin(x/\epsilon)$, and two common ways for adding momentum:

\subsubsection{Heavy ball}
The iteration is \citep{polyak1964some}
$v_{n+1}=\gamma y_n-\eta\nabla f(x_n), ~
x_{n+1}=x_n+v_{n+1}$, 
with $v_0=0$. See the stochasticity of $x$ in Fig.\ref{fig_heavyBall}.
\begin{figure}[H]
	\centering
	\hfill
	\subfigure[\footnotesize{Evolution of an ensemble}]{\includegraphics[width=0.3\linewidth]{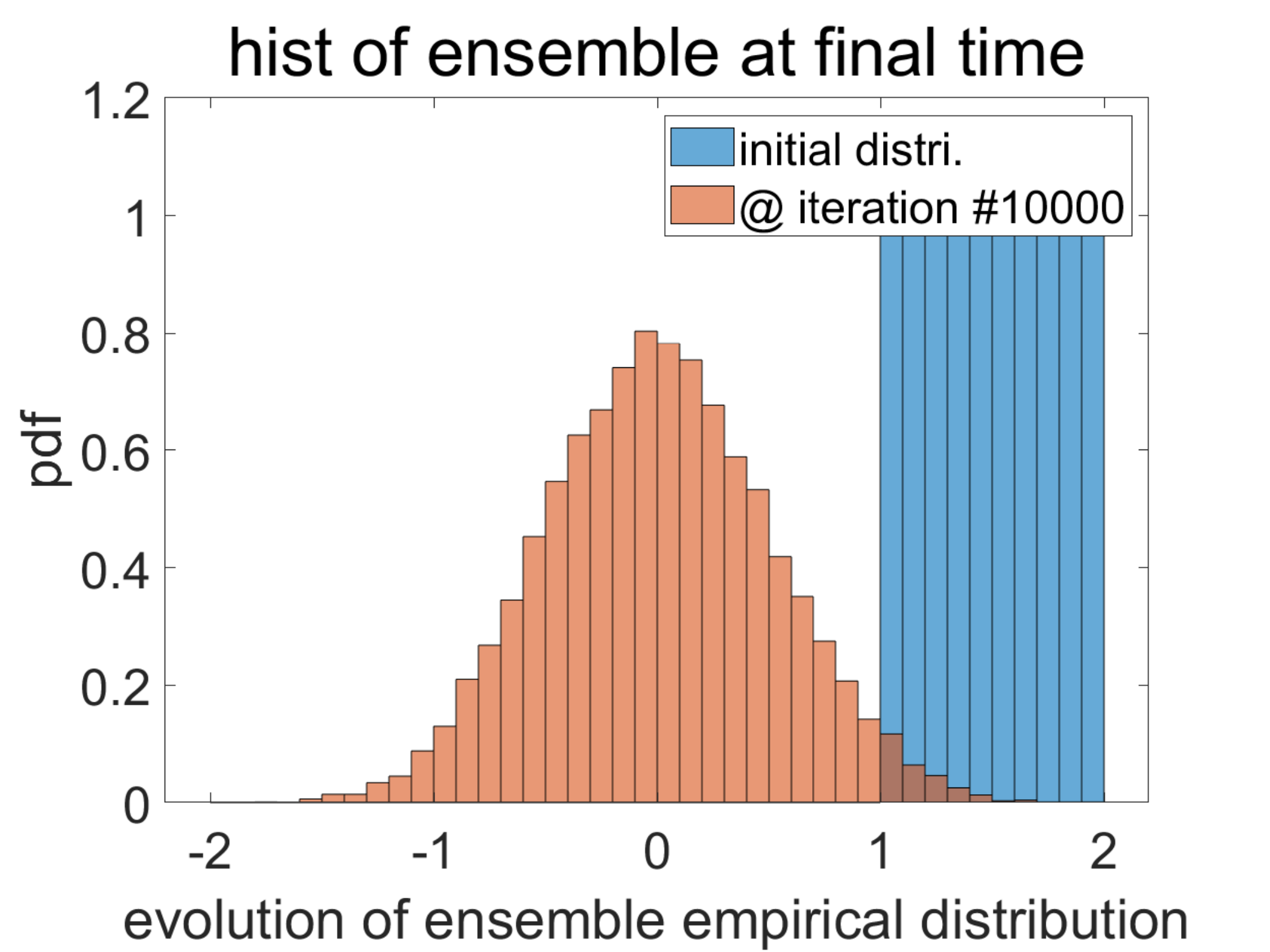}}
	\hfill
	\subfigure[\footnotesize{Empirical distrib. of an orbit}]{\includegraphics[width=0.3\linewidth]{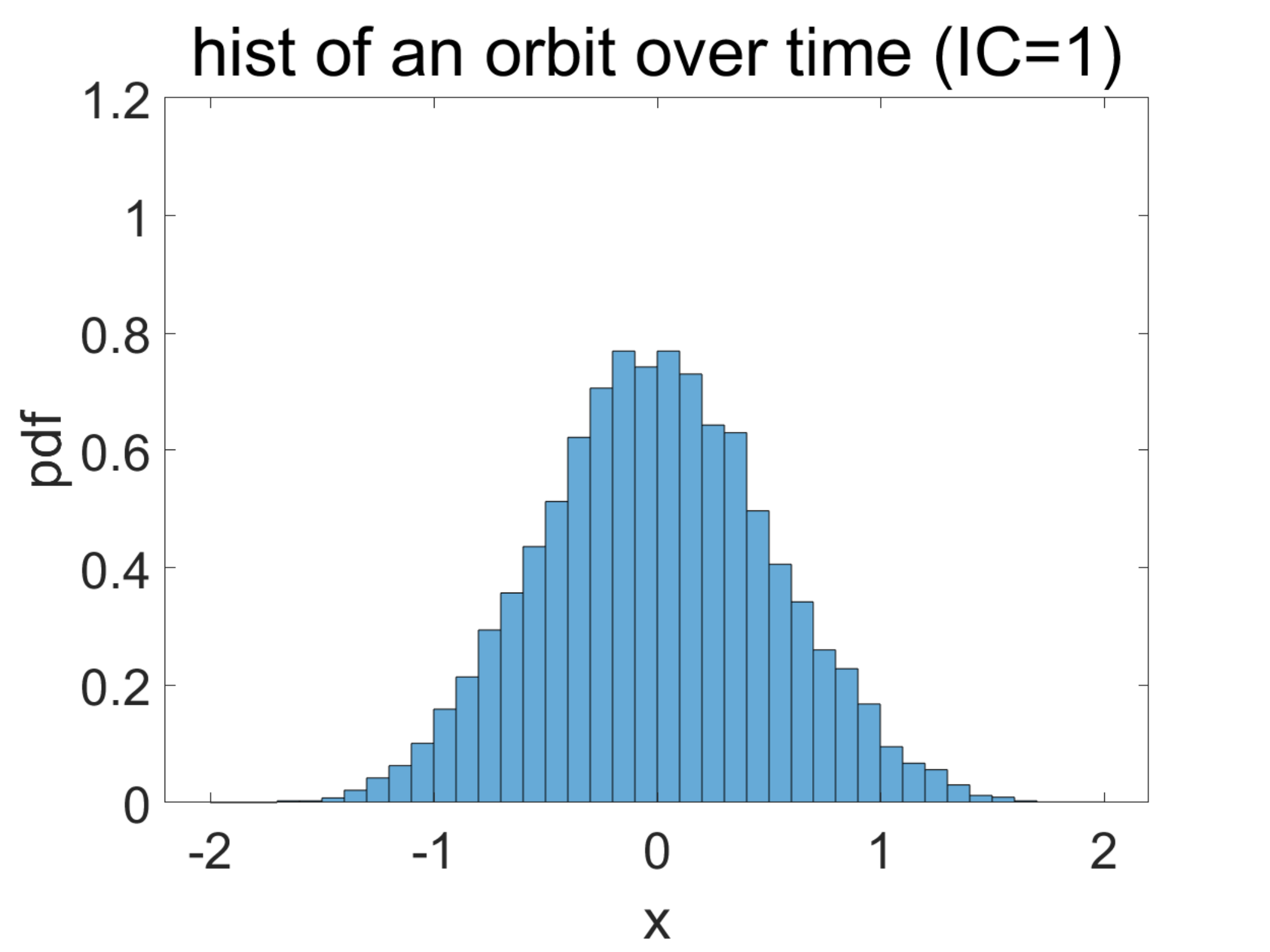}}
	\hfill
	\subfigure[\footnotesize{Iterations in an orbit}]{\includegraphics[width=0.3\linewidth]{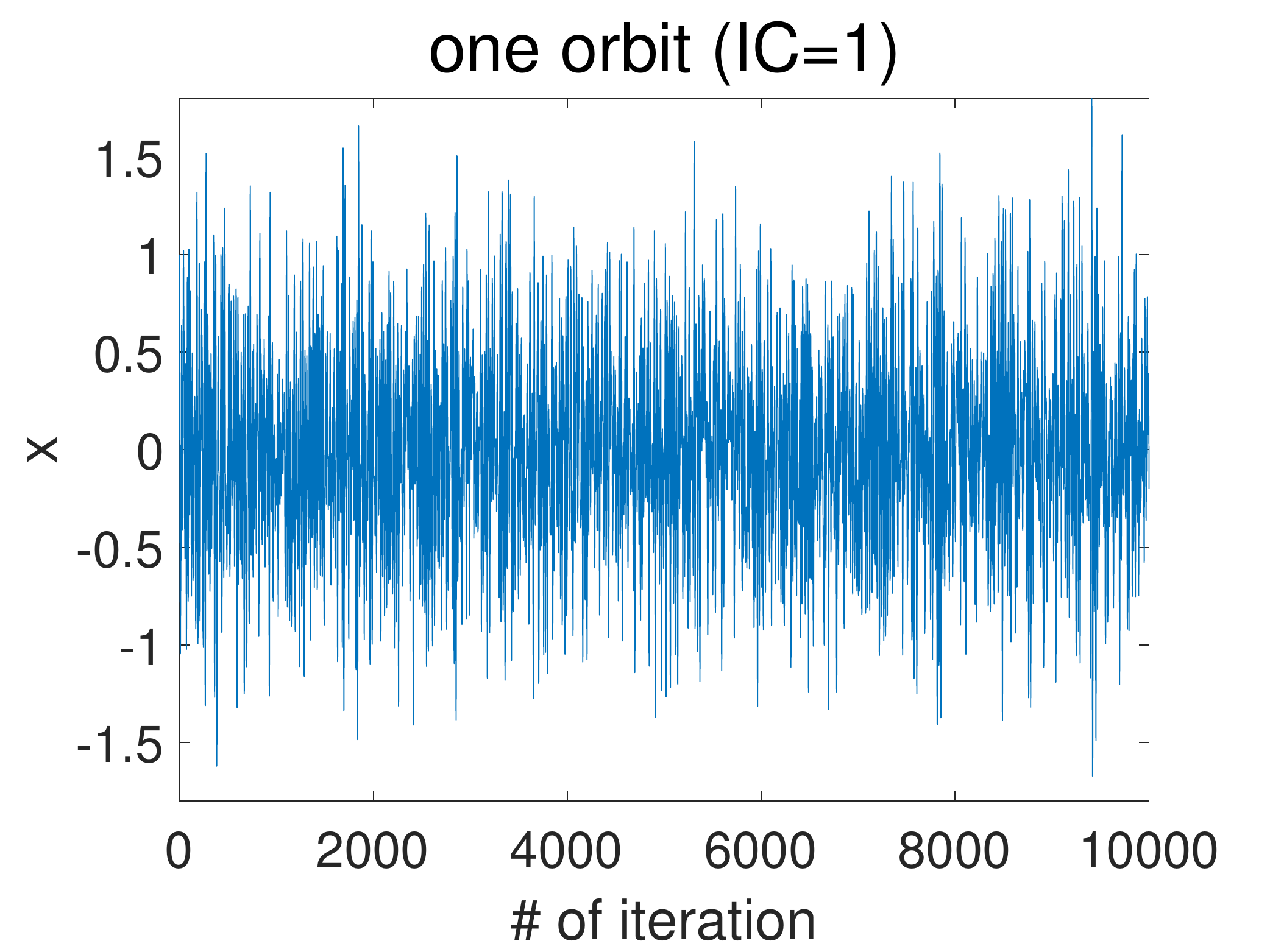}}
	\hfill
	\caption{Heavy ball experiment. $\eta=0.01$, $\epsilon=0.0001$, and $\gamma=0.9$.}
	\label{fig_heavyBall}
\end{figure}

\subsubsection{Nesterov Accelerated Gradient for strongly convex function (NAG-SC)}
The iteration is \citep{nesterov2013introductory}
$
y_{k+1}=x_k-\eta\nabla f(x_k), ~
x_{k+1}=y_{k+1}+c(y_{k+1}-y_k)$,
with $y_0=x_0$. $c=\frac{1-\sqrt{\mu \eta}}{1+\sqrt{\mu \eta}}$ where $\mu$ is supposed to be the strong convexity constant; we chose $\mu$ to be that for $f_0$, in this case $\mu=1$. See the stochasticity of $x$ in Fig. \ref{fig_NAGSC}. The smaller variance is due a different scaling for relating $\eta$ to a timestep in continuous time (see e.g., \cite{su2014differential}).
\begin{figure}[H]
	\centering
	\hfill
	\subfigure[\footnotesize{Evolution of an ensemble}]{\includegraphics[width=0.3\linewidth]{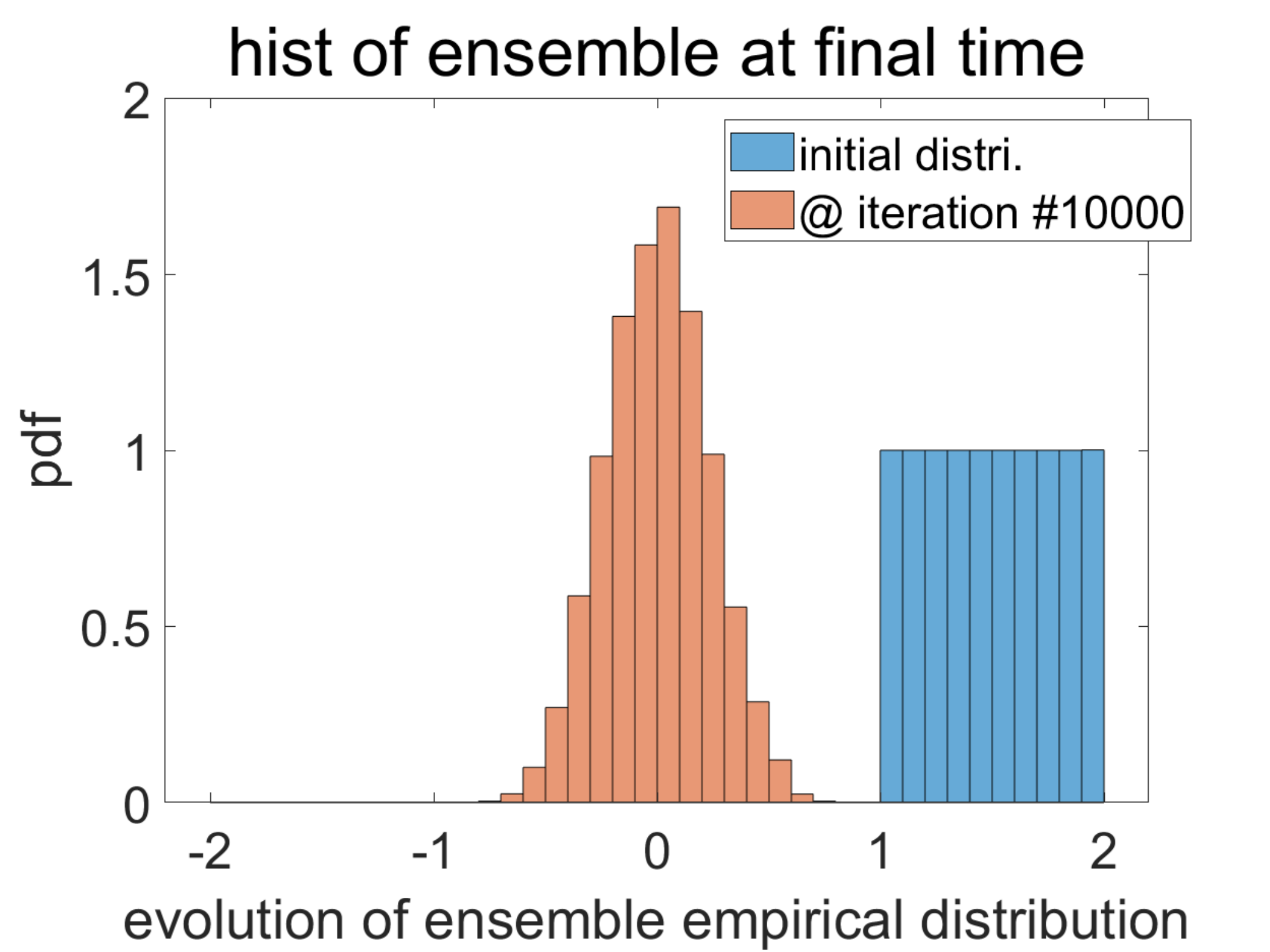}}
	\hfill
	\subfigure[\footnotesize{Empirical distrib. of an orbit}]{\includegraphics[width=0.3\linewidth]{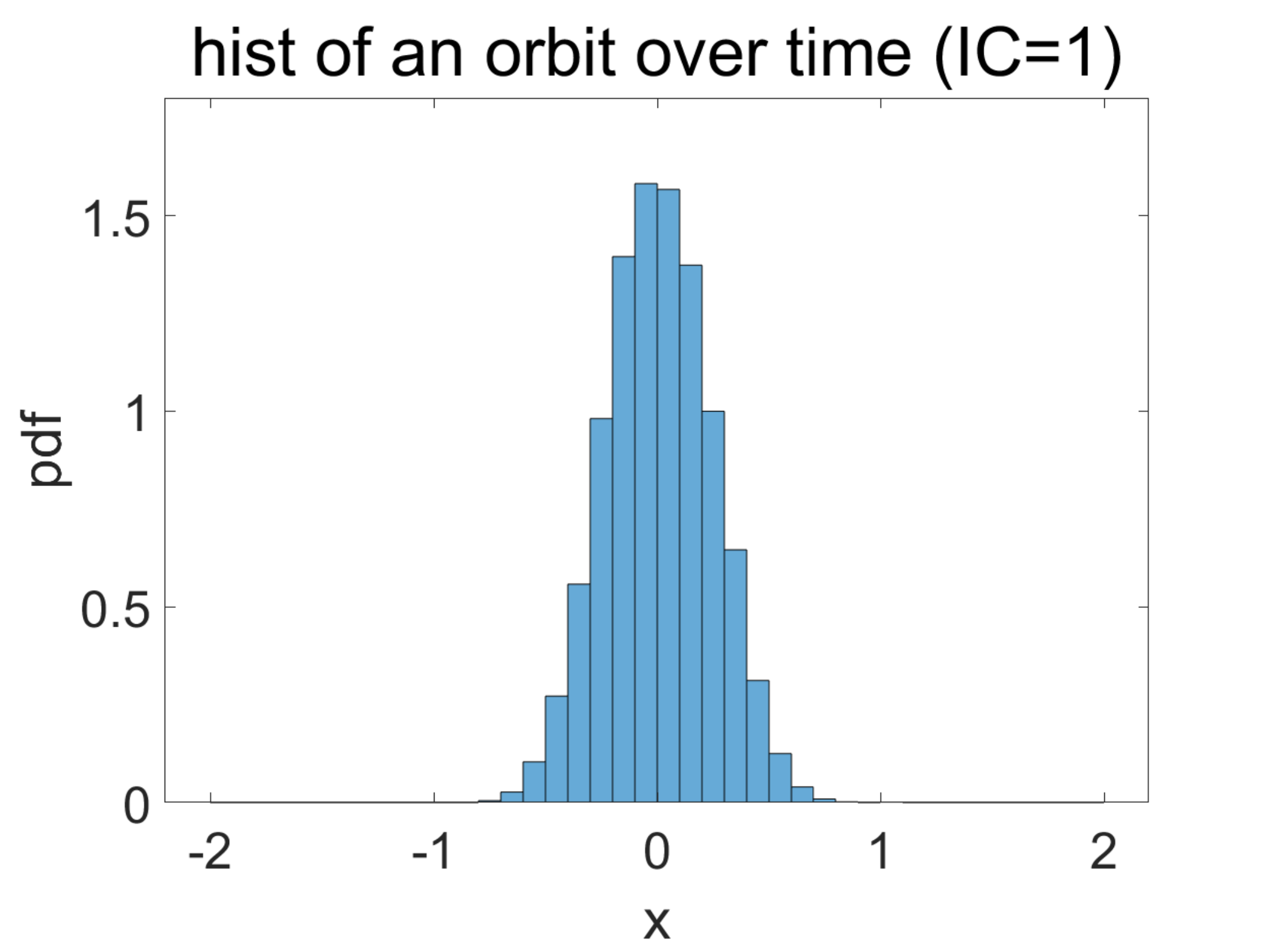}}
	\hfill
	\subfigure[\footnotesize{Iterations in an orbit}]{\includegraphics[width=0.3\linewidth]{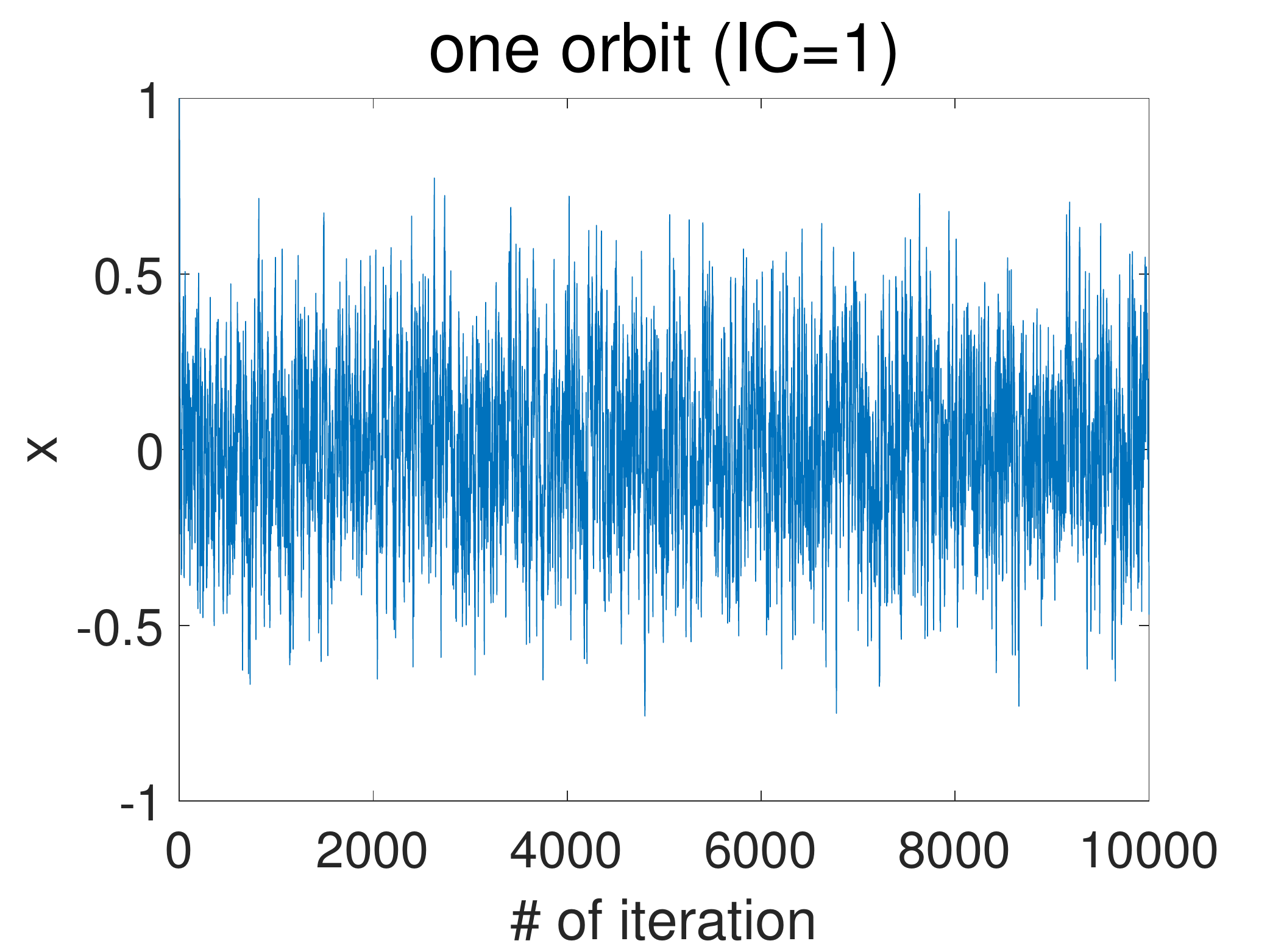}}
	\hfill
	\caption{NAG-SC experiment. $\eta=0.01$, $\epsilon=0.0001$.}
	\label{fig_NAGSC}
\end{figure}

\subsection{The nonconvex $f_0$ dichotomy: to escape or not to escape macroscopic potential well created by $f_0$?}
\label{sec_NT_nonvonvex}

What will happen when $f_0$ is nonconvex but multimodal? Both escapes from $f_0$'s local minima (and the corresponding potential wells) and nonescapes will be possible. Roughly speaking, it depends on how strong $f_{1,\epsilon}$ is when compared with $f_0$. Rmk.\ref{rmk_nonconvex} provided some discussions. To elaborate more, we first make a general remark:
\begin{remark}
	\label{rmk_nonconvex_chaos}
	As theoretically shown, especially in section \ref{sec_LYC}, \ref{sec_PD} and \ref{sec_LyapExp}, we see that chaos can be just a localized small-scale behavior, thus independent of the convexity of $f_0$. However, the limiting distribution of the deterministic map is a global property and it should depend on the global behavior of $f_0$. As explained in Rmk.\ref{rmk_nonconvex}, when $f_0$ is not convex, it can happen that an orbit cannot jump between potential wells, and then unique ergodicity is lost in the sense that multiple ergodic foliations appear and respectively localize to individual potential wells. In this case, the limiting statistics is no longer unique. However, every connected subset of the support of an invariant distribution of the stochastic map can be an ergodic foliation, so if we regard the invariant distributions of the deterministic map and the stochastic map as convex combinations of the invariant distributions in each potential well, the conclusion in Theorem \ref{thm_sameLimitStats} still stands.
\end{remark}
Then we demonstrate two possible outcomes concretely in numerical experiments. We will use the same test function, which is $f_0(x)=k(x^2-1)^2$ and $f_{1,\epsilon}(x)=\epsilon\sin(x/\epsilon)$. $x>0$ and $x<0$ are two potential wells of $f_0$.

We already obtained a bound on the relative strength between $f_0$ and $f_{1,\epsilon}$; it is $k_{critical}=\frac{3\sqrt{3}}{8}$ for whether the point can jump from one potential well to another. Fig.'s \ref{fig_nonconvexStrongF1} and \ref{fig_nonconvexWeakF1} respectively illustrates the long-time statistics of GD when $k=0.05 < k_{critical}$ and $k=5 < k_{critical}$. Results are consistent with theoretical predictions.

\begin{figure}[H]
	\centering
	\subfigure[Invariant distribution]{\includegraphics[width=0.45\linewidth]{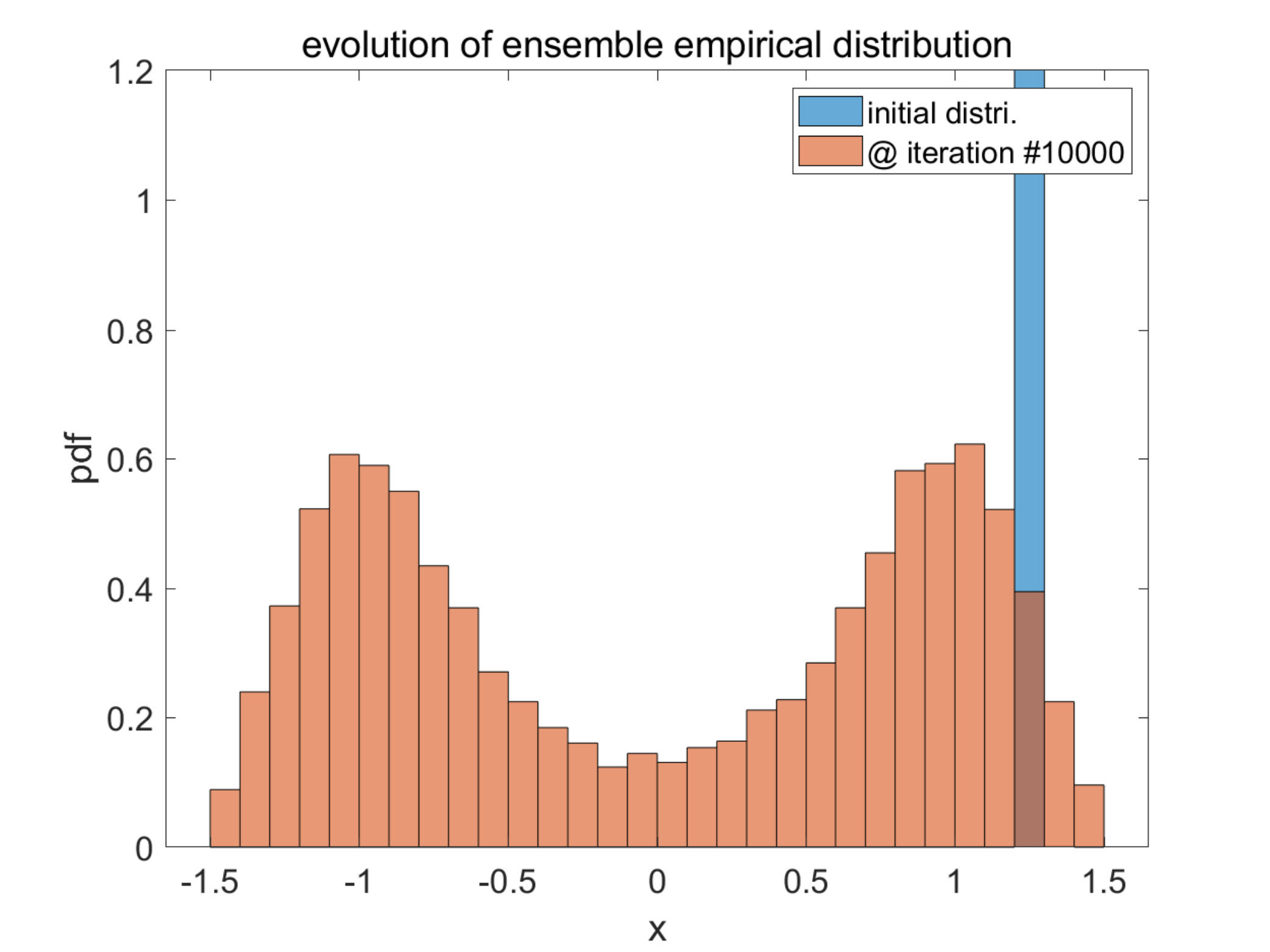}}
	\hfill
	\subfigure[Histogram of a trajectory]{\includegraphics[width=0.45\linewidth]{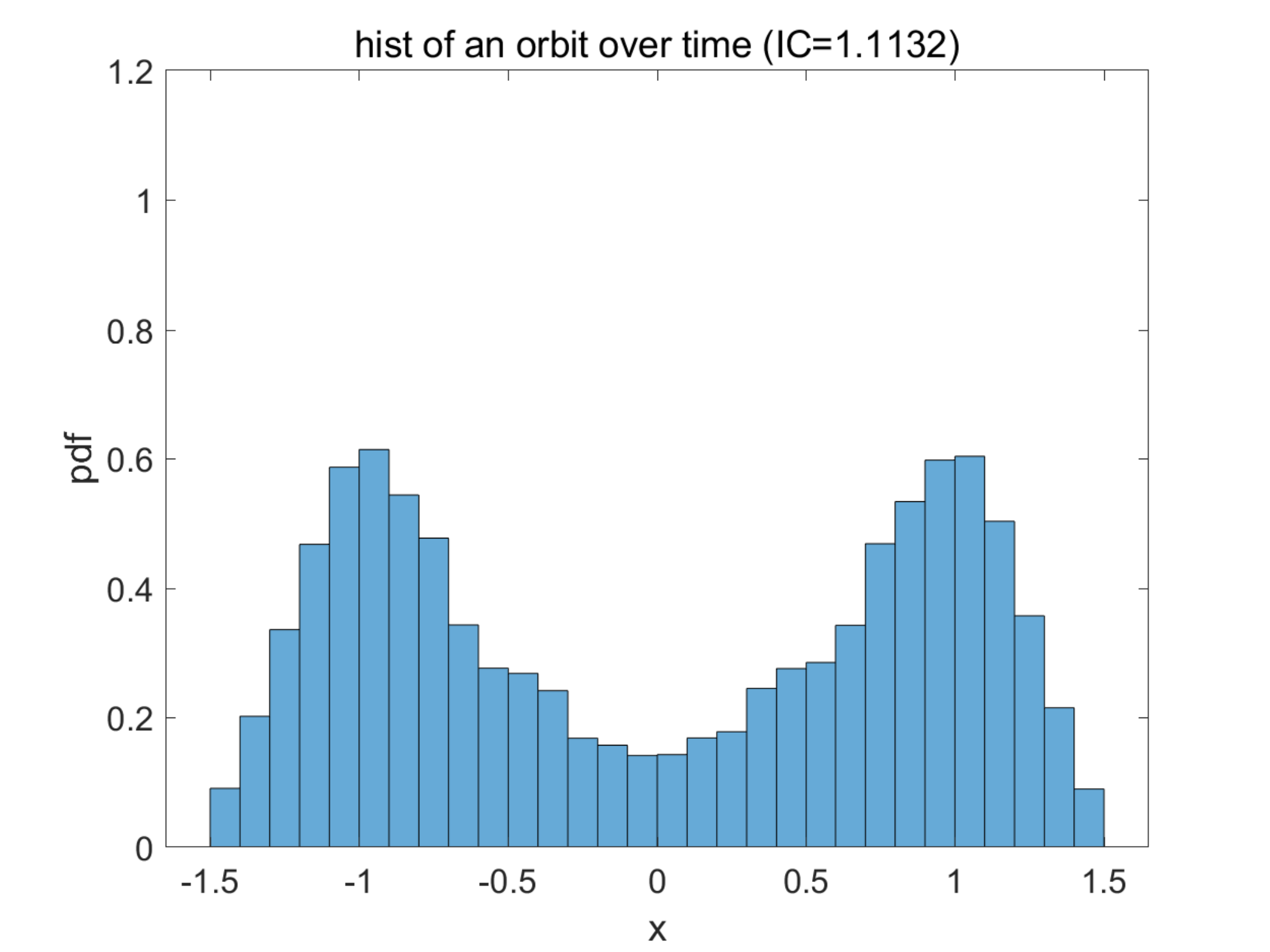}}
	\hfill
	\subfigure[Histogram of another trajectory]{\includegraphics[width=0.45\linewidth]{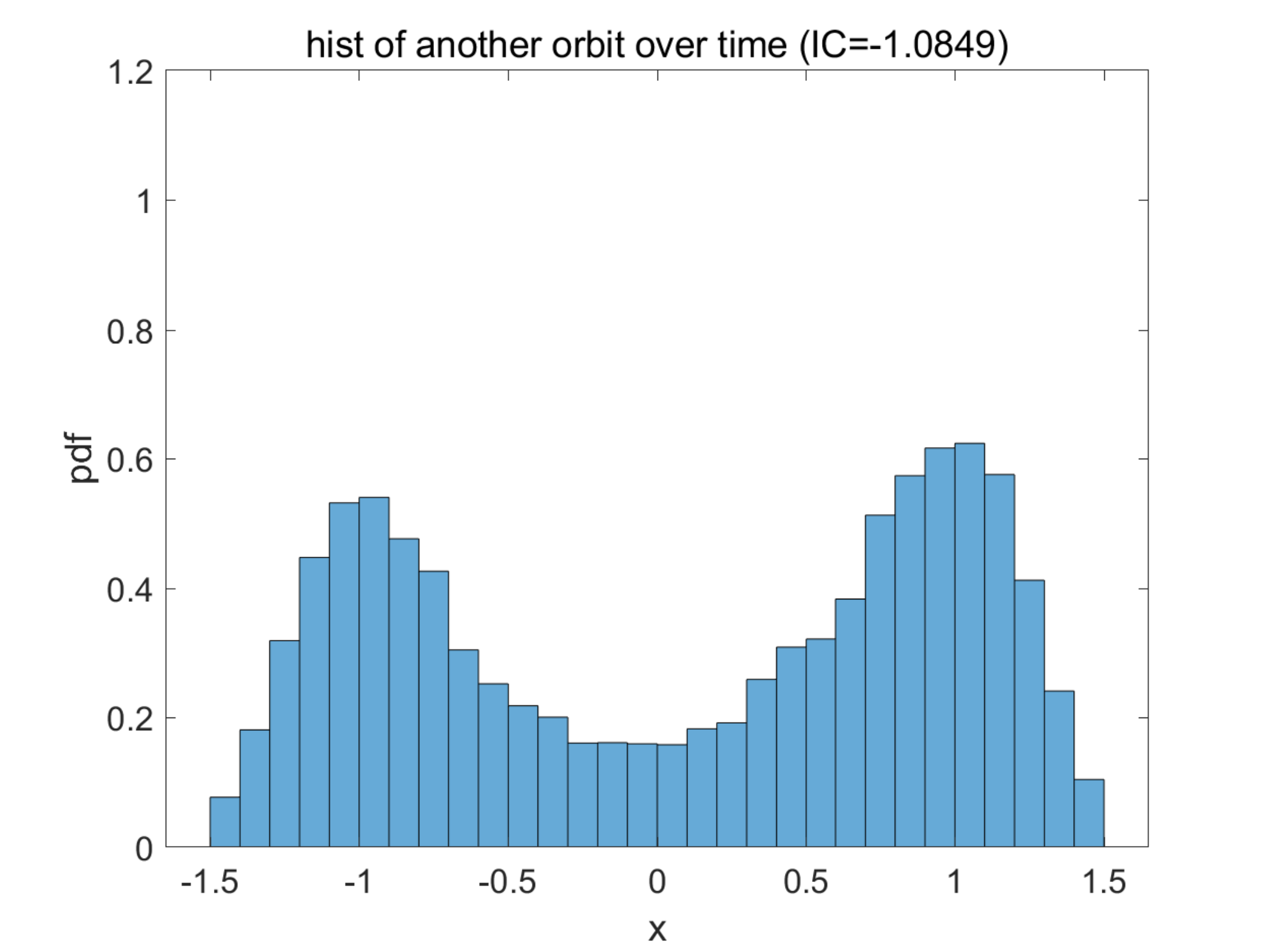}}
	\hfill
	\subfigure[One trajectory]{\includegraphics[width=0.45\linewidth]{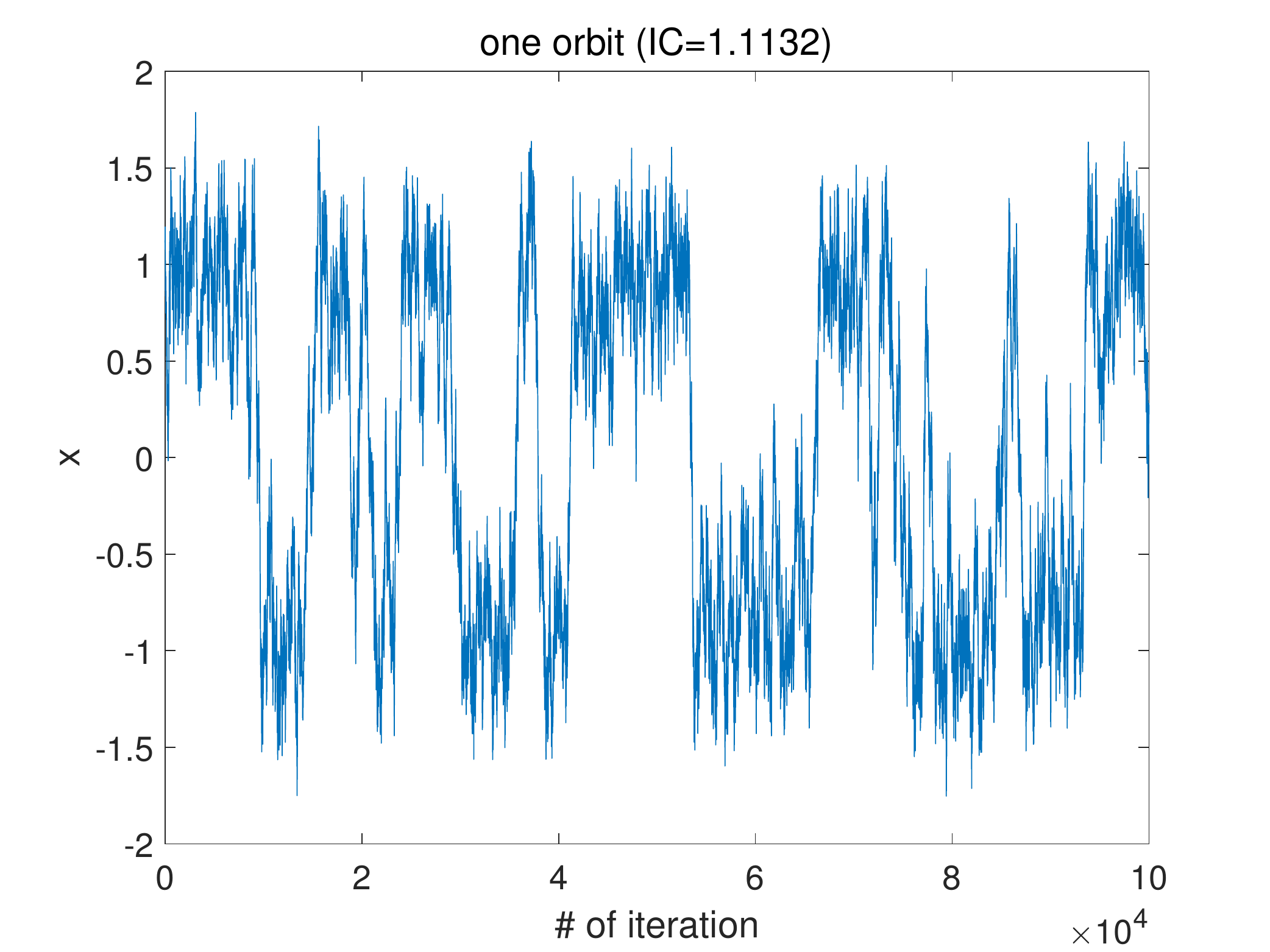}}
	
	\caption{A non-convex mixing example. The initial condition is concentrated in the right potential well but barrier crossing happens. $k=0.02$, $\eta=0.05$ and $\epsilon=0.0001$.}
	\label{fig_nonconvexStrongF1}
\end{figure}
\begin{figure}[H]
	\label{NT_nonmixing}
	\centering
	\subfigure[One of the invariant distributions]{\includegraphics[width=0.45\linewidth]{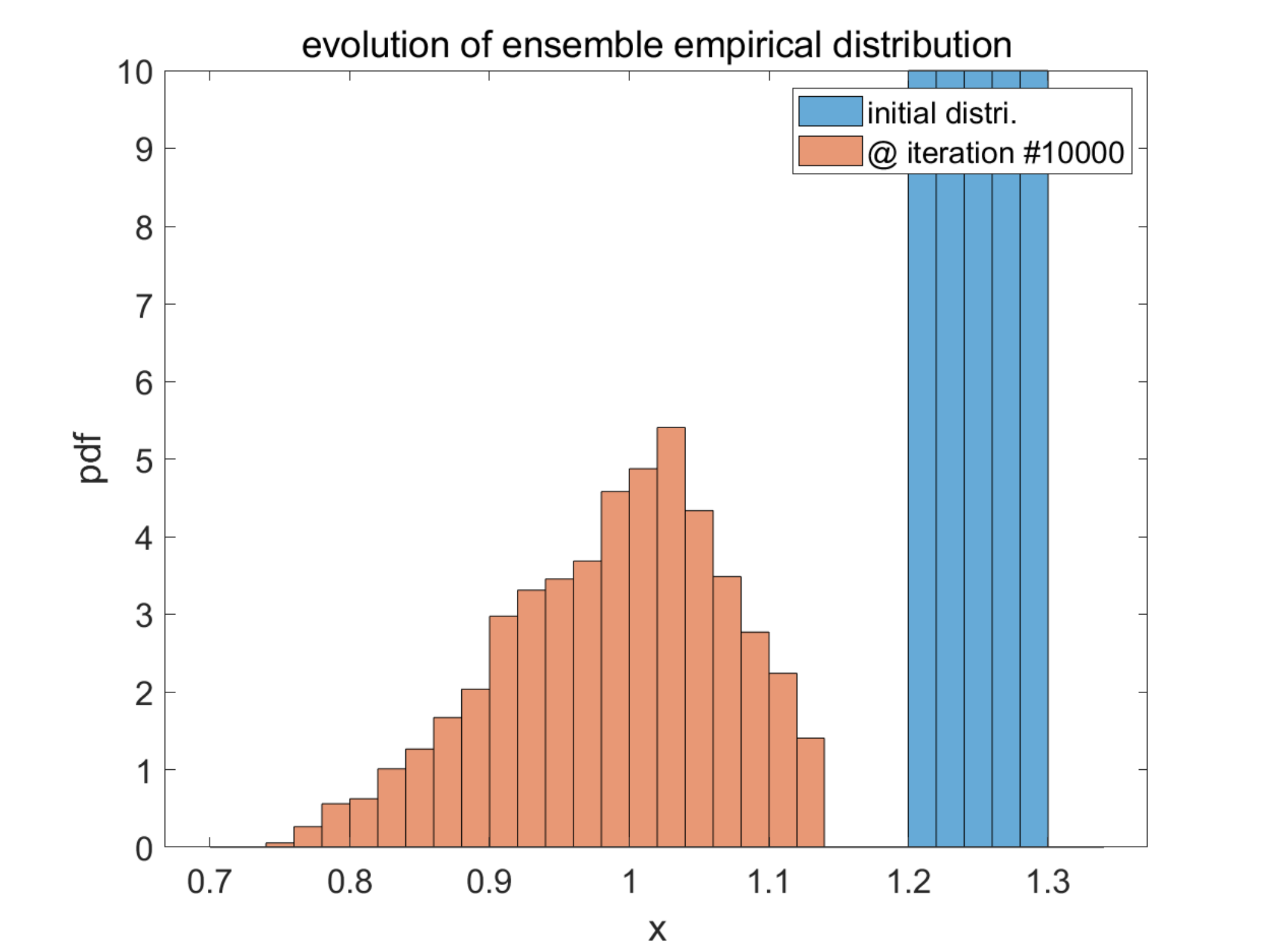}}
	\hfill
	\subfigure[Histogram of a trajectory, starting in the right well]{\includegraphics[width=0.45\linewidth]{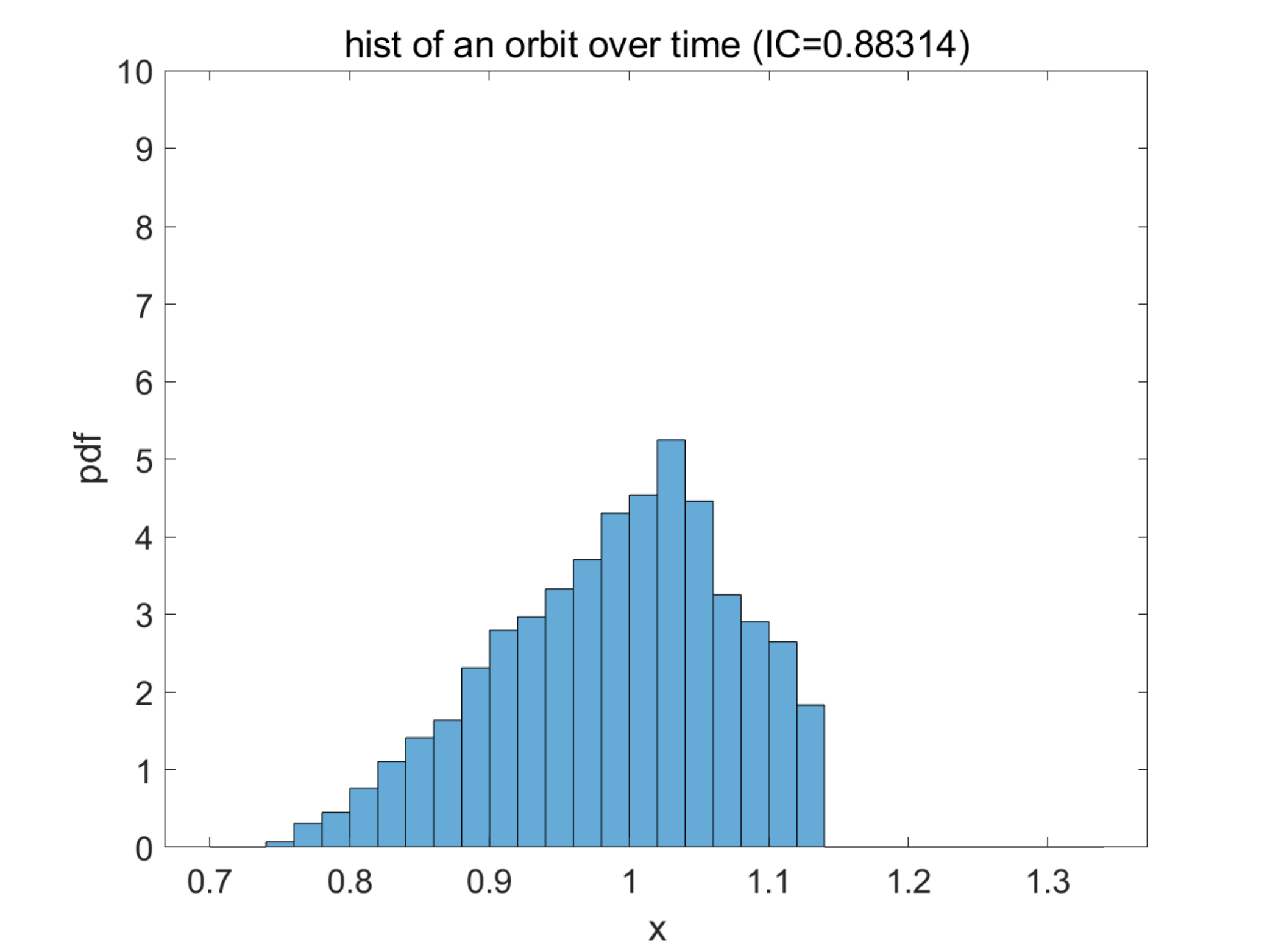}}
	\hfill
	\subfigure[Histogram of another trajectory, starting in the left well]{\includegraphics[width=0.45\linewidth]{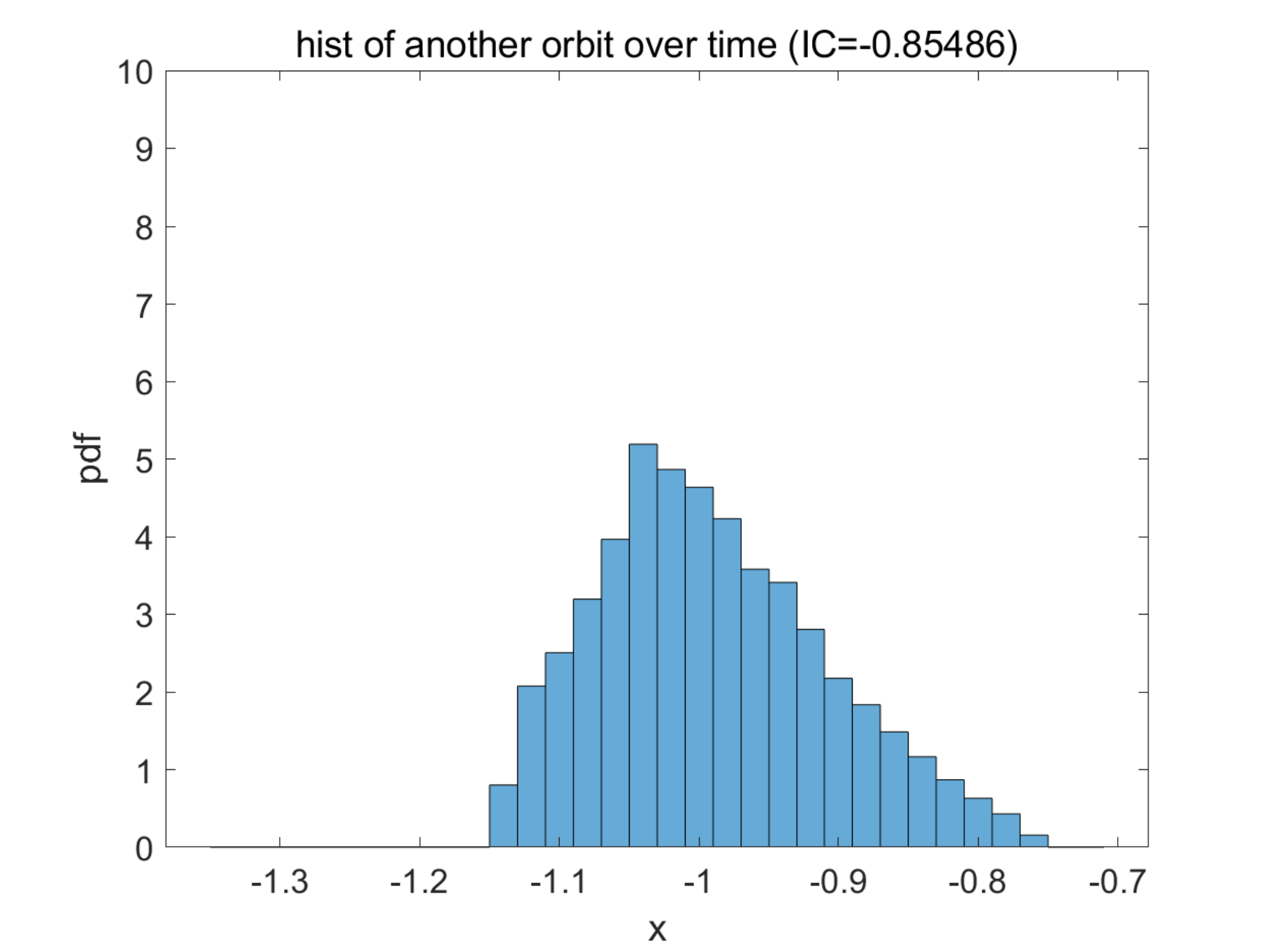}}
	\hfill
	\subfigure[Landscape of $f_0$]{\includegraphics[width=0.45\linewidth]{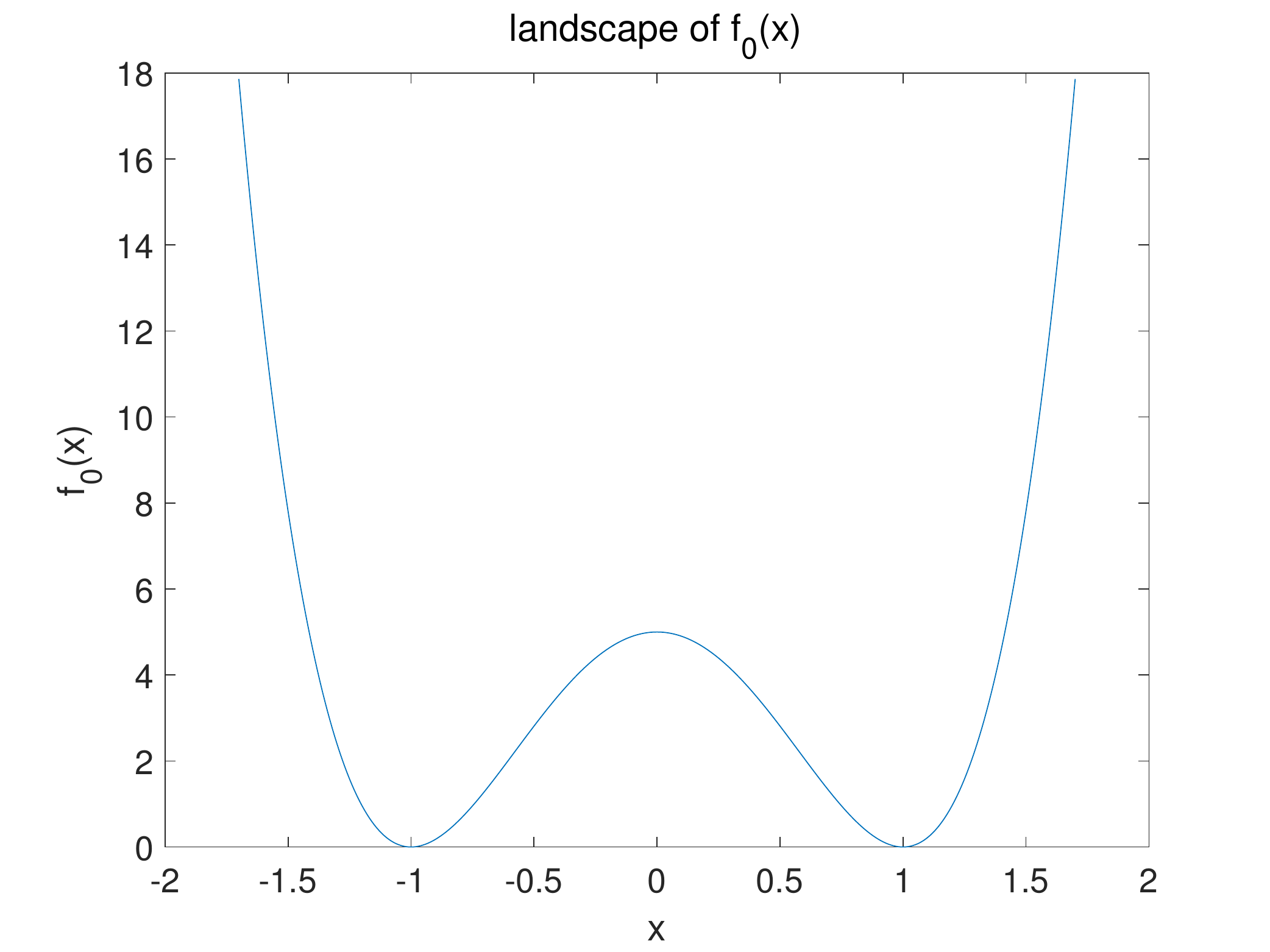}}
	\caption{A non-convex and non-mixing example. The initial condition is concentrated in the right potential well but no orbit can cross the potential barrier at $x=0$. There is at least another invariant distribution in the left potential well due to symmetry. But if one restricts to the foliation within the potential well, convergence to a statistical limit still occurs. $k=5$, $\eta=0.05$ and $\epsilon=0.0001$.}
	\label{fig_nonconvexWeakF1}
\end{figure}

Interestingly, we observe that Rmk.\ref{rmk_invDis_Taylor} still holds even though the orbit is confined in one potential well if $k$ is large. As $f''(1)>0$, the function is strongly convex in a neighborhood of $x=1$, and rescaled Gibbs can be approximated by a Gaussian density of $\exp(-16k(x-1)^2)/Z$. Fig.\ref{fig_nonconvexGaussianApprx} shows that the ensemble empirical distribution indeed converges to this prediction as $\eta\rightarrow 0$.

\begin{figure}[H]
	\label{NT_nonconvex_normal}
	\centering
	\hfill
	\subfigure[$\eta=0.05$]{\includegraphics[width=0.22\linewidth]{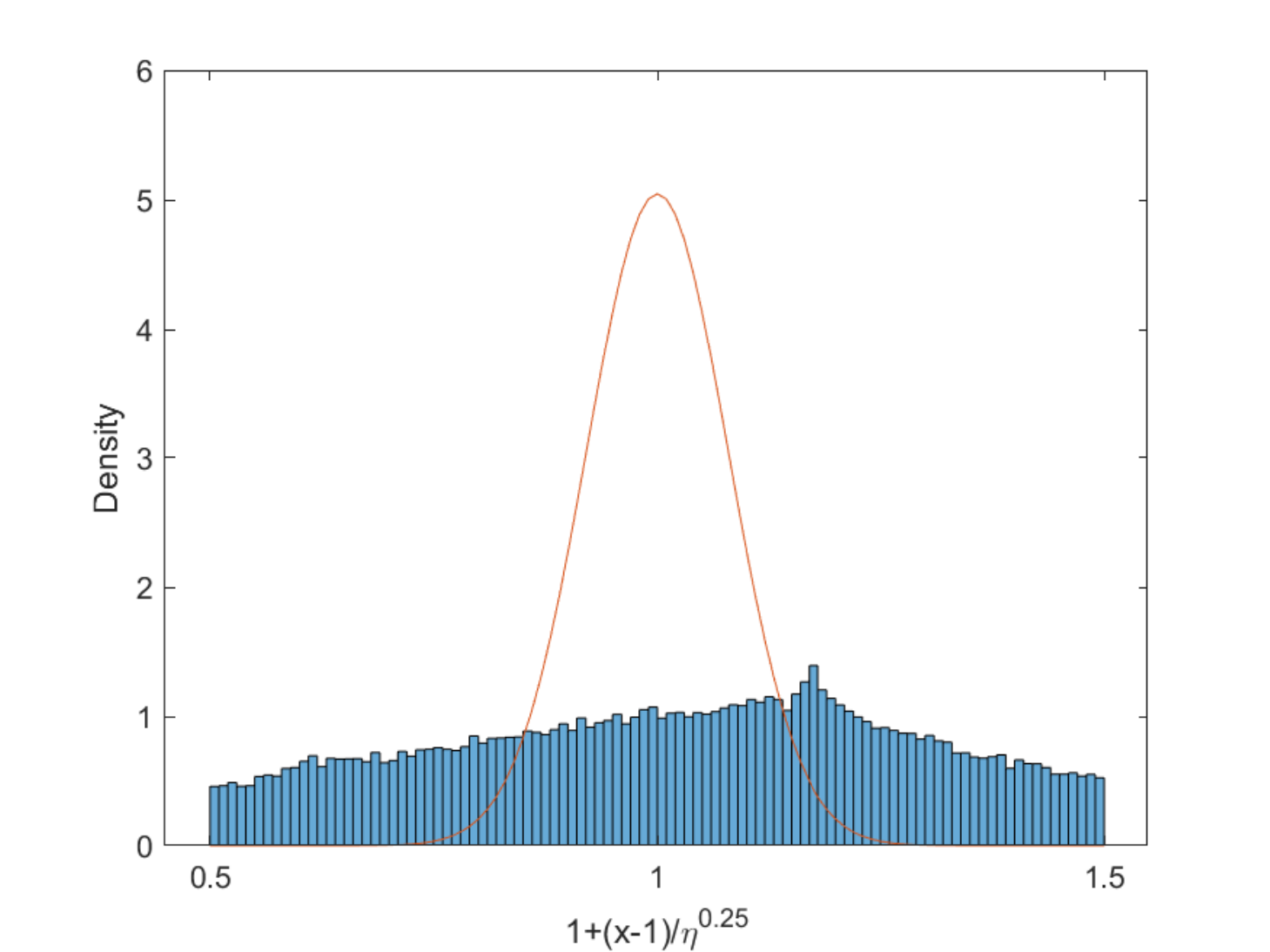}}
	\hfill
	\subfigure[$\eta=0.02$]{\includegraphics[width=0.22\linewidth]{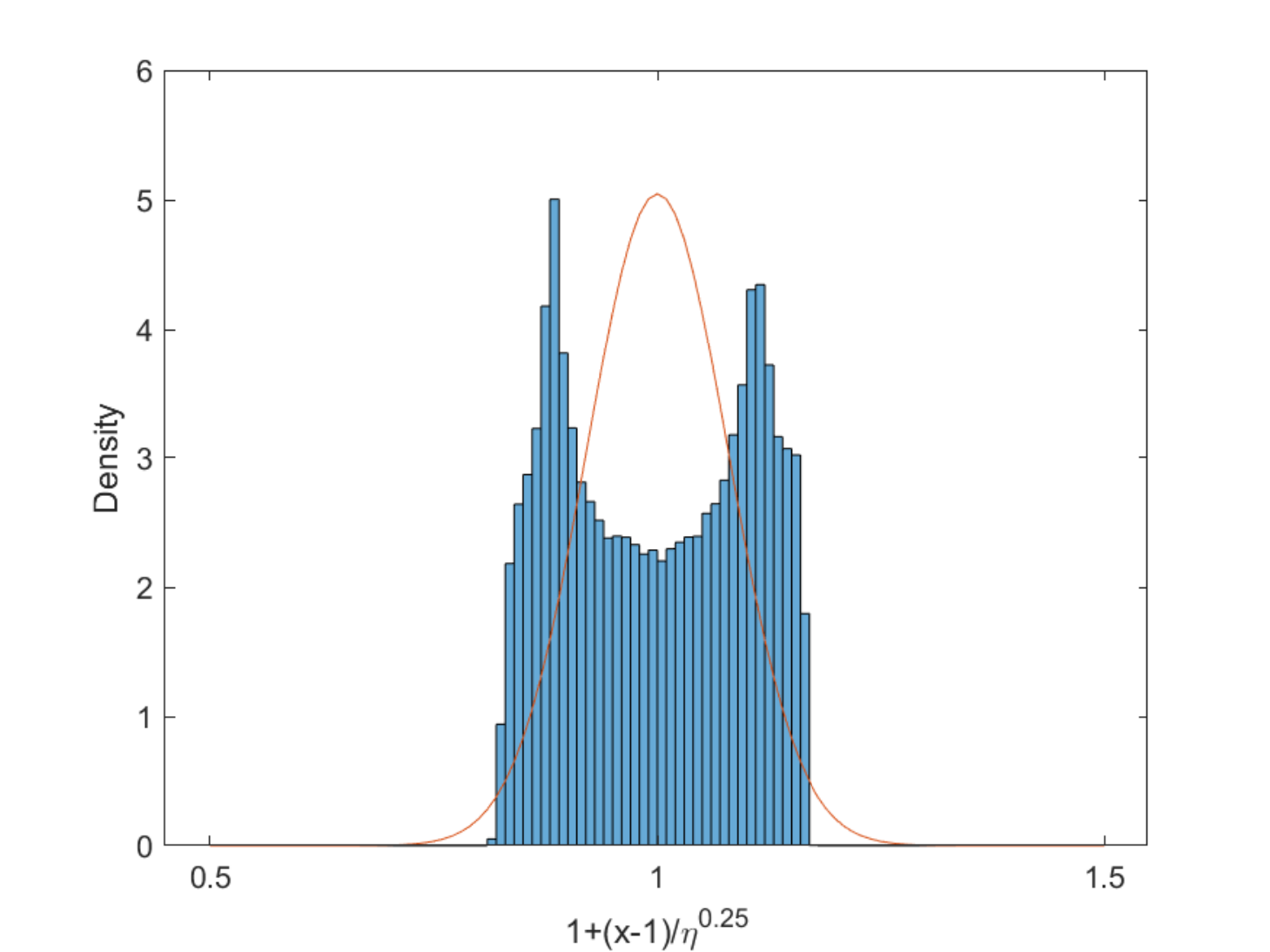}}
	\hfill
	\subfigure[$\eta=0.01$]{\includegraphics[width=0.22\linewidth]{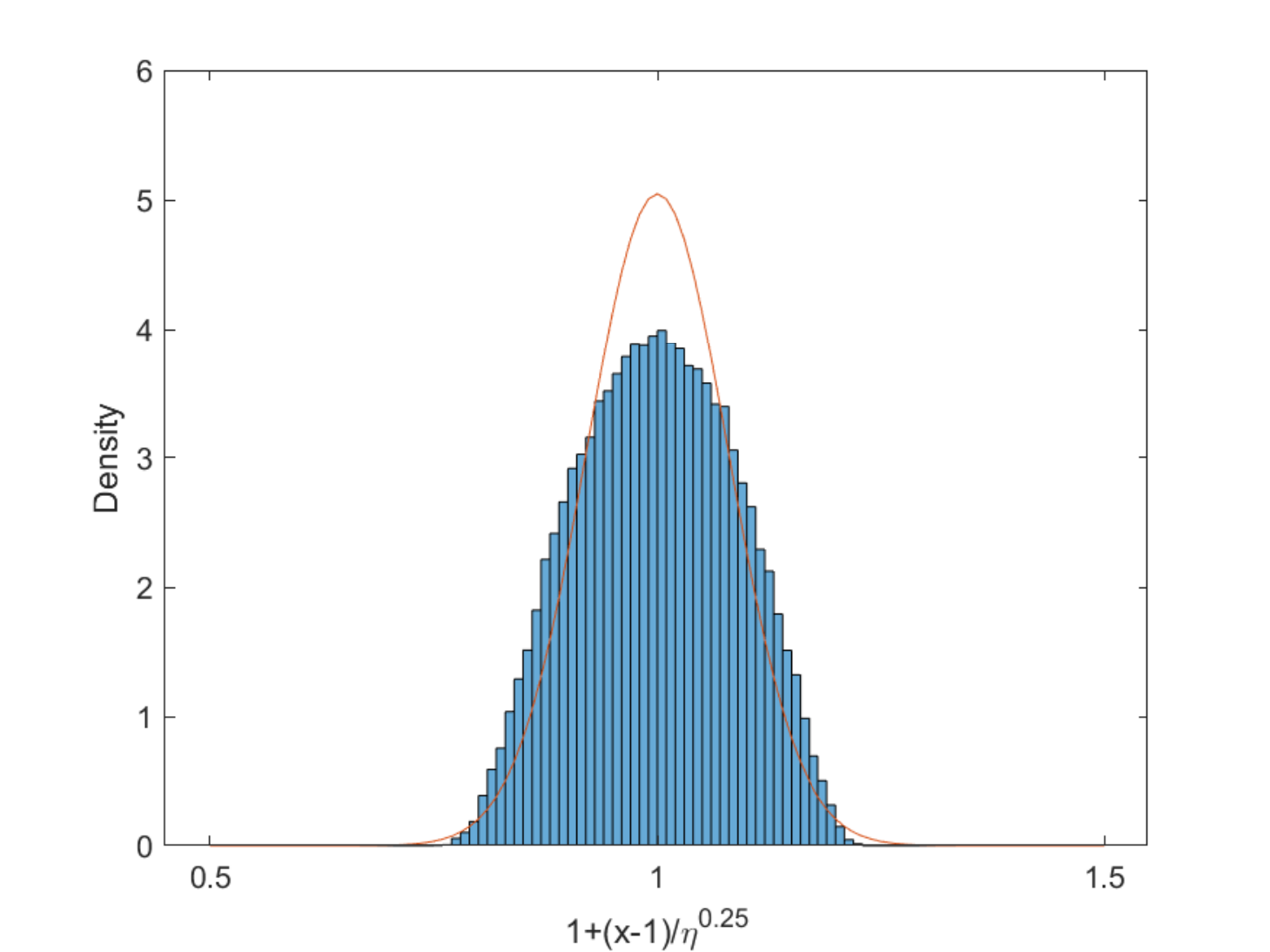}}
	\hfill
	\subfigure[$\eta=0.001$]{\includegraphics[width=0.22\linewidth]{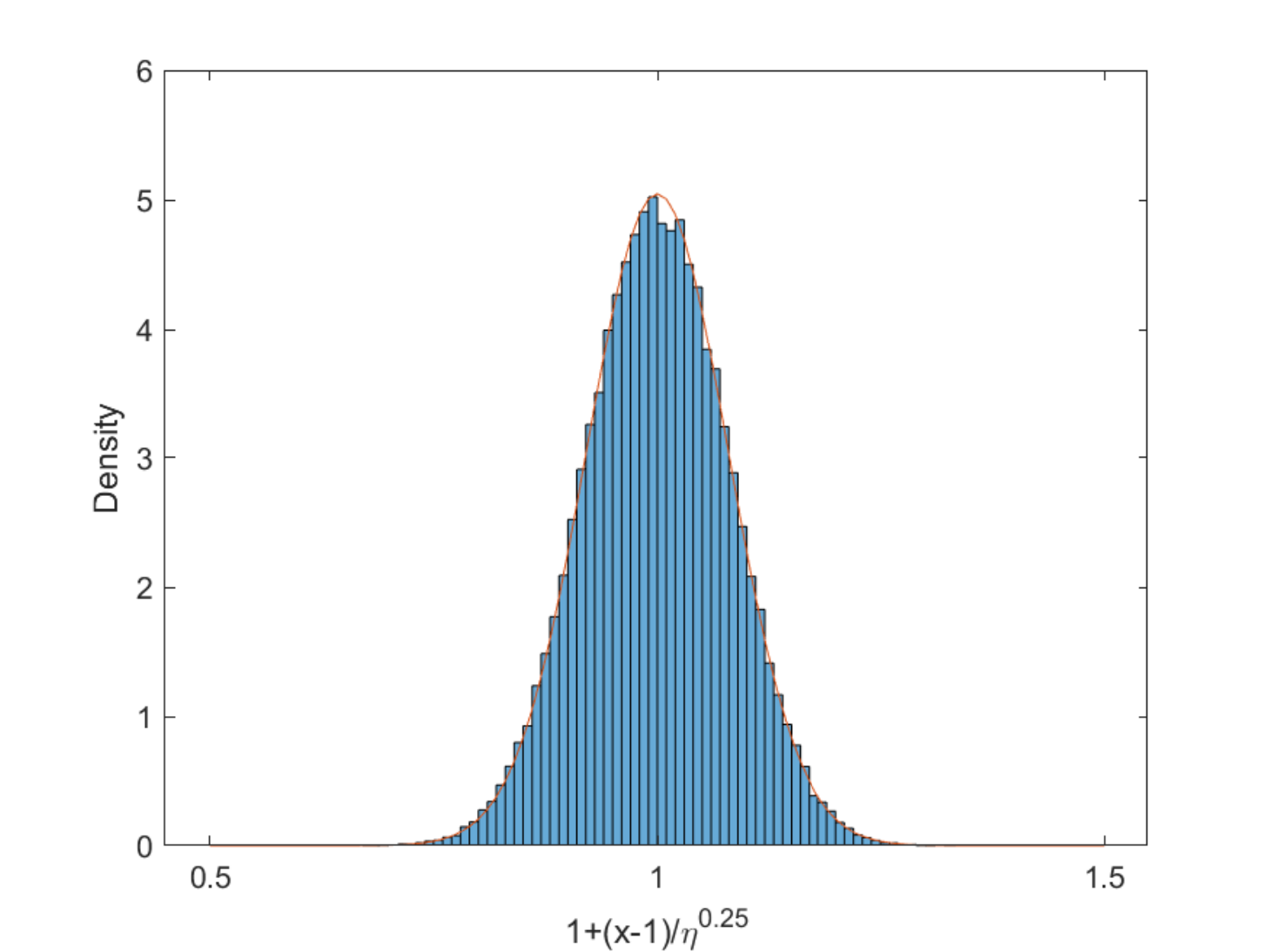}}
	\caption{Empirical distributions of a sufficiently evolved ensemble for different $\eta$ values when $k=5$. The red line is the theoretical approximation in Rmk.\ref{rmk_invDis_Taylor}. Note x-axis has been zoomed in via $x \mapsto 1+(x-1)/\sqrt\eta$ for focusing on the essential part.}
	\label{fig_nonconvexGaussianApprx}
\end{figure}

\end{document}